\documentclass[a4]{article}
\usepackage{hyperref}
\usepackage{url}
\usepackage[T1]{fontenc}
\usepackage{ae,aecompl}
\usepackage[utf8x]{inputenc}
\usepackage{amsmath}
\usepackage{empheq}
\usepackage{amsfonts}
\usepackage{amssymb}
\usepackage{amsthm}
\usepackage{graphicx, textcomp, caption}
\usepackage{pst-math}
\usepackage{pst-node}
\usepackage[affil-it]{authblk}
\usepackage{enumitem} %, showframe}

\usepackage{placeins}

\let\Oldsection\section
\renewcommand{\section}{\FloatBarrier\Oldsection}

\let\Oldsubsection\subsection
\renewcommand{\subsection}{\FloatBarrier\Oldsubsection}

\let\Oldsubsubsection\subsubsection
\renewcommand{\subsubsection}{\FloatBarrier\Oldsubsubsection}

\title{Global Navigation Using Predictable and Slow Feature Analysis in Multiroom Environments, Path Planning and~Other Control Tasks}
\author{Stefan~Richthofer\footnote{Electronic address: \texttt{stefan.richthofer@ini.rub.de}; 
Corresponding author}, Laurenz~Wiskott\footnote{Electronic address: 
\texttt{laurenz.wiskott@ini.rub.de}}}
\affil{Institut f\"ur Neuroinformatik,\\ Ruhr-Universit\"at Bochum, Germany}

\providecommand{\absatz}{\vspace{6 mm}}

\providecommand{\abs}[1]{\lvert#1\rvert}

\providecommand{\norm}[1]{\lVert#1\rVert}

\providecommand{\av}[1]{\left\langle#1\right\rangle}

\providecommand{\bigav}[1]{\big\langle#1\big\rangle}

\providecommand{\coloneq}[0]{\mathrel{\mathop:}=}
\providecommand{\id}[0]{\mathbf{I}}
\providecommand{\trph}[0]{\Omega_t}
\providecommand{\eqcolon}[0]{= \mathrel{\mathop:}}

\DeclareMathOperator*{\vol}{vol}

\DeclareMathOperator*{\eq}{=}
\DeclareMathOperator*{\appr}{\approx}

\DeclareMathOperator*{\lra}{\Leftrightarrow}

\DeclareMathOperator*{\ra}{\Rightarrow}
\DeclareMathOperator{\sign}{sign}
\DeclareMathOperator{\mvec}{vec}
\DeclareMathOperator{\erf}{erf}
\DeclareMathOperator{\lag}{L}
\providecommand{\hist}[2]{\left(\lag^1 #1, \ldots, \lag^{#2} #1 \right)}

\DeclareMathOperator{\tr}{Tr}
\DeclareMathOperator{\Hermite}{H}
\DeclareMathOperator{\Tschebyschow}{T}

\DeclareMathOperator{\orth}{O}

\DeclareMathOperator*{\opmin}{minimize}

\providecommand{\optmin}[1]{\displaystyle\opmin_{#1} \qquad}
\DeclareMathOperator*{\subjectto}{subject \; to \qquad}

\providecommand{\pfaerrob}[2]{\bigav{\; \norm{#2-#1}^2 \; }}
\makeatletter
\newcommand{\subalign}[1]{%
	\vcenter{%
		\Let@ \restore@math@cr \default@tag
		\baselineskip\fontdimen10 \scriptfont\tw@
		\advance\baselineskip\fontdimen12 \scriptfont\tw@
		\lineskip\thr@@\fontdimen8 \scriptfont\thr@@
		\lineskiplimit\lineskip
		\ialign{\hfil$\m@th\scriptstyle##$&$\m@th\scriptstyle{}##$\crcr
			#1\crcr
		}%
	}
}
\makeatother

\newtheoremstyle{defi}% name
  {}%		Space above
  {}%		Space below
  {\slshape}%{\itshape}%			Body font
  {}%			Indent amount 1
  {\bfseries}%			Theorem head font
  {}%			Punctuation after theorem head
  {\newline}%			Space after theorem head
  {}%			Theorem head spec (can be left empty, meaning ‘normal’)

\theoremstyle{defi}
\newtheorem{alg}{Algorithm}
\newtheorem{subr}{Subroutine}

\newtheorem{lem}{Lemma}
\newtheorem{thm}{Theorem}

\begin{document}

\maketitle

\begin{abstract}
Extended Predictable Feature Analysis (PFAx) \cite{2017arXiv171200634R} is an extension of PFA \cite{DBLP:conf/icmla/RichthoferW15} that allows generating a goal-directed control signal of an agent whose dynamics has previously been learned during a training phase in an unsupervised manner. PFAx hardly requires assumptions or prior knowledge of the agent's sensor or control mechanics, or of the environment. It selects features from a high-dimensional input by intrinsic predictability and organizes them into a reasonably low-dimensional model.

While PFA obtains a well predictable model, PFAx yields a model ideally suited for manipulations with predictable outcome.
This allows for goal-directed manipulation of an agent and thus for local navigation, i.e.\ for reaching states where intermediate actions can be chosen by a permanent descent of distance to the goal. The approach is limited when it comes to global navigation, e.g.\ involving obstacles or multiple rooms.

In this article, we extend theoretical results from \cite{SprekelerWiskott-2008}, enabling PFAx to perform stable global navigation. So far, the most widely exploited characteristic of Slow Feature Analysis (SFA) was that slowness yields invariances. We focus on another fundamental characteristics of slow signals: They tend to yield monotonicity and one significant property of monotonicity is that local optimization is sufficient to find a global optimum.

We present an SFA-based algorithm that structures an environment such that navigation tasks hierarchically decompose into subgoals. Each of these can be efficiently achieved by PFAx, yielding an overall global solution of the task. The algorithm needs to explore and process an environment only once and can then perform all sorts of navigation tasks efficiently. We support this algorithm by mathematical theory and apply it to different problems.
%Furthermore, we validate the theory by comparing its implications with experimental results.
%We present an SFA-based algorithm that decomposes an environment into globally solvable subgoals, each of which can be efficiently optimized by PFAx. We support this algorithm by mathematical theory and apply it to different problems. Furthermore, we validate the theory by comparing its implications with experimental results.
%While PFA obtains a well predictable model, PFAx yields a model ideally suited for manipulations with predictable outcome.
%(i.e.\ beyond general sanity criteria like signals being ergodic) 
\end{abstract}

\pagebreak

\section{Introduction} \label{sec:introduction}

The original motivation of this work is based on the idea to apply the unsupervised learning algorithm Slow Feature Analysis (SFA) \cite{WiskottSejnowski-2002} to interactive scenarios. The motivation for this idea is based on the experience that SFA was successfully used in various (passive) analysis tasks that closely relate to such scenarios, e.g.\ learning place cells \cite{FranziusSprekelerEtAl-2007e, Schoenfeld2015}, identifying objects invariant under spacial transformations \cite{10.1007/978-3-540-87536-9_98, 10.1007/3-540-46084-5_14, FranziusWilbertEtAl-2011}, blind source separation \cite{SprekelerZitoEtAl-2014}, visual tasks like face recognition and age estimation \cite{Escalante-B.Wiskott-2013b}.
In previous work, we recognized predictability as a crucial feature for tackling interactive scenarios, as these require estimation of consequences of possible actions. This lead to the invention of Predictable Feature Analysis (PFA) \cite{DBLP:conf/icmla/RichthoferW15, 2017arXiv171200634R}, an algorithm strongly inspired by SFA -- while SFA selects features by slowness, PFA selects them by predictability. Before we get into more detail of these algorithms, we briefly collect possible application fields.

Path planning of mobile robots is an application area that closely fits our implicit prototype assumptions. We imagine a robot in an environment that perceives sensory input of some kind and emits an action signal that controls its motors (c.f.\ Figure~\ref{fig:cycle}). A naturally arising task is to control the robot such that it reaches a desired state in the environment.

\begin{figure}[h!]
	\centering
	\includegraphics[width=0.6\hsize]{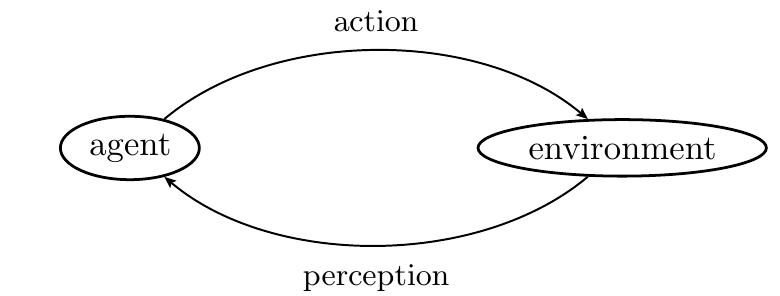}
	\caption{Perception/action cycle}
	\label{fig:cycle}
\end{figure}

SFA has been frequently applied to model a rather similar scenario concerning a rat instead of a robot. Of course, it was not attempted to control the rat, but to obtain biologically plausible phenomena like place cells.
A more general interactive problem setting is reinforcement learning (RL), where also an agent is acting in an environment, aiming for a maximal accumulated reward over time. In this fashion, that setting extends our notion of a sensor signal and a control signal by a reward signal.
Also control theory of dynamical systems, involving tasks like pendulum swing up, pole and cart balancing, fits into the notion of an action/perception loop illustrated in figure~\ref{fig:cycle}. The phase space of the system can be seen as environment in this case.
A rich repertoire of work exists that links these fields in various ways. Traditionally, RL algorithms are applied to path planning, or dynamical systems. We list a selection of such articles throughout this section.

As a unifying notion of the named areas' essentials we stick to the idea of controlling an agent in an environment, aiming for a specific goal state. Environment, agent and control are represented as abstract, continuous sensor and control signals. %We assume continuous state and action spaces as this is the setting SFA and PFA are designed for.
Based on this, we perceive the navigation into a goal state as an optimization problem.
This matches the setting we tackled with PFAx in \cite{2017arXiv171200634R} and we continue with a comprehension of that approach.

\subsection{Predictable Feature Analysis (Extended)}
Predictable Feature Analysis (PFA) \cite{DBLP:conf/icmla/RichthoferW15} is an unsupervised learning algorithm that was developed to efficiently turn high-dimensional input data into a low-dimensional model consisting of well predictable features.

In \cite{2017arXiv171200634R} we have shown that by using an extension to PFA -- namely PFAx -- it is possible to learn well controllable features that are sufficient to solve local navigation tasks. By taking supplementary information into account for prediction, PFAx can find features that are ideally predictable based on themselves and under the assumption that a supplementary signal can be used as a helper for prediction. Such a supplementary signal does -- however -- not participate in feature extraction. Providing the control signal from the RL setting (specific action chosen at each time-step) as supplementary information, we can obtain features that strongly depend on the supplementary information in terms of predictability. By inverting that relation we can compute the control signal that would most likely yield a specific desired outcome, given the agent's current state. In this sense the obtained features are well controllable.

Solving a complex navigation task usually cannot be achieved within a single time step, so we transformed the agent's state as far as possible towards the goal state in each time step (greedy optimization), using least squares distance in feature space as a cost function. The resulting approximate gradient descent easily gets stuck in a local optimum, thus PFAx is only suitable to perform local navigation.

\subsection{Approach in this work}
%Relevanz, lokale Optima zu vermeiden: Global path planning using artificial potential fields

%The system (possibly agent in an environment) is complicated and difficult to optimize, because the state space is high-dimensional and non-monotonic and the effect of the control signal unknown.
%Through unsupervised exploration the system learn a lower-dimensional monotonic feature representation of the environment will a locally well characterized control signal.
%Based on the learned representation, any optimization/navigation problem can be solved by local descent on the monotonic feature representation, using the learned effects of the control signals.

A key observation is that Slow Feature Analysis (SFA) \cite{WiskottSejnowski-2002} can be used to decompose a given environment into features that are represented as monotonic signals across the environment (see \cite{SprekelerWiskott-2008}). We refer to these monotonic features as \textsl{sources}. Computing them from the usual slow features found by SFA requires additional processing. This is provided by the xSFA algorithm \cite{SprekelerZitoEtAl-2014}. The theoretical analysis of xSFA only scopes the case that statistically independent sources exist. We extend this analysis in section \ref{sec:xSFA_on_manifolds} and establish a geometrical characterization of the solutions in terms of potential, monotonicity, geodesics and representation of the data manifold. A major contribution of this work is to clarify what (x)SFA-induced monotonicity means in higher dimensions. These results motivate the navigation algorithm proposed in section \ref{sec:Global-navigation-algorithm}.

Obtaining the sources involves an extensive unsupervised exploration phase in which PFAx can learn the effects of the control signal and (x)SFA can learn a model of the environment.
%Monotonicity of the sources slow features are composed of ensures allows 
Based on the learned representation, any navigation problem can be solved by local descent on the monotonic feature representation, using the learned effects of the control signal. In figure \ref{fig:sfa-goal-dist-illustration} we illustrate this with the interval $[0, 100]$ serving as a 1D environment. As a sensor representation we model five differently parametrized grid cells using overlapping Gaussians. The component (\textsl{sf1}) obtained by xSFA yields a monotonic representation of the environment. Concerning the goal distance measures on the right, imagine we wanted to move an agent from e.g.\ position $80$ to the goal at $40$. Based on local techniques (imagine a limited perception range, e.g.\ one or two units) it is impossible to efficiently find the goal regarding sensor space. Measuring distance by \textsl{sf1}, the same task -- actually any navigation task -- is well feasible.

\begin{figure}[!ht]
	\centering
	\captionsetup{width=.95\linewidth}
	\includegraphics[width=1.\hsize]{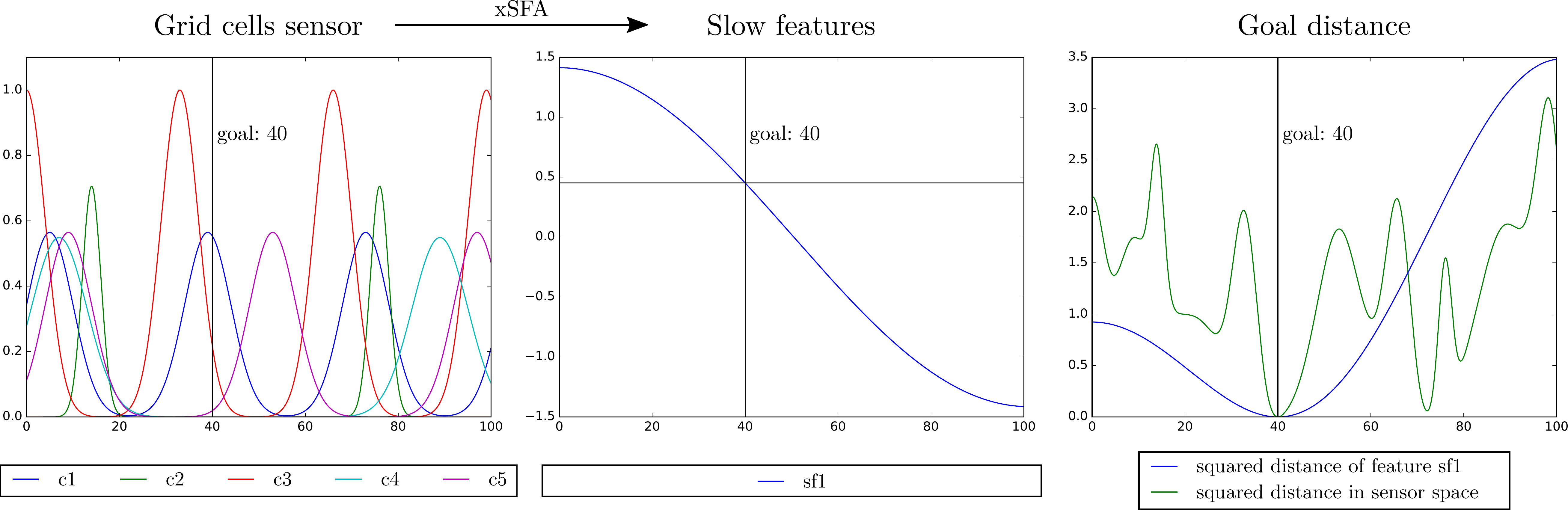}
	\caption{Illustration of how xSFA obtains a cost function suitable for efficient global optimization on $[0, 100]$. \textsl{Left:} Overlapping Gaussians are used to model grid cells as non-monotonic representation. \textsl{Center:} Results of xSFA applied to the representation on the left. The extracted component (\textsl{blue}) yields a monotonic representation of the environment. \textsl{Right:} Comparison of goal distance functions. Squared goal distance in sensor space (\textsl{green}) \textsl{cannot} be optimized globally by local methods. Squared goal distance in feature space (\textsl{blue}) \textsl{can} be optimized globally by local methods.}
	\label{fig:sfa-goal-dist-illustration}
\end{figure}

A special case of the approach in this work was studied in \cite{8202307}. Leveraging the monotonicity of the slowest SFA components, a robot is navigated in an approximately open environment around an obstacle. This asserts the feasibility of the method in principle.
A major difference to the approach presented here is that the control of the robot is not learned, but assumed to be known. Also the estimation of the gradient is done in a different manner. The navigation is based on a fixed selection of SFA components, which limits it to environments yielding spatial dimensions of roughly the same size, e.g.\ with quadratic or circular boundary. In that sense, our paper presents a generalization of that approach, vastly relaxing the geometrical requirements on the environment. However, we still require some non-geometric limitations on the environment:

\begin{itemize}
	%\item The environment is simply connected
	\item The environment is fully observable, i.e.\ each position in the environment yields a unique representation in sensor space
	\item The environment is stationary, i.e.\ constant over time, contains no blinking lights, no flickering colors or moving objects
\end{itemize}

%Note that these limitations would not significantly simplify the task in terms of a general black-box global optimization problem. None of them allows for assumptions like the resulting problem being convex.
The named limitations are not inherent and in section \ref{sec:conclusion} we suggest extensions of the algorithm to overcome each of them. They are just simplifying assumptions to focus on the core method in this work.

\subsection{Connection to optimization} \label{sec:conn-opt}
Using a distance measure (e.g.\ least squares distance in sensor or feature space) of the agent's current state to a goal state as a cost function, a navigation task can be seen as an optimization problem. Complex tasks usually cannot be achieved within a single time step. This can be resolved in several ways, e.g.:
\begin{itemize}
	\item An optimization problem could be defined on the space of possible paths rather than on a space of possible actions.
	\item Greedy optimization can be used, i.e.\ transforming the agent's state as far as possible towards the goal state in each time step (c.f.\ gradient descent).
	\item Optimization can scope on finding a good policy.
\end{itemize}
Finding a policy is the typical approach in RL, which is discussed in the next section.
Greedy optimization easily gets stuck in a local optimum and finding the globally optimal path in an arbitrary environment is a general black-box global optimization problem. Without further assumptions on the environment, such a problem would in general not be convex, quadratic or anything significantly useful. Historically, a wide range of techniques has been developed to deal with such black-box problems, e.g.\ evolutionary algorithms, swarm algorithms, convex relaxations, cutting plane methods, branch and bound methods, stochastic methods and many more. These techniques have in common that they are magnitudes slower than methods for efficiently optimizable problem domains like convex, quadratic or linear problems. Apart from that, they usually cannot guarantee convergence or actual optimality of the obtained result.

The method in this work can be seen as a reverse approach applicable to our specific setting. Instead of attempting to solve the resulting difficult optimization problem, we create the optimization problem such that it is efficiently solvable. Given that we already use a model to represent the agent's state, we have a certain degree of freedom in forming this model.
The presented method can intentionally form the model such that all possible goal states in the environment are efficiently achievable from any initial state.

In this context, we exploit SFA-induced monotonicity as a link between local and global optimization: Exactly on monotonic functions, the distance measure to a goal value has only a single optimum. Thus, exactly on these functions, local techniques like PFAx are sufficient for global optimization of goal-distance. However, for multidimensional environments there does not exist a single strictly monotonic component that can cover all possible goal states. Instead we will find a hierarchical decomposition into components to cover the whole environment, defining a precise sequence of efficiently and globally solvable optimization tasks to reach any desired goal state that exists in the environment, provided that it has been sufficiently explored.

%-papers, die Path planning mit general blackbox optimization approaches angehen:
%http://www.e-ijaet.org/media/25I14-IJAET0514200_v6_iss2_780to788.pdf (neural network)
%Robot Path Planning Based on Random Coding Particle Swarm Optimization
%The robot path planning based on ant colony and particle swarm fusion algorithm
%http://ieeexplore.ieee.org/abstract/document/870304/ (evolutionary)

\subsection{Connection to Reinforcement Learning}

Navigating an agent in an environment has significant links to RL. However, RL usually considers an arbitrary reward signal, typically in discrete state and action spaces, while approaches for continuous RL exist. Navigation tasks and path planning can be encoded in the RL setting by measuring the distance of the agent's current state from the goal state by a certain metric (e.g.\ euclidean distance in sensor or feature space). Using an inversion of this measure (e.g.\ $\frac{1}{x}$ or $\max - x$) as reward function would yield maximum reward at the goal. The main formal difference to the optimization view is that the objective function, i.e.\ the accumulated future reward, in RL terms the \textsl{value function} is unknown and is usually learned in popular methods such as Q-Learning.
Apart from that, having a reward function that is non-zero across wider ranges of the environment is untypical for RL. Usually reward is only given right at the goal and the value function is learned across multiple sessions.

In terms of RL, the approach in this work would mean that during an initial exploration phase, reward would be completely ignored and instead the topology of the environment and the dynamics of the agent would be exhaustively learned. Based on this, any reward signal that measures distance to some goal position can be maximized efficiently. Goal and start position can be arbitrary and any number of such tasks with varying start and goal can be performed efficiently based on a single initial exploration. This is possible, because the model is aligned to the environment and not the specific reward function. Thus, our approach is especially valuable if the goal is not known during exploration and if many tasks with different goals need to be performed in the same environment. Note that in terms of RL, our approach does not account for the exploration-vs-exploitation issue. It rather performs exhaustive exploitation after exhaustive exploration.

Proto value functions (PVFs) \cite{MahadevanMaggioni-2007} are an RL concept that shares some characteristics with the presented approach. They are frequently used to discover bottleneck states and options in RL settings. SFA has been used earlier to approximate PVFs, e.g.\ in \cite{DBLP:conf/icann/LuciwS12, JMLR:v14:boehmer13a}. %, which are also used to decompose RL environments into subgoals by discovering bottleneck states \cite{MahadevanMaggioni-2007}.
The extracted slow features can then be used as basis functions for linear models that solve the RL problem, such as LSTD \cite{Lagoudakis:2002:LMR:645861.670291}.

Originally, the central building block of PVFs are Laplacian eigenmaps (LEMs) rather than slow features. LEMs are traditionally more affected by the curse of dimensionality in RL as the dimensionality of the graph Laplacian depends on the number of data points.
\cite{doi:10.1162/NECO_a_00214} provides the missing link between SFA and LEMs. Specifically, that work identifies conditions under which the problem settings of SFA and LEMs become equivalent and describes how LEMs can be approximated by slow features. The application of SFA as replacement for LEMs proposed in \cite{DBLP:conf/icann/LuciwS12, JMLR:v14:boehmer13a} was originally enabled by the named article.
For our purpose, slow features yield another crucial advantage over LEMs: SFA's central optimization problem can be analyzed by Sturm-Liouville theory, which allows to formally prove monotonicity results for slow features \cite{SprekelerWiskott-2008}. %Monotonicity is a central property required by our approach.

In this work we primarily take the optimization perspective onto the navigation setting. Our approach is rather driven by xSFA and its theoretical implications and we propose that based on results of a sufficiently converged xSFA processing, it is already feasible to fully solve our setting by local and efficient optimization techniques.

%Papers, die Path Planning mit RL angehen:
%Reinforcement Learning-Based Path Planning for Autonomous Robots
%http://www.e-ijaet.org/media/25I14-IJAET0514200_v6_iss2_780to788.pdf
%Navigation and path planning using reinforcement learning for a Roomba robot

\subsection{Connection to path planning}

Path planning of mobile robots closely fits our assumed setting as it deals with finding a safe path to navigate a robot through a complex environment. Work in this area usually assumes a specific goal location. A key difference is that the robot's dynamics are usually known and focus is fully on planning the path. Our approach on the other hand does not incorporate a safety criterion, but this could be modeled as an additional part of the sensor.

The optimization issues we stated in section \ref{sec:conn-opt} are widely recognized in this field, especially avoidance of local optima is a central concern \cite{100007}. We pointed out the nature of a black box optimization problem and indeed various typical black box optimization approaches have been applied to path planning:
Evolutionary methods \cite{870304}, neural networks \cite{5286557}, particle swarm optimization \cite{Su_robotpath}, ant colony optimization \cite{TAN2007279, 8242802} and others \cite{4058742}.
Path planning has also been frequently approached using RL \cite{7955160, 6974463, NIPS1993_843, Igarashi2002, doi:10.1177/0278364907087426}.
%path-planning
%Reinforcement Learning-Based Path Planning for Autonomous Robots
%http://ieeexplore.ieee.org/document/4058742/
%http://www.e-ijaet.org/media/25I14-IJAET0514200_v6_iss2_780to788.pdf (neural network)
%Robot Path Planning Based on Random Coding Particle Swarm Optimization
%The robot path planning based on ant colony and particle swarm fusion algorithm
%Global path planning using artificial potential fields (Relevanz, lokale Optima zu vermeiden)

\subsection{Other related work}

We give an overview of various other more or less closely related approaches. Some papers are listed because they apply a hierarchical decomposition of some sort to RL scenarios, others are listed because they deal with slowness or predictability.

%Automatic Discovery of Subgoals in Reinforcement Learning using Diverse Density
%Learning Options in Reinforcement Learning
\cite{DBLP:conf/icml/McGovernB01} and \cite{10.1007/3-540-45622-8_16} use diverse density to discover bottleneck states as useful subgoals for RL tasks. This has some parallels to the algorithm in this paper in the sense that bottlenecks occur as special states. In \ref{sec:bottleneck} we explicitly study the behavior of SFA around a bottleneck and suggest in a side note how slow features can be used as a bottleneck detector.

%Design Principles of the Hippocampal Cognitive Map
\cite{NIPS2014_5340} suggests how a hierarchical decomposition of a problem space can be achieved using the successor representation and its eigenvalue decomposition. That work draws a number of links to studies of animal behavior, observations in the hippocampus and to cognitive maps. It contains many interesting notes regarding biological plausibility of decomposition approaches.

%Hierarchically organized behavior and its neural foundations: A reinforcement learning perspective
\cite{BOTVINICK2009262} addresses the scaling problem/curse of dimensionality in RL. They develop a hierarchical notion of RL (HRL) in a model-free actor-critic approach. The work features an extensive discussion of implications for neuroscience and psychology.

%Autonomous Learning of State Representations for Control
\cite{DBLP:journals/ki/BohmerSBRO15} gives an overview of methods for autonomous RL directly based on sensor-observations. The mainly discussed algorithms are deep auto encoders and SFA. They mention slow distractors (e.g.\ the position of a slowly moving sun in an outdoor scenario) as a typical issue of SFA. This supports our idea of a combination of SFA with predictability in form of PFAx, which should be able to discard slow distractors. However, PFAx can be affected by predictable distractors, but these can be identified by the coefficients of the matrix incorporating the control signal.
Finally they explicitly point out the idea of combining notions of slowness with predictability, referencing the follwoing work:

\cite{Jonschkowski-13-ERLARS} combines notions of slowness and predictability to learn state representations for RL using a neural network. To combine these notions they propose a hybrid cost function consisting of arbitrarily weighted terms for slowness, predictability and non-constantness.
%Learning Task-Specific State Representations by Maximizing Slowness and Predictability

There are a number of approaches related to SFA, PFA or PFAx we discussed in a little more detail in \cite{2017arXiv171200634R}: 
Contingent Feature Analysis (CFA) \cite{DBLP:conf/icann/Sprague14}, Forecastable Component Analysis (ForeCA) \cite{goerg13}, Graph-based Predictable Feature Analysis (GPFA) \cite{Weghenkel:2017:GPF:3140707.3140724}, Predictive Projections \cite{DBLP:conf/ijcai/Sprague09}, Neighborhood Components Analysis (NCA) \cite{DBLP:conf/nips/GoldbergerRHS04}, A Canonical Analysis of Multiple Time Series \cite{boxTiao1977}.

\section{Local navigation using predictable features with supplementary information (PFAx)} \label{sec:extraction}
We start with a comprehension of the PFAx algorithm \cite{2017arXiv171200634R} which extends the PFA algorithm \cite{DBLP:conf/icmla/RichthoferW15} to incorporate supplementary information. Later we extend the method to enable global navigation. Given an $n$-dimensional input-signal $\mathbf{x}(t)$, PFA's objective is to find $r$ most predictable output components, referred to as “predictable features”. PFAx additionally considers a signal $\mathbf{u}(t)$ and extracts $r$ components from $\mathbf{x}$ such that they are most predictable if $\mathbf{u}$ can be used as an additional helper for the prediction.

Like SFA, PFAx performs a linear extraction, but can incorporate a non-linear expansion $\mathbf{h}(\mathbf{x})$ as a preprocessing step. For this we usually use monomials up to a certain degree, which essentially yields a polynomial extraction overall\footnote[1]{For higher degree, Legendre or Bernstein polynomials should be favored over monomials because of better numerical stability.}. Note that by the Stone-Weierstrass theorem this technique can approximate any continuous function and moreover also regulated functions (piecewise continuous). However, high degree expansion can require significant cost in terms of training data and computation. Applying PFA hierarchically like is done with SFA in \cite{FranziusSprekelerEtAl-2007e, Schoenfeld2015} can help to keep these costs tractable.

\subsection{Recall PFAx} \label{sec:pfax}
In the PFAx setting, predictability is measured by \textbf{linear, auto-regressive processes} which are widely used to model time-related problems.
That means, each value of an extracted signal should be as predictable as possible by a linear combination of $p$ recent values.

This yields the problem of finding vectors $\mathbf{a} \in \mathbb{R}^{n}$ and $\mathbf{b} \in \mathbb{R}^{p}$ such that
\begin{align}
\mathbf{a}^T \mathbf{z}(t) \quad \appr^! \;& \quad b_{1} \mathbf{a}^T \mathbf{z}(t-1) + \ldots + b_{p} \mathbf{a}^T \mathbf{z}(t-p) \;
= \; \mathbf{a}^T \hist{\mathbf{z}}{p}(t) \; \mathbf{b} \; \eqcolon \; \mathbf{a}^T \hat{\mathbf{z}}_\mathbf{b}(t) \label{pfa-criterionARScalar}
\end{align}

\begin{figure}[ht]
	\centering
	\captionsetup{width=.7\linewidth}
	\includegraphics[width=0.75\hsize]{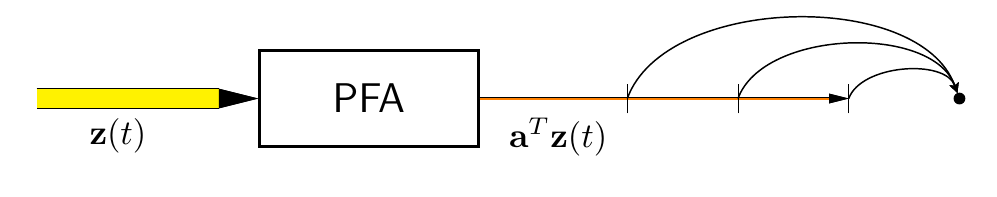}
	\caption{Illustration of PFA with extraction vector $\mathbf{a}$. Components are selected to be well predictable.}
	\label{fig:PFAIllustration}
\end{figure}

where $\lag$ denotes the \textsl{lag operator} (also \textsl{backshift operator}), i.e.\ $\lag^k \mathbf{z}(t) = \mathbf{z}(t-k)$,
and $\mathbf{z}$ is the expanded representation of our input signal $\mathbf{x}$, sphered over a finite training phase $\trph$ with average notation $\av{s(t)}~\coloneq~\frac{1}{\abs{\trph}} \sum_{t\in \trph} s(t)$:
\begin{align} \label{sphering1}
\tilde{\mathbf{z}}(t) \quad &\coloneq \quad \mathbf{h}(\mathbf{x}(t)) - \av{\mathbf{h}(\mathbf{x}(t))} &&\text{(make mean-free)}\\
\mathbf{z}(t) \quad &\coloneq \quad \mathbf{S} \tilde{\mathbf{z}}(t) \qquad \text{with} \quad \mathbf{S} \coloneq \av{\tilde{\mathbf{z}} \tilde{\mathbf{z}}^T}^{-\frac{1}{2}}, \; \mathbf{z}(t) \in \mathbb{R}^{n} &&\text{(normalize covariance)} \label{sphering2}
\end{align}

\eqref{pfa-criterionARScalar} can be formulated as a least squares optimization problem over the training phase $\trph$. We extend the problem to multiple output components $\mathbf{a}_1^T \mathbf{z}, \ldots, \mathbf{a}_r^T \mathbf{z}$ and to avoid trivial or repeated solutions we add constraints that require them to have unit variance and to be pairwise uncorrelated:

\begin{align} \label{pfa}
\text{For} \; i \in \{1, \ldots, r\} \notag \\
\begin{split} %{@{}r@{\quad}ccrr@{}}
\optmin{\mathbf{a}_i \in \mathbb{R}^{n}, \mathbf{b} \in \mathbb{R}^{p}} & \av{\norm{\mathbf{a}_i^T (\mathbf{z} - \hat{\mathbf{z}}_\mathbf{b})}^2} \\
% \av{\dot{}\dot{\mathbf{z}}^T} \mathbf{a}_i\\
\subjectto	& \mathbf{a}_i^T \av{\mathbf{z}} \hphantom{\mathbf{a}_i \mathbf{a}_j \mathbf{z}^T} \, = \quad 0 \quad \hphantom{\forall \; j < i} \quad \text{(zero mean)}\\
& \mathbf{a}_i^T \av{\mathbf{z} \mathbf{z}^T} \mathbf{a}_i \hphantom{\mathbf{a}_j} = \quad 1 \quad \hphantom{\forall \; j < i} \quad \text{(unit variance)}\\
& \mathbf{a}_i^T \av{\mathbf{z} \mathbf{z}^T} \mathbf{a}_j \hphantom{\mathbf{a}_i} = \quad  0 \quad \forall \; j < i \quad \text{(pairwise decorrelation)}
\end{split}
\end{align}

Because of the sphering $\av{\mathbf{z}} = 0$ and $\av{\mathbf{z} \mathbf{z}^T} = \id$, the constraints simplify to
\begin{equation} \label{PFA-constraints-no-matrix}
\mathbf{a}_i^T \mathbf{a}_j = \delta_{ij}
\end{equation}
With $\mathbf{A}_r~\coloneq~\left(\mathbf{a}_1,~\ldots,~\mathbf{a}_r~\right)~\in~\mathbb{R}^{n~\times~r}$, constraint \eqref{PFA-constraints-no-matrix} is automatically fulfilled if we choose
\begin{equation} \label{PFA-constraints-matrix}
\mathbf{A}_r \quad = \quad \mathbf{A}\mathbf{I}_r \quad \text{with} \quad \mathbf{A} \in \orth(n)
\end{equation}
$\orth(n) \subset \mathbb{R}^{n \times n}$ denotes the space of orthogonal transformations, i.e.\ $\mathbf{A}\mathbf{A}^T = \mathbf{I}$ and $\mathbf{I}_r \in \mathbb{R}^{n \times r}$ denotes the reduced identity matrix consisting of the first $r$ Euclidean unit vectors as columns.

Problem \eqref{pfa} is not readily solvable. As a prerequisite for a solvable relaxation we define \linebreak $\mathbf{m}(t) \coloneq \mathbf{A}_r^T \mathbf{z}(t)$ and extend the prediction model to matrix notation \eqref{pfa-predictingWithB}. To keep things compact we directly switch to the PFAx notion by incorporating supplementary information $\mathbf{u}$ in \eqref{pfa-predictingWithU}. 

\begin{align} \label{pfa-predictingWithB}
\mathbf{m}(t) \quad \appr^! &\quad \mathbf{B}_{1} \hphantom{\mathbf{U}\mathbf{u}}\!\!\!\!\!\!\! \mathbf{m}(t-1)  + \; \ldots \; + \mathbf{B}_{p} \hphantom{\mathbf{U_q}\mathbf{u}}\!\!\!\!\!\!\!\!\!\! \mathbf{m}(t-p) \qquad \text{with}\quad \mathbf{B}_i \hphantom{\mathbf{U}}\!\!\!\!\! \in \mathbb{R}^{r \times r} \\
\label{pfa-predictingWithU}
+ &\quad \mathbf{U}_{1} \hphantom{\mathbf{B}\mathbf{m}}\!\!\!\!\!\!\! \mathbf{u}(t-1)  + \; \ldots \; + \mathbf{U}_{q} \hphantom{\mathbf{B_p}\mathbf{m}}\!\!\!\!\!\!\!\!\!\! \mathbf{u}(t-q) \qquad \text{with}\quad \mathbf{U}_i \hphantom{\mathbf{B}}\!\!\!\!\! \in \mathbb{R}^{r \times n_{\mathbf{u}}} \\
\label{pfa-predictingWithBU}
= &\quad \mathbf{B} \; \underbrace{\mvec(\hist{\mathbf{z}}{p}(t))}_{\eqcolon \; \mathbf{\zeta}(t)} \;\; + \;\; \mathbf{U} \; \underbrace{\mvec(\hist{\mathbf{u}}{q}(t))}_{\eqcolon \; \mathbf{\mu}(t)}
%\; \eqcolon \; \mathbf{A}_r^T \hat{\mathbf{z}}_{\mathbf{B}, \mathbf{U}}(t)
\end{align}
\eqref{pfa-predictingWithBU} uses block matrix notation $\mathbf{B}~=(\mathbf{B}_1, \ldots, \mathbf{B}_p)~\in~\mathbb{R}^{r \times rp}$ and $\mathbf{U}~=(\mathbf{U}_1, \ldots, \mathbf{U}_q)~\in~\mathbb{R}^{r \times n_{\mathbf{u}}q}$.

\begin{figure}[!ht]
	\centering
	\captionsetup{width=.8\linewidth}
	\includegraphics[width=0.95\hsize]{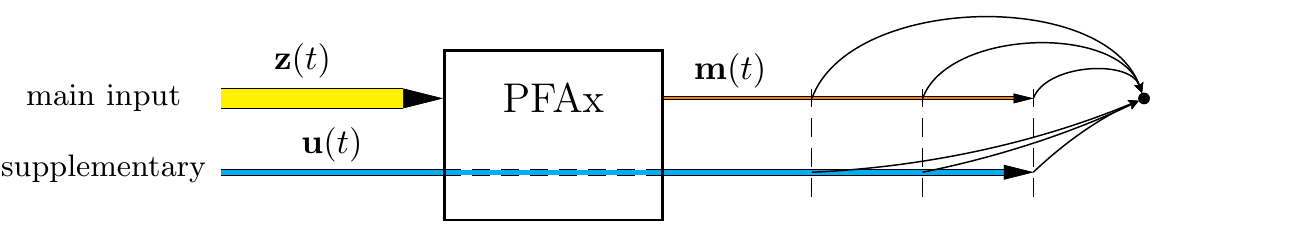}
	\caption{Illustration of PFAx. Components are selected to be well predictable if supplementary information is taken into account.}
\end{figure}

In \eqref{pfaNonAutoRegressiveWhitened} we will state the PFAx optimization problem in terms of \eqref{pfa-predictingWithBU}, but that formulation requires a formula to express the prediction matrices $\mathbf{B}$ and $\mathbf{U}$ in terms of the extraction matrix $\mathbf{A}_r$.
Matrix calculus allows us to compute the ideal prediction matrices for a given extraction $\mathbf{A}_r$:

\begin{align} \label{pfa-fittingB}
\mathbf{B}(\mathbf{A}_r) \quad \coloneq &\quad \Big( \mathbf{A}_r^T \av{\mathbf{z} \mathbf{\zeta}^T} - \mathbf{U}(\mathbf{A}_r) \av{\mathbf{\mu} \mathbf{\zeta}^T} \Big) \underline{\mathbf{A}_r} \;\; \Big( \underline{\mathbf{A}_r^T} \av{\mathbf{\zeta} \mathbf{\zeta}^T} \underline{\mathbf{A}_r} \Big)^{-1} \\
\mathbf{U}(\mathbf{A}_r) \quad \coloneq &\quad \Big( \mathbf{A}_r^T \av{\mathbf{z} \mathbf{\mu}^T} - \mathbf{B}(\mathbf{A}_r) \; \underline{\mathbf{A}_r^T} \av{\mathbf{\zeta} \mathbf{\mu}^T} \Big) \; \av{\mathbf{\mu} \mathbf{\mu}^T}^{-1} \label{pfa-fittingU}
\end{align}

This uses the shortcut notation defined for any matrix $\mathbf{M} \; \in \; \mathbb{R}^{n \times m}$:
\begin{equation} \label{multiA}
\underline{\mathbf{M}} \quad \coloneq \quad \mathbf{I}_{p, p} \otimes \mathbf{M} \quad = \qquad \underbrace{\!\!\!\!\!\!\left( \begin{smallmatrix} \mathbf{M}&  & \mathbf{0} \\  & \ddots &  \\ \mathbf{0} &  & \mathbf{M} \end{smallmatrix} \right)\!\!\!\!\!\!}_{\text{$p$ times $\mathbf{M}$}} \qquad \in \; \mathbb{R}^{np \times mp}
\end{equation}

Equations \eqref{pfa-fittingB} and \eqref{pfa-fittingU} are derived as follows. For a given $\mathbf{A}_r$ the optimal $\mathbf{B}$, $\mathbf{U}$ must solve
\begin{equation} \label{BU-ReducedDimRankTransformSolution-optProblem}
\optmin{\subalign{ \qquad\mathbf{B} &\in \mathbb{R}^{r \times rp} \\ \qquad\mathbf{U} &\in \mathbb{R}^{n_{\mathbf{u}} \times n_{\mathbf{u}}q}}} \;
\pfaerrob{\mathbf{B} \underline{\mathbf{A}_r^T} \zeta - \mathbf{U} \mu}{\mathbf{A}_r^T \mathbf{z}} \;\quad \eqcolon \quad f(\mathbf{B}, \mathbf{U})
\end{equation}
Writing $f(\mathbf{B}, \mathbf{U}) \;=\; \tr \Big( \bigav{ \big( \mathbf{A}_r^T \mathbf{z} - \mathbf{B} \underline{\mathbf{A}_r^T} \zeta - \mathbf{U} \mu \big) \big( \mathbf{A}_r^T \mathbf{z} - \mathbf{B} \underline{\mathbf{A}_r^T} \zeta - \mathbf{U} \mu \big)^T } \Big)$ we can expand $f$ to
\begin{align} \label{BU-ReducedDimRankTransformSolution-expand}
\quad
f(\mathbf{B}, \mathbf{U}) \;
= \; \tr \Big( \; \hphantom{-} &&\mathbf{A}_r^T \av{\mathbf{z} \mathbf{z}^T} \mathbf{A}_r
\quad-&& \mathbf{A}_r^T \av{\mathbf{z} \mathbf{\zeta}^T} \underline{\mathbf{A}_r} \mathbf{B}^T
\quad-&& \mathbf{A}_r^T \av{\mathbf{z} \mathbf{\mu}^T} \mathbf{U}^T
\nonumber \; \hphantom{\Big)} \quad \\
- && \mathbf{B} \underline{\mathbf{A}_r^T} \av{\mathbf{\zeta} \mathbf{z}^T} \mathbf{A}_r
\quad+&& \mathbf{B} \underline{\mathbf{A}_r^T} \av{\mathbf{\zeta} \mathbf{\zeta}^T} \underline{\mathbf{A}_r} \mathbf{B}^T
\quad+&& \mathbf{B} \underline{\mathbf{A}^T_r} \av{\mathbf{\zeta} \mathbf{\mu}^T} \mathbf{U}^T
\; \hphantom{\Big)} \quad \\
- && \mathbf{U} \av{\mathbf{\mu} \mathbf{z}^T} \mathbf{A}_r
\quad+&& \mathbf{U} \av{\mathbf{\mu} \mathbf{\zeta}^T} \underline{\mathbf{A}_r} \mathbf{B}^T
\quad+&& \mathbf{U} \av{\mathbf{\mu} \mathbf{\mu}^T} \mathbf{U}^T \; \Big)
\nonumber \quad
\end{align}
and set its matrix derivatives to zero:
\begin{align} \label{B-ReducedDimRankTransformSolution-derive}
\frac{\partial}{\partial \mathbf{B}} f(\mathbf{B}, \mathbf{U}) \quad &= \quad -\; 2 \mathbf{A}_r^T \av{\mathbf{z} \mathbf{\zeta}^T} \underline{\mathbf{A}_r} \;+\; 2 \mathbf{U} \av{\mathbf{\mu} \mathbf{\zeta}^T} \underline{\mathbf{A}_r} \;+\; 2 \mathbf{B} \underline{\mathbf{A}_r^T} \av{\mathbf{\zeta} \mathbf{\zeta}^T} \underline{\mathbf{A}_r} & \eq^! \quad \mathbf{0} \\ \label{U-ReducedDimRankTransformSolution-derive}
\frac{\partial}{\partial \mathbf{U}} f(\mathbf{B}, \mathbf{U}) \quad &= \quad -\; 2 \mathbf{A}_r^T \av{\mathbf{z} \mathbf{\mu}^T} \hphantom{\underline{\mathbf{A}_r}} \;+\; 2 \mathbf{B} \underline{\mathbf{A}^T_r} \av{\mathbf{\zeta} \mathbf{\mu}^T} \;+\; 2 \mathbf{U} \av{\mathbf{\mu} \mathbf{\mu}^T} \hphantom{\underline{\mathbf{A}_r}\underline{\mathbf{A}_r^T}} & \eq^! \quad \mathbf{0}
\end{align}
Solving \eqref{B-ReducedDimRankTransformSolution-derive} for $\mathbf{B}$ yields \eqref{pfa-fittingB} and solving \eqref{U-ReducedDimRankTransformSolution-derive} for $\mathbf{U}$ yields \eqref{pfa-fittingU}.
In \eqref{pfa-fittingB} and \eqref{pfa-fittingU}, $\mathbf{B}(\mathbf{A}_r)$ and $\mathbf{U}(\mathbf{A}_r)$ are defined implicitly. By inserting \eqref{pfa-fittingU} into \eqref{pfa-fittingB} and solving for $\mathbf{B}(\mathbf{A}_r)$ we get the explicit formula
\begin{equation} \label{pfa-fittingBUExplicit}
\mathbf{B}(\mathbf{A}_r) = \mathbf{A}_r^T \Big( \av{\mathbf{z} \mathbf{\zeta}^T} - \av{\mathbf{z} \mathbf{\mu}^T} \av{\mathbf{\mu} \mathbf{\mu}^T}^{-1} \av{\mathbf{\mu} \mathbf{\zeta}^T} \Big) \underline{\mathbf{A}_r} \; \Big( \underline{\mathbf{A}_r^T} \left( \av{\mathbf{\zeta} \mathbf{\zeta}^T} - \av{\mathbf{\zeta} \mathbf{\mu}^T} \av{\mathbf{\mu} \mathbf{\mu}^T}^{-1} \av{\mathbf{\mu} \mathbf{\zeta}^T} \right) \underline{\mathbf{A}_r} \Big)^{-1}
\end{equation}

If $\av{\mathbf{\mu}\mathbf{\mu}^T}$ is not (cleanly) invertible due to very small or zero-valued eigenvalues, we recommend to project away the eigenspaces corresponding to eigenvalues below a critical threshold. These indicate redundancies in the signal and can therefore be dropped: In an eigenvalue decomposition of $\av{\mathbf{\mu}\mathbf{\mu}^T}$ replace eigenvalues below the threshold by $0$ and invert the others. Use the resulting matrix as a proxy for $\av{\mathbf{\mu}\mathbf{\mu}^T}^{-1}$. Proceed equivalently with other matrices where arising inversions are not computable due to near-zero-eigenvalues.

We define the ideal linear predictor $\hat{\mathbf{z}}^{(0)}$ for the original signal without extraction, i.e.\ $\mathbf{A}_r = \mathbf{I}$:
\begin{equation} \label{pfaBUpredictor}
\hat{\mathbf{z}}^{(0)}(t) \quad \coloneq \quad \mathbf{B}(\mathbf{I}) \; \mathbf{\zeta}(t) \;\; + \;\; \mathbf{U}(\mathbf{I}) \; \mathbf{\mu}(t)
\end{equation}
and can now refine problem \eqref{pfa} to
\begin{equation} \label{pfaNonAutoRegressiveWhitened}
\optmin{\mathbf{A} \in \orth(n)} \; \quad \bigav{\; \norm{\mathbf{A}_r^T(\mathbf{z} - \hat{\mathbf{z}}^{(0)})}^2 \; } \quad = \quad  \tr \Big( \mathbf{A}_r^T \bigav{\; (\mathbf{z} - \hat{\mathbf{z}}^{(0)}) (\mathbf{z} - \hat{\mathbf{z}}^{(0)})^T \; } \mathbf{A}_r \Big)
\end{equation}
which can be solved by choosing $\mathbf{A}$ such that it diagonalizes
$\bigav{\; (\mathbf{z} - \hat{\mathbf{z}}^{(0)}) (\mathbf{z} - \hat{\mathbf{z}}^{(0)})^T \; }$
and sorts the $r$ smallest eigenvalues to the upper left. From this we obtain the prediction model by calculating $\mathbf{B}_{\mathbf{z}}(\mathbf{A}_r)$  and $\mathbf{U}_{\mathbf{z}}(\mathbf{A}_r)$.

In \cite{DBLP:conf/icmla/RichthoferW15}, we proposed an iterated prediction as a heuristic method to better avoid overfitting. In \cite{2017arXiv171200634R} we extended this method to comply with supplementary information as follows.
We define a matrix $\mathbf{V}$ which implements the autoregressive model and predicts $\mathbf{\zeta}(t+1)$ from $\mathbf{\zeta}(t)$:
\begin{align}
\mathbf{V} \hphantom{(t)}\;\;\, \quad \coloneq &\quad \Big( \av{\mathbf{\zeta}(t+1) \mathbf{\zeta}^T} - \mathbf{I}_{np,n}\mathbf{U}(\mathbf{I}) \av{\mathbf{\mu} \mathbf{\zeta}^T} \Big) \av{\mathbf{\zeta} \mathbf{\zeta}^T}^{-1} \\
\hat{\mathbf{z}}^{(i)}(t) \quad \coloneq &\quad \mathbf{B}(\mathbf{I})\mathbf{V}^i \mathbf{\zeta}(t-i) + \mathbf{I}^T_{np, n} \sum^i_{j=0} \mathbf{V}^j \mathbf{I}_{np,n} \mathbf{U}(\mathbf{I}) \mathbf{\mu}(t-j)
\end{align}
Note that $\hat{\mathbf{z}}^{(i)}$ is consistent with $\hat{\mathbf{z}}^{(0)}$ in \eqref{pfaBUpredictor} for $i=0$. Based on $\hat{\mathbf{z}}^{(i)}$ we proposed the optimization problem
\begin{equation} \label{ipfaNonAutoRegressiveWhitenedRewrite}
\optmin{\mathbf{A} \in \orth(n)} \sum_{i=0}^k \bigav{\; \norm{\mathbf{A}_r^T(\mathbf{z}-\hat{\mathbf{z}}^{(i)})}^2 \; } \quad = \quad \tr \Bigg( \mathbf{A}_r^T \sum_{i=0}^k \bigav{\big(\mathbf{z} - \hat{\mathbf{z}}^{(i)}\big) \big(\mathbf{z} - \hat{\mathbf{z}}^{(i)}\big)^T} \mathbf{A}_r \Bigg)
\end{equation}
It can be solved by the same procedure as \eqref{pfaNonAutoRegressiveWhitened}:
Choose $\mathbf{A}$ such that it diagonalizes $\sum_{i=0}^k \bigav{\left(\mathbf{z} - \hat{\mathbf{z}}^{(i)}\right) \left(\mathbf{z} - \hat{\mathbf{z}}^{(i)}\right)^T}$ and sort the lowest $r$ eigenvalues to the upper left. Apply $\mathbf{B}_{\mathbf{z}}(\mathbf{A}_r)$  and $\mathbf{U}_{\mathbf{z}}(\mathbf{A}_r)$ to get the prediction matrices for the obtained extraction matrix.

\subsection{Generating a control signal for local navigation} \label{sec:controlPFA}

Considering an agent exploring an environment, we present the agent's perception as main input $\mathbf{z}$ to PFAx and provide the preceding control command as supplementary information $\mathbf{u}$. This way the extracted predictable features will be a compact representation of perception aspects that are influenced by the control commands. We assume that the control signal is somehow generated during training phase, e.g.\ randomly within some constraints.
After training phase, PFAx provides $\mathbf{A}_r$, $\mathbf{B}(\mathbf{A}_r)$ and $\mathbf{U}(\mathbf{A}_r)$ and we want to reach a goal position $\mathbf{m}^*$ in feature space, assuming that feature space is sufficiently representative to let us actually reach the associated goal in our environment. As explained earlier, we will apply SFA on top of PFAx, so in contrast to the original PFAx setting, we must consider an additional extraction matrix $\mathbf{A}_{\text{SFA}}$. This is a true generalization as $\mathbf{A}_{\text{SFA}} = \id$ yields the original setting. Note that SFA also incorporates a mean shift, which is omitted here for simplicity and considering that PFAx output should be already mean free. Further more this would only result in a shift component for the cost function and is thus irrelevant for optimization.
To calculate the ideal control command we minimize the least square distance between predicted features and goal features w.r.t.\ a proceeding linear SFA step:

\begin{align} \label{ipfaOptControl}
\optmin{\mathbf{u}(t) \; \in \; \mathbb{R}^{n_{\mathbf{u}}}} &\norm{\mathbf{m}^* - \mathbf{A}_{\text{SFA}}^T \big( \mathbf{B}(\mathbf{A}_r) \mathbf{A}_r^T \mathbf{\zeta}(t+1) - \mathbf{U}(\mathbf{A}_r) \mathbf{\mu}(t+1) \big)}^2 \\
= \quad &\norm{\underbrace{\mathbf{m}^* - \mathbf{A}_{\text{SFA}}^T \Big( \mathbf{B}(\mathbf{A}_r) \mathbf{A}_r^T \mathbf{\zeta}(t+1) - \Big(\sum_{j=2}^q \mathbf{U}_j(\mathbf{A}_r) \mathbf{u}(t-j+1)\Big) \Big)}_{\eqcolon \; \mathbf{u}^* \; \in \; \mathbb{R}^{r}} - \mathbf{A}_{\text{SFA}}^T \mathbf{U}_1 \mathbf{u}(t)}^2 \\
= \quad &\norm{\mathbf{u}^*  - \underbrace{\mathbf{A}_{\text{SFA}}^T \mathbf{U}_1}_{
		\eqcolon \; \tilde{\mathbf{U}}_1
		%r \times n_{\mathbf{u}}
		} \mathbf{u}(t)}^2
\end{align}
This problem is readily solved by choosing $\mathbf{u}(t) \coloneq \tilde{\mathbf{U}}_1^{-1} \mathbf{u}^*$ (or $\mathbf{u}(t)~\coloneq~(\tilde{\mathbf{U}}_1^T\tilde{\mathbf{U}}_1)^{-1} \tilde{\mathbf{U}}_1^T \mathbf{u}^*$, if $\tilde{\mathbf{U}}_1$ is not square or not invertible).
Note that this would also minimize $\norm{\mathbf{u}^*  - \tilde{\mathbf{U}}_1 \mathbf{u}(t)}$. However, to incorporate constraints on $\mathbf{u}$, the squared distance is much friendlier for optimization.

\begin{figure}[ht]
	\centering
	\captionsetup{width=.8\linewidth}
	\includegraphics[width=0.9\hsize]{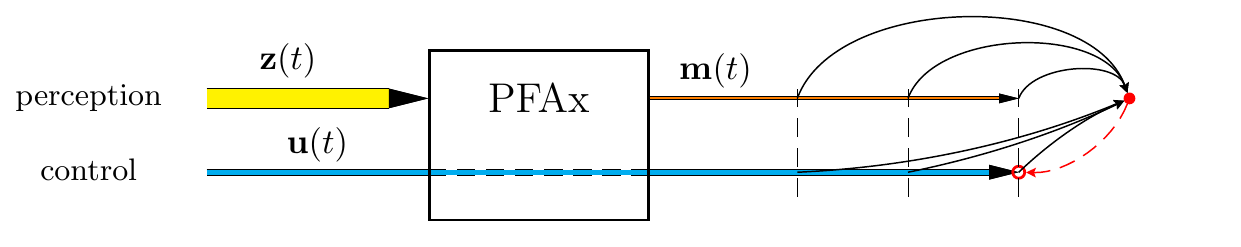}
	\caption{Illustration of controlling predictable features. The relation of control signal and prediction is inverted to obtain the control command that most likely yields the desired outcome.}
\end{figure}

Later we will model an agent moving with constant speed, which involves a normalized-length-constraint:
\begin{equation}
\optmin{\substack{\mathbf{u}(t) \; \in \; \mathbb{R}^{n_{\mathbf{u}}}\\ \norm{\mathbf{u}(t)} \; = \; c}} \norm{\mathbf{u}^*  - \tilde{\mathbf{U}}_1 \mathbf{u}(t)}^2
\end{equation}
This is equivalent to the inhomogeneous eigenvalue problem
\begin{align} \label{ipfaOptControlInhomEVP}
\tilde{\mathbf{U}}_1^T \tilde{\mathbf{U}}_1 \mathbf{u}(t) \quad = \quad &\lambda \mathbf{u}(t) + \tilde{\mathbf{U}}_1^T \mathbf{u}^* \\
\norm{\mathbf{u}(t)} \quad = \quad &c
\end{align}
In \cite{MATTHEIJ1987507} such problems are approached. One method from there can also be found in the appendix of \cite{2017arXiv171200634R}. In that work we provide some experiments indicating that this method is suitable for local navigation, but cannot readily navigate its way globally, e.g.\ around obstacles or through doors connecting multiple rooms. The following section extends this method such that it is capable of solving these kind of global navigation tasks.

\section{From local to global navigation} \label{sec:localToGlobal}

To achieve global navigation, the Slow Feature Analysis algorithm (SFA) \cite{WiskottSejnowski-2002} and its extension xSFA \cite{SprekelerZitoEtAl-2014} for blind source separation play an important role. Especially the mathematical foundation of xSFA, which is grounded on the mathematical analysis of SFA in \cite{SprekelerWiskott-2008} forms a key component for the navigation approach presented here. So we first comprehend the original SFA algorithm and then sketch its mathematical foundation, also stating key results of the theory that xSFA is based on. Finally we apply these results to our navigation setting, yielding an efficient algorithm for global navigation.

\subsection{Recall SFA} \label{sec:sfa}
Like PFA selects components by predictability, SFA selects them by slowness. As it was a central inspiration for PFA, SFA has some more similarities to it:
The extraction is also optimized over a training phase $\trph$ and to avoid trivial/constant or repeated solutions, the output signals must have unit variance, zero mean and must be pairwise uncorrelated. We refer to the transformation as $\mathbf{g}(\mathbf{x})$, i.e.\ the $i$th extracted signal is given as $\mathbf{y}_i(t)\coloneq\mathbf{g}_i(\mathbf{x(t)})$. Note that these depend instantaneously on the input signal $\mathbf{x}(t)$, so SFA cannot just fulfill its goal by forming a lowpass filter. Operating on a general function space $\mathcal{F}$ that fulfills the necessary mathematical requirements of integrability and differentiability, the SFA optimization problem can be formulated as follows:
\begin{align} \label{sfa-F}
\text{For} \; i \in \{1, \ldots, r\} \notag \\
\begin{split} %{@{}r@{\quad}ccrr@{}}
\optmin{\mathbf{g}_i \in \mathcal{F}} & \av{\dot{\mathbf{y}}_i^2} \\
\subjectto	& \av{\mathbf{y}_i}
\hphantom{\mathbf{y}_i^2 \mathbf{y}_j} \, \!\!
= \quad 0 \quad \hphantom{\forall \; j < i} \quad \text{(zero mean)}\\
& \av{\mathbf{y}_i^2}
\hphantom{\mathbf{y}_i \mathbf{y}_j} \!\!
= \quad 1 \quad \hphantom{\forall \; j < i} \quad \text{(unit variance)}\\
& \av{\mathbf{y}_i \mathbf{y}_j }
\hphantom{ \mathbf{y}_i^2} \, \!\!
= \quad  0 \quad \forall \; j < i \quad \text{(pairwise decorrelation)}
\end{split}
\end{align}

Restricting $\mathcal{F}$ to be finite dimensional, e.g.\ to the space of polynomials up to a certain degree, transforms \eqref{sfa-F} into an efficiently solvable eigenvalue problem. With the notation familiar from the PFA description in section \ref{sec:pfax}, let $\mathbf{h}$ denote a basis of $\mathcal{F}$. Then using $\mathbf{h}$ as a nonlinear expansion on the input signal $\mathbf{x}$, extraction can be performed by linear transformation and projection. With an initial sphering, i.e.\ \eqref{sphering1} and \eqref{sphering2} from section \ref{sec:pfax} we can set $\mathbf{g}_i(\mathbf{x(t)})~\coloneq~\mathbf{a}_i^T\mathbf{z}(t)$ for extraction vectors $\mathbf{a}_i~\in~\mathbb{R}^{n}$. SFA then becomes the following linearized version of \eqref{sfa-F}:
\begin{align} \label{sfa}
\text{For} \; i \in \{1, \ldots, r\} \notag \\
\begin{split} %{@{}r@{\quad}ccrr@{}}
\optmin{\mathbf{a}_i \in \mathbb{R}^{n}} & \mathbf{a}_i^T \av{\dot{\mathbf{z}}\dot{\mathbf{z}}^T} \mathbf{a}_i\\
\subjectto	& \mathbf{a}_i^T \av{\mathbf{z}} \hphantom{\mathbf{a}_i \mathbf{a}_j \mathbf{z}^T} \, = \quad 0 \quad \hphantom{\forall \; j < i} \quad \text{(zero mean)}\\
& \mathbf{a}_i^T \av{\mathbf{z} \mathbf{z}^T} \mathbf{a}_i \hphantom{\mathbf{a}_j} = \quad 1 \quad \hphantom{\forall \; j < i} \quad \text{(unit variance)}\\
& \mathbf{a}_i^T \av{\mathbf{z} \mathbf{z}^T} \mathbf{a}_j \hphantom{\mathbf{a}_i} = \quad  0 \quad \forall \; j < i \quad \text{(pairwise decorrelation)}
\end{split}
\end{align}
Like in section \ref{sec:pfax}, the sphering yields $\av{\mathbf{z}} = 0$ and $\av{\mathbf{z} \mathbf{z}^T} = \id$, transforming the constraints to \eqref{PFA-constraints-no-matrix} and its associated matrix notation \eqref{PFA-constraints-matrix}.
Choosing $\mathbf{a}_i$ as eigenvectors of $\av{\dot{\mathbf{z}}\dot{\mathbf{z}}^T}$, corresponding to the eigenvalues in ascending order, yields $\mathbf{A}_r$ solving \eqref{sfa} globally. \cite{WiskottSejnowski-2002} describes this procedure in detail.
%In the following $\mathbf{m} \coloneq \mathbf{A}_r^T \mathbf{z}$ denotes the extracted signal.

\eqref{sfa} is an important approximation \eqref{sfa-F} for practical solvability. To get an idea to what solutions \eqref{sfa} would converge if we increase the dimension of $\mathbf{h}$, we focus again on the SFA version concerning an unrestricted function space $\mathcal{F}$ and the ideal solutions one would expect there. More specifically, we focus on the scenario where $\mathbf{x(t)}$ is composed of statistically independent sources $\mathbf{s}_{\alpha}$. \cite{SprekelerWiskott-2008} and \cite{SprekelerZitoEtAl-2014} analyze this case, proposing xSFA as an extension to SFA that can identify such sources. We comprehend some theory and results:

Assuming that $\mathbf{x(t)}$ is an ergodic process, SFA can be formulated in terms of the ensemble (i.e.\ the set of possible values of $\mathbf{x}$ and $\dot{\mathbf{x}}$) using the probability density $p_{\mathbf{x}, \dot{\mathbf{x}}}(\mathbf{x}, \dot{\mathbf{x}})$. The corresponding marginal and conditional densities are defined as $p_{\mathbf{x}}(\mathbf{x})~\!\!\coloneq~\!\!\int p_{\mathbf{x}, \dot{\mathbf{x}}}(\mathbf{x}, \dot{\mathbf{x}}) d^n \dot{x}$ and $p_{\dot{\mathbf{x}} | \mathbf{x}}(\dot{\mathbf{x}} | \mathbf{x})~\!\!\coloneq~\!\!\frac{p_{\mathbf{x}, \dot{\mathbf{x}}}(\mathbf{x}, \dot{\mathbf{x}})}{p_{\mathbf{x}}(\mathbf{x})}$.
Further assuming that the ensemble averages $\av{f(\mathbf{x}, \dot{\mathbf{x}})}_{\mathbf{x}, \dot{\mathbf{x}}}~\!\!\coloneq~\!\!\int p_{\mathbf{x}, \dot{\mathbf{x}}}(\mathbf{x}, \dot{\mathbf{x}}) f(\mathbf{x}, \dot{\mathbf{x}}) d^n x d^n \dot{x}$,
$\av{f(\mathbf{x})}_{\mathbf{x}}~\!\!\coloneq~\!\!\int p_{\mathbf{x}}(\mathbf{x}) f(\mathbf{x}) d^n x$ and $\av{f(\mathbf{x}, \dot{\mathbf{x}})}_{\dot{\mathbf{x}} | \mathbf{x}}(\mathbf{x})~\!\!\coloneq~\!\!\int p_{\dot{\mathbf{x}} | \mathbf{x}}(\dot{\mathbf{x}}|\mathbf{x}) f(\mathbf{x}, \dot{\mathbf{x}}) d^n \dot{x}$ all exist and using the chain rule, the SFA optimization problem can be stated in terms of the ensemble as well:
\begin{align} \label{sfa-ensemble}
\text{For} \; i \in \{1, \ldots, r\} \notag \\
\begin{split} %{@{}r@{\quad}ccrr@{}}
\optmin{\mathbf{g}_i \in \mathcal{F}} & \sum_{\gamma, \nu} \av{\partial_{\gamma}\mathbf{g}_i(\mathbf{x}) \av{\dot{\mathbf{x}}_{\gamma} \dot{\mathbf{x}}_{\nu}}_{\dot{\mathbf{x}} | \mathbf{x}} \partial_{\nu}\mathbf{g}_i(\mathbf{x})}_{\mathbf{x}} \\
\subjectto	& \av{\mathbf{g}_i(\mathbf{x})}_{\mathbf{x}}
\hphantom{\mathbf{g}_i^2(\mathbf{x}) \mathbf{g}_j(\mathbf{x})} \, \!\!
= \quad 0 \quad \hphantom{\forall \; j < i} \quad \text{(zero mean)}\\
& \av{\mathbf{g}_i^2(\mathbf{x})}_{\mathbf{x}}
\hphantom{\mathbf{g}_i(\mathbf{x}) \mathbf{g}_j(\mathbf{x})} \!\!
= \quad 1 \quad \hphantom{\forall \; j < i} \quad \text{(unit variance)}\\
& \av{\mathbf{g}_i(\mathbf{x}) \mathbf{g}_j(\mathbf{x}) }_{\mathbf{x}}
\hphantom{ \mathbf{g}_i^2(\mathbf{x})} \, \!\!
= \quad  0 \quad \forall \; j < i \quad \text{(pairwise decorrelation)}
\end{split}
\end{align}
A key result from \cite{SprekelerWiskott-2008} is that the ideal solutions for SFA on an unrestricted function space can be found by solving the following eigenvalue equation given the partial differential operator $\mathcal{D} \coloneq -\frac{1}{p_{\mathbf{x}}(\mathbf{x})} \sum_{\gamma, \nu} \partial_{\gamma}p_{\mathbf{x}}(\mathbf{x}) \av{\dot{\mathbf{x}}_{\gamma} \dot{\mathbf{x}}_{\nu}}_{\dot{\mathbf{x}} | \mathbf{x}}(\mathbf{x}) \partial_{\nu}$:

\begin{equation} \label{dgl-x}
\mathcal{D} \mathbf{g}_i(\mathbf{x}) \quad = \quad \lambda_i \mathbf{g}_i(\mathbf{x})
\end{equation}
under the von Neumann boundary conditions
\begin{equation} \label{neumann-x}
\sum_{\gamma, \nu} n_{\gamma}(\mathbf{x}) p_{\mathbf{x}}(\mathbf{x}) \av{\dot{\mathbf{x}}_{\gamma} \dot{\mathbf{x}}_{\nu}}_{\dot{\mathbf{x}} | \mathbf{x}}(\mathbf{x}) \partial_{\nu} \mathbf{g}_i(\mathbf{x})
\end{equation}
where $n_{\gamma}(\mathbf{x})$ is the $\gamma$th component of the normal vector at the boundary point $\mathbf{x}$. Assuming the input signal $\mathbf{x}(t)$ is composed of statistically independent sources $\mathbf{s}_{\alpha}$ for $\alpha \in \{1, \ldots, S \}$, this result can be formulated in terms of the sources. Because of statistical independence we have $p_{\mathbf{s}, \dot{\mathbf{s}}}(\mathbf{s}, \dot{\mathbf{s}}) = \prod_{\alpha}  p_{\mathbf{s}_{\alpha}, \dot{\mathbf{s}}_{\alpha}}(\mathbf{s}_{\alpha}, \dot{\mathbf{s}}_{\alpha})$,
$p_{\mathbf{s}}(\mathbf{s}) = \prod_{\alpha}  p_{\mathbf{s}_{\alpha}}(\mathbf{s}_{\alpha})$ and
$\av{\dot{\mathbf{s}}_{\alpha} \dot{\mathbf{s}}_{\beta}}_{\dot{\mathbf{s}} | \mathbf{s}}(\mathbf{s}) = \delta_{\alpha \beta} \av{\dot{\mathbf{s}}_{\alpha}^2}_{\dot{\mathbf{s}}_{\alpha} | \mathbf{s}_{\alpha}}(\mathbf{s}_{\alpha})$.
 $\mathcal{D}(\mathbf{s})$ can be decomposed as
\begin{equation}
\mathcal{D} (\mathbf{s}) \quad = \quad \sum_{\alpha} \mathcal{D}_{\alpha} (\mathbf{s}_{\alpha})
\end{equation}
Regarding this decomposition, \eqref{dgl-x} and \eqref{neumann-x} can be reformulated such that, with an additional normalization constraint, the following equations formulate SFA in terms of the sources:
\begin{align}
\mathcal{D}_{\alpha} \mathbf{g}_{\alpha i} \quad &= \quad \lambda_{\alpha i} \mathbf{g}_{\alpha i} \label{dgl-s}\\ 
p_{\alpha} \av{\dot{\mathbf{s}}_{\alpha}^2}_{\dot{\mathbf{s}}_{\alpha} | \mathbf{s}_{\alpha}} \partial_{\alpha} \mathbf{g}_{\alpha i} \quad &= \quad 0 \quad\quad\quad \text{on the boundary}\label{neumann-s} \\ 
\av{\mathbf{g}_{\alpha i}^2}_{\mathbf{s}_{\alpha}} \quad &= \quad 1
\end{align}
\textsl{Theorem 2} in \cite{SprekelerWiskott-2008} / \textsl{Theorem 1} in \cite{SprekelerZitoEtAl-2014} states that the solutions of \eqref{dgl-x} are composed from solutions of \eqref{dgl-s}:
\begin{align}
\mathbf{g}_{\mathbf{i}}(\mathbf{s}) \quad &= \quad \prod_{\alpha} \; \mathbf{g}_{\alpha \mathbf{i}_{\alpha}}(\mathbf{s}_{\alpha}) \label{x-from-s} \\
\lambda_{\mathbf{i}} \quad &= \quad \sum_{\alpha} \; \lambda_{\alpha \mathbf{i}_{\alpha}}
\end{align}
with $\mathbf{i} = (i_1, \ldots, i_S) \in \mathbb{N}^S$ denoting a multi index to select the right combination of sources. Choosing the $r$ smallest eigenvalues $\lambda_{\mathbf{i}}$ yields the $r$ slowest output signals.

Another crucial result from \cite{SprekelerWiskott-2008} and \cite{SprekelerZitoEtAl-2014} states 
monotonicity of each first harmonic $\mathbf{g}_{\alpha 1}(\mathbf{s}_{\alpha})$ w.r.t.\ $\mathbf{s}_{\alpha}$. For later reference we denote this result as Lemma \ref{lem1} and comprehend the proof. We extend the lemma by remarking that it does not require $p_{\alpha}$ to be a probability distribution. It rather works for any strictly positive weighting function. We will make use of this fact later on.
%that for each source $\mathbf{s}_{\alpha}$ the first harmonic $\mathbf{g}_{\alpha 1}(\mathbf{s}_{\alpha})$ is a monotonic signal of the source. As this result is central for our method, we state its proof again:

\begin{lem} \label{lem1}
If $\mathbf{s}$ consists of statistically independent components $\mathbf{s}_{\alpha}$ like introduced above, then
for each source $\mathbf{s}_{\alpha}$ the first harmonic $\mathbf{g}_{\alpha 1}(\mathbf{s}_{\alpha})$ is a monotonic signal of the source $\mathbf{s}_{\alpha}$.
This also holds if the distribution $p_{\alpha}$ of $\mathbf{s}_{\alpha}$ is not a probability distribution, but any strictly positive weighting function.
\end{lem}
\begin{proof}
In standard form of a Sturm-Liouville problem and assuming that $\mathbf{s}_{\alpha}$ maps to the interval $I_{\alpha}~\!\!=~\!\![a_{\alpha}, b_{\alpha}]$, \eqref{dgl-s}/\eqref{neumann-s} are stated as
\begin{align}
\partial_{\alpha} p_{\alpha} \av{\dot{\mathbf{s}}_{\alpha}^2}_{\dot{\mathbf{s}}_{\alpha} | \mathbf{s}_{\alpha}} \partial_{\alpha} \mathbf{g}_{\alpha i} + \lambda_{\alpha i} p_{\alpha} \mathbf{g}_{\alpha i} \quad &= \quad  0 \label{dgl-SL}\\ 
p_{\alpha} \av{\dot{\mathbf{s}}_{\alpha}^2}_{\dot{\mathbf{s}}_{\alpha} | \mathbf{s}_{\alpha}} \partial_{\alpha} \mathbf{g}_{\alpha i} \quad &= \quad 0 \quad \forall \; \mathbf{s}_{\alpha} \in \{a_{\alpha}, b_{\alpha}\} \label{neumann-SL}
\end{align}
With Sturm-Liouville theory stating that $\mathbf{g}_{\alpha 1}$ has only one zero $\xi~\!\!\in~\!\!(a_{\alpha},~\!\!b_{\alpha})$ we assume that
without loss of generality
%\OE{}
$\mathbf{g}_{\alpha 1}~\!\!<~\!\!0$ for $\mathbf{s}_{\alpha}~\!\!<~\!\!\xi$ and $\mathbf{g}_{\alpha 1}~\!\!>~\!\!0$ for $\mathbf{s}_{\alpha}~\!\!>~\!\!\xi$.

\begin{align}
\eqref{dgl-SL} \quad \ra_{\hphantom{\eqref{neumann-SL}}}&\quad
\partial_{\alpha} p_{\alpha} \av{\dot{\mathbf{s}}_{\alpha}^2}_{\dot{\mathbf{s}}_{\alpha} | \mathbf{s}_{\alpha}} \partial_{\alpha} \mathbf{g}_{\alpha 1} \quad = \quad  -\underbrace{\lambda_{\alpha 1} p_{\alpha}}_{> 0} \mathbf{g}_{\alpha 1} \quad < \quad 0 \quad \forall \; \mathbf{s}_{\alpha} > \xi \\
\ra_{\hphantom{\eqref{neumann-SL}}}&\quad
p_{\alpha} \av{\dot{\mathbf{s}}_{\alpha}^2}_{\dot{\mathbf{s}}_{\alpha} | \mathbf{s}_{\alpha}} \partial_{\alpha} \mathbf{g}_{\alpha 1} \quad \text{monotonically increasing on } (\xi, b_{\alpha}] \\
\ra_{\eqref{neumann-SL}}&\quad
\! \underbrace{p_{\alpha} \av{\dot{\mathbf{s}}_{\alpha}^2}_{\dot{\mathbf{s}}_{\alpha} | \mathbf{s}_{\alpha}}}_{> 0} \partial_{\alpha} \mathbf{g}_{\alpha 1} \quad > \quad 0 \quad \text{on } (\xi, b_{\alpha}) \\
\ra_{\hphantom{\eqref{neumann-SL}}}&\quad
\partial_{\alpha} \mathbf{g}_{\alpha 1} \quad > \quad 0 \quad \text{on } (\xi, b_{\alpha}) \\
\lra_{\hphantom{\eqref{neumann-SL}}}&\quad
\mathbf{g}_{\alpha 1} \quad \text{monotonically increasing on } (\xi, b_{\alpha}] \label{xSFA-monotonicity}
\end{align}
Equivalently it holds that $\mathbf{g}_{\alpha 1}$ is monotonically increasing on $[a_{\alpha}, \xi)$, implying that $\mathbf{g}_{\alpha 1}$ is monotonically increasing on the whole interval $I_{\alpha}$.

The calculation above does not require $p_{\alpha}$ to be a probability distribution, but only to be a strictly positive weighting function.
\end{proof}

We list some additional important results from \cite{SprekelerWiskott-2008} and \cite{SprekelerZitoEtAl-2014}:

\begin{itemize}
\item If the sources are normally distributed, i.e.\ $p_{\alpha}(\mathbf{s}_{\alpha}) = \frac{1}{\sqrt{2 \pi}} e^{\frac{1}{2}\mathbf{s}_{\alpha}^2}$, then $\av{\dot{\mathbf{s}}_{\alpha}^2}_{\dot{\mathbf{s}}_{\alpha} | \mathbf{s}_{\alpha}}$ is constant and the Hermite polynomials $\Hermite_i$ yield the solutions $\mathbf{g}_{\alpha i}(\mathbf{s}_{\alpha}) = \frac{1}{\sqrt{2^i i!}} \Hermite_i(\frac{\mathbf{s}_{\alpha}}{\sqrt{2}})$ with $\lambda_{\alpha i} = \frac{i}{\av{\dot{\mathbf{s}}_{\alpha}^2}_{\dot{\mathbf{s}}_{\alpha} | \mathbf{s}_{\alpha}}}$.

\item If the sources are uniformly distributed, then $\av{\dot{\mathbf{s}}_{\alpha}^2}_{\dot{\mathbf{s}}_{\alpha} | \mathbf{s}_{\alpha}}$ is constant and the solutions are given by Sturm-Liouville theory as harmonic oscillations $\mathbf{g}_{\alpha i}(\mathbf{s}_{\alpha}) = \sqrt{2} \cos \big(i \pi \frac{\mathbf{s}_{\alpha}}{L_{\alpha}} \big)$ with $\lambda_{\alpha i}~\!\!=~\!\!\av{\dot{\mathbf{s}}_{\alpha}^2}_{\dot{\mathbf{s}}_{\alpha} | \mathbf{s}_{\alpha}} \big(\frac{\pi}{L_{\alpha}} i \big)^2$, assuming that $\mathbf{s}_{\alpha}$ takes values in the interval $[0, L_{\alpha}]$. Therefore one refers to $\mathbf{g}_{\alpha i}$ as the $i$th harmonic of the source $\mathbf{s}_{\alpha}$.
Note that in this case, all higher harmonics can be calculated from the first harmonic using the Chebyshev polynomials $\Tschebyschow_i$: $\mathbf{g}_{\alpha i} = \Tschebyschow_i(\mathbf{g}_{\alpha 1})$

\item The slowest signal found by SFA is plainly the first harmonic $\mathbf{g}_{\alpha 1}$ of the slowest source. This result is a corner stone of xSFA as it allows to clean subsequent signals from the first source. Iterating this procedure finally yields all sources.

%\item For each source $\mathbf{s}_{\alpha}$ the first harmonic $\mathbf{g}_{\alpha 1}(\mathbf{s}_{\alpha})$ is a monotonic signal of the source. We will exploit especially this characteristic as it allows us to navigate along the first harmonic of each source using local optimization to find the globally optimal value for our task.

\item \eqref{x-from-s} implies that each output component $\mathbf{g}_{\mathbf{i}}$ of SFA is a product of harmonics $\mathbf{g}_{\alpha i}$ of earlier obtained sources.
\end{itemize}

\subsubsection{SFA harmonics illustrated on a 1D random walk} \label{sec:harmonics-1D}

To illustrate the role of harmonics for one specific source, we demonstrate the theory on a simple 1D random walk on the interval $[0, 100]$. An agent starts at position $50$ and each step is chosen by a uniform distribution over the interval $[-\frac{1}{2}, \frac{1}{2}]$. Steps exceeding the left or right boundary are simply cut off.

\begin{figure}[ht]
	\centering
	\captionsetup{width=.90\linewidth}
	\includegraphics[width=1.0\hsize]{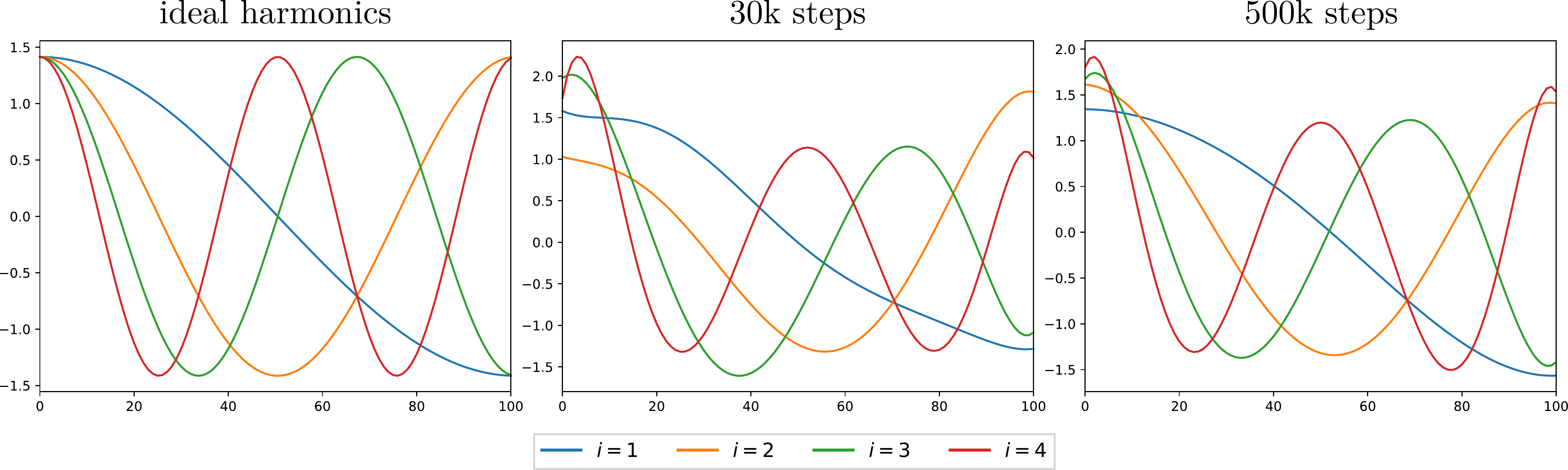}
	\caption{Illustration of the first four harmonics. The ideal harmonics predicted by the theory (left) and harmonics extracted by SFA from a 1D random walk on $[0, 100]$, i.e.\ from a single source. A longer random walk (right) approximates uniform distribution -- and thus the ideal harmonics -- better than a shorter (middle).}
	\label{fig:1D-harmonics}
\end{figure}

We provide the plain position as input to SFA, using monomials up to the sixth degree as expansion. We extract the first four harmonics and compare them to those predicted by the theory. Although a 1D random walk usually yields a normal distribution around the starting point, for long training phases and due to the boundaries, the distribution actually approaches uniformity with some bias near the boundary. Figure \ref{fig:1D-harmonics} illustrates this effect, comparing a shorter walk consisting of $30000$ steps with a longer walk consisting of $500000$ steps. Another notable effect is that lower harmonics are usually extracted cleaner than higher harmonics, which is due to the limited monomial expansion. Further note that the monotonicity of the first harmonic is mostly preserved even if the harmonic itself was not cleanly extracted. The algorithm presented in this paper benefits from this effect as it mainly exploits the first harmonic of each source and especially its monotonicity.

\subsection{xSFA on manifolds} \label{sec:xSFA_on_manifolds}

%In this section we extend the xSFA theory to cover the case of input signals composed from (approximately) continuous but not necessarily statistically independent sources.
For the blind source separation setting in \cite{SprekelerZitoEtAl-2014} it is assumed that the input is composed from statistically independent sources. This assumption is not necessarily appropriate for the setting studied in this paper. A closer fit can be found in \cite{FranziusSprekelerEtAl-2007e} where the agent's state space, denoted \textsl{configuration space} $\mathcal{V}$, is considered a manifold embedded in the sensor space that yields the data subject of study.
%That idea has so far not been studied in context of xSFA. Data constrained to a manifold cannot be expected to stem from statistically independent sources.
In this section we extend xSFA theory from \cite{SprekelerZitoEtAl-2014} to such a manifold setting and establish a geometrical characterization of the solutions in terms of potential, monotonicity, geodesics and manifold representation. These results motivate the navigation algorithm proposed in section \ref{sec:Global-navigation-algorithm}.

\subsubsection{Slowest features are monotonic flows on the state space}

Lemma \ref{lem1} shows monotonicity with respect to the slowest source, but does not characterize the source itself in context of the state space $\mathcal{V}$. We can show that under certain assumptions, the slowest features actually correspond to monotonic flows across $\mathcal{V}$.
In the notation from \cite{FranziusSprekelerEtAl-2007e} the sources in terms of xSFA are the agent's possible configurations $\mathbf{s} \in \mathcal{V}$. For a fully observable environment and a sufficiently rich sensor, each sensor input value $\mathbf{x}(t)$ can be identified with a value $\mathbf{s} = \psi(\mathbf{x}) \in \mathcal{V}$ such that $\psi$ is an bijective map.
%We investigate the behavior of the sources $\mathbf{s}$ across the agent's state space $\mathbf{S}$. For a fully observable environment, each sensor input value $\mathbf{x}(t)$ can be identified with a value $\mathbf{y} = \psi(\mathbf{x}) \in \mathbf{S}$ such that $\psi$ is an injective map. If $\mathbf{x}$ is composed from sources by $\mathbf{x}(t) = F(\mathbf{s}(t))$ we can map the sources to state space via $\psi \circ F$ %$\psi(F(\mathbf{s}))$.
%and express each state in terms of the sources via $v(\mathbf{s}) \coloneq (\psi \circ F)^{-1}(\mathbf{s})$.
Each output component $\mathbf{g}_i$ of xSFA is a scalar field on $\mathcal{V}$, mapping $\mathbf{s}$ to a real number. We can assume $\mathbf{g}_i$ is bounded so it actually maps to an interval $I_i = [\text{min}_i, \text{max}_i]$, i.e.\
$\mathbf{g}_i \colon \; \mathcal{V} \longrightarrow I_i$. Let $\mathbf{g}^{-1}_i(\theta)$ denote the fiber of $\theta \in I_i$. These are also known as level sets or equipotential sets.
With a plain $I$ we denote the unit interval $I \coloneq [0, 1]$.
Mapping from $\mathcal{V}$ to $\mathbb{R}$, $\mathbf{g}_i$ usually performs a dimensionality reduction, unless $\mathcal{V} \subset \mathbb{R}$. It proves advantageous to study the layout of this reduction, i.e.\ of the fibers of $\mathbf{g}_i$ separately from its value. We achieve this separation by splitting
\begin{equation} \label{tigi}
\mathbf{g}_i(\mathbf{s}) \quad \eqcolon \quad \tilde{\mathbf{g}}_i( t_i(\mathbf{s}) ) \qquad \text{with} \qquad
\tilde{\mathbf{g}}_i \colon \hphantom{\mathcal{V}} I \longrightarrow I_i \; , \qquad
t_i \colon \hphantom{I} \mathcal{V} \longrightarrow I
\end{equation}
%As a by product of this proof we will show that $t_i$ have coordinate character in the sense that 
$t_i$ is some scalar field on $\mathcal{V}$ realizing the level sets of $\mathbf{g}_i$, while $\tilde{\mathbf{g}}_i$ is a real-valued function realizing the value of $\mathbf{g}_i$ on top of $t_i$. We refer to $t_i$ as the coordinate function of $\mathbf{g}_i$, because it defines a one-dimensional coordinate for $\tilde{\mathbf{g}}_i$ on $\mathcal{V}$.
If $\tilde{\mathbf{g}}_i$ is injective, $t_i$ must have the same level sets as $\mathbf{g}$. Otherwise $t_i$ assigns distinct values to separate connectivity components of level sets of $\mathbf{g}$ whenever such components are induced by $\tilde{\mathbf{g}}_i$ being non-injective.
Also note that $t_i$ can differ from $\mathbf{g}$ in velocity and that $t_i$ only underlies Von Neumann boundary conditions where its fibers hit the boundary orthogonally. Note that the choice of $t_i$ is not unique. E.g.\ every composition of $t_i$ with a bijective function yields another valid $t_i$. This gives us the freedom to assume additional properties on $t_i$, most notably uniform velocity of its integral curves.

%For given sets $\mathbf{g}^{-1}_i(\text{min}_i)$, to $\mathbf{g}^{-1}_i(\text{max}_i)$ an example of such a coordinate function $t$ can be constructed from the intrinsic minimal distance function $d_I$.
%Let $d_I$ denote the minimal length of a path connecting two subsets in $\mathcal{V}$. The function
%%Then for given $\mathbf{g}^{-1}_i(\text{min}_i)$, $\mathbf{g}^{-1}_i(\text{max}_i)$ the function
%\begin{equation}
%d_{\text{min}_i, \text{max}_i}(\mathbf{a}) \quad \coloneq \quad \frac{d_I(\{\mathbf{a}\}, \mathbf{g}^{-1}_i(\text{min}_i)) - d_I(\{\mathbf{a}\}, \mathbf{g}^{-1}_i(\text{max}_i))}{d_I(\mathbf{g}^{-1}_i(\text{min}_i), \mathbf{g}^{-1}_i(\text{max}_i))} + \frac{1}{2}
%\end{equation}
%yields a one-dimensional coordinate on $\mathcal{V}$ close to the ideal $t_i$ with $\mathbf{g}_i = \tilde{\mathbf{g}}_i \circ t_i$.

%In this section we use the projector function $\pi_{\alpha}$ to denote e.g.\ $\mathbf{s}_{\alpha} = \pi_{\alpha}(\mathbf{s})$.

\begin{thm} \label{thm1}
Let $\mathbf{g}_i$ be the solution components of xSFA. For every $\tilde{\mathbf{g}}_i$, $t_i$ like defined in \eqref{tigi} with $\mathbf{g}_i = \tilde{\mathbf{g}}_i \circ t_i$ the following holds.
Let $\varphi_{\mathbf{a}, \mathbf{b}} \colon \; I \longrightarrow \mathcal{V}$ be an integral curve of $\nabla t_i$
%connecting $\mathbf{a}, \mathbf{b} \in \mathcal{V}$
from $\mathbf{a} \in \mathbf{g}^{-1}_i(\text{min}_i)$ to $\mathbf{b} \in \mathbf{g}^{-1}_i(\text{max}_i)$ or contrary.
%that is never tangent to a fiber.
%, i.e.\ $(\nabla \mathbf{g}_i)(\varphi_{\mathbf{a}, \mathbf{b}}) \nabla \varphi_{\mathbf{a}, \mathbf{b}} \neq 0 \;\; \forall \; \theta \in I$.
Then it runs through $\mathcal{V}$ strictly monotonically w.r.t.\ $\mathbf{g}_i$, i.e.\ $\mathbf{g}_i \circ \varphi_{\mathbf{a}, \mathbf{b}}$ is a strictly monotonic function.

\end{thm}

%Theorem \ref{thm1}
Intuitively  this means that $\mathbf{g}_i$ has no local extrema or bumps spatially ``between'' its minimal and maximal level sets. It does not rule out local extrema completely but they must be somewhat isolated from the main flow, e.g.\ in another branch of $\mathcal{V}$.
Theorem \ref{thm1} is the first step of characterizing
$\mathbf{g}_i$ to consist of monotonic flows that bridge the potential spanned by $\mathbf{g}^{-1}_i(\text{min}_i)$ and $\mathbf{g}^{-1}_i(\text{max}_i)$ in the slowest possible fashion, or in -- terms of potential theory -- with minimal energy. Note that due to super position principle, $\mathbf{g}_i$ can consist of multiple overlapping flows of this kind. Then it can happen that $\mathbf{g}^{-1}_i(\text{min}_i)$ or $\mathbf{g}^{-1}_i(\text{max}_i)$ is not connected. %We will study such cases later.

\begin{proof}
Let $R_i$ be the set of integral curves $\varphi(\theta)$ of $\nabla t_i$.
%that run from $\mathbf{g}^{-1}_i(\text{min}_i)$ to $\mathbf{g}^{-1}_i(\text{max}_i)$.
We assume that $\mathcal{V}$ provides sufficient structure to define integration over $R_i$, e.g.\ $\mathcal{V}$ could be a Riemannian manifold. With $dV$ we denote integration by volume over $\mathcal{V}$ and with $dS$  we denote integration by volume over $R_i$ in the sense that $R_i$ is a hyper surface in $\mathcal{V}$.
With $D$ denoting the Jacobi matrix, we can write the SFA optimization criterion as follows (c.f.\ optimization problem 2 in \cite{FranziusSprekelerEtAl-2007e}):
\begin{align} \label{SFA_critereon_manifold}
\av{(\nabla \mathbf{g}_i(\mathbf{s}) \; \dot{\mathbf{s}})^2}_{\mathbf{s}, \dot{\mathbf{s}}} \; & = \;
  \int_{\mathcal{V}} p_\mathbf{s}(\mathbf{s}) \; (D \mathbf{g}_i)(\mathbf{s})
\underbrace{
\av{\dot{\mathbf{s}} \dot{\mathbf{s}}^T}_{\dot{\mathbf{s}} | \mathbf{s}} }_{
\qquad\eqcolon \; \mathbf{K}(\mathbf{s})}
  	(D \mathbf{g}_i)^T(\mathbf{s}) \; dV \\
& = \; \int_{\mathcal{V}} p_\mathbf{s}(\mathbf{s}) \; (D \tilde{\mathbf{g}}_i)^2 \; (D t_i)
%\av{\dot{\mathbf{s}} \dot{\mathbf{s}}^T}_{\dot{\mathbf{s}} | \mathbf{s}}
\mathbf{K}(\mathbf{s})
(D t_i)^T \; dV
	\vphantom{\underbrace{
	\av{\dot{\mathbf{s}} \dot{\mathbf{s}}^T}_{\dot{\mathbf{s}} | \mathbf{s}} }_{
	\qquad\eqcolon \; \mathbf{K}(\mathbf{s})}}\\
& = \; \int_{R_i} p_\varphi(\varphi) \int_I \underbrace{p_{\theta | \varphi}(\theta | \varphi) \;
%g (D \varphi)
\rho(\theta)
}_{\qquad \eqcolon \; \tilde{p}(\theta)}\; (D \tilde{\mathbf{g}}_i)^2 \; \underbrace{(D t_i)(\varphi(\theta))
%\av{\dot{\mathbf{s}} \dot{\mathbf{s}}^T}_{\dot{\mathbf{s}} | \varphi(\theta)}
\mathbf{K}(\mathbf{s})
(D t_i)^T(\varphi(\theta))}_{\qquad \eqcolon \; \mathbf{K}_i(\theta)} \; d \theta dS
\end{align}
Here, $\rho$ denotes the volume element regarding $\theta$. Since $\mathbf{K}(\mathbf{s})$ can be interpreted as the empirically measured inverse metric tensor of $\mathcal{V}$, we have $\rho(\theta) = \sqrt{\abs{\det((\nabla \varphi)^T(\theta) \mathbf{G}(\mathbf{s}) (\nabla \varphi)(\theta))}}$ with $\mathbf{G}(\mathbf{s})~\!=~\!\mathbf{K}^{-1}(\mathbf{s})$.
%Here, $g(\mathbf{V}) = \sqrt{\det(\mathbf{V}^T \mathbf{V})}$ denotes the Gramian determinant. If $\mathcal{V}$ is equipped with a non-standard metric tensor $\mathbf{G}$ we have $g(\mathbf{V}) = \sqrt{\abs{\det(\mathbf{V}^T \mathbf{G}(\mathbf{s}) \mathbf{V})}}$.

%Consequently we have \linebreak $\mathbf{K}_i(\theta) = (D t_i)(\varphi(\theta)) \mathbf{G}(\varphi(\theta)) \av{\dot{\mathbf{s}} \dot{\mathbf{s}}^T}_{\dot{\mathbf{s}} | \varphi(\theta)} \mathbf{G}(\varphi(\theta)) (D t_i)^T(\varphi(\theta))$.
%Note that both quantities are strictly positive: $g(D \varphi) > 0$, $\mathbf{K}_i(\theta) > 0$. 

We transform the unit variance constraint in a similar way:

\begin{align}
\av{\mathbf{g}^2_i(\mathbf{s})}_{\mathbf{s}} \; & = \;
\int_{\mathcal{V}} p_\mathbf{s}(\mathbf{s}) \; \mathbf{g}_i^2(\mathbf{s}) \; dV
\quad = \quad \int_{\mathcal{V}} p_\mathbf{s}(\mathbf{s}) \; \tilde{\mathbf{g}}_i^2(t_i(\mathbf{s})) \;  dV \\
& = \; \int_{R_i} p_\varphi(\varphi) \int_I p_{\theta | \varphi}(\theta | \varphi) \;
%g(D \varphi)
\rho(\theta)
\; \tilde{\mathbf{g}}_i^2(t_i(\varphi(\theta))) \; d \theta \; dS \\
& = \; \int_{R_i} p_\varphi(\varphi) \underbrace{\int_I \tilde{p}(\theta) \; \tilde{\mathbf{g}}_i^2(t_i(\varphi(\theta))) \; d \theta}_{\qquad \eqcolon \; v_\varphi \; > \; 0} dS
\end{align}

A valid solution $\mathbf{g}_i$ must yield

\begin{equation}
\int_{R_i} p_\varphi(\varphi) \; v_\varphi \; dS \quad = \quad 1
\end{equation}

Since $p_\varphi(\varphi), v_\varphi > 0$, every $\varphi \in R_i$ contributes a positive quantity to the overall unit variance. Let $v_\varphi^*$ be the family of quantities that yield the slowest signal $\mathbf{g}_1$. The distribution of variance across $R_1$ is the only tread-off between the integral curves forming $t_1$, so we can conclude that for each $\varphi \in R_1$, $\mathbf{g}_1 \circ \varphi$ must be the solution of an optimization problem of the following form:

\begin{align}
\begin{split} \label{SFA_phi}
\optmin{\tilde{\mathbf{g}}_1 \in \mathcal{F}} & \int_I \tilde{p}(\theta) \; (D \tilde{\mathbf{g}}_1)^2 \; \mathbf{K}_1(\theta) \; d \theta \\
\subjectto	& \int_I \tilde{p}(\theta) \; \tilde{\mathbf{g}}_1^2(t_1(\varphi(\theta))) \; d \theta
\quad = \quad v_\varphi^*
\end{split}
\end{align}

The crucial advantage of having $t_i$ split off is that now $\mathbf{K}_i(\theta)$ is scalar-valued.
So \eqref{SFA_phi} is an ordinary SFA optimization problem defined on the interval $I$. It is a bit special, because the variance is not normalized to $1$ but to $v_\varphi^*$ and $\tilde{p}(\theta)$ is not a probability distribution but a general strictly positive weight function. However, these are just scaling issues and the mathematical theory of SFA solutions is still applicable.
The underlying space $I$ is one-dimensional, so it can only involve a single source, which must have coordinate character on $I$, i.e.\ be bijective and continuous, thus monotonic. Since a single source is always statistically independent, we can apply Lemma \ref{lem1} and find that for every $\varphi \in R_1$ the slowest solution $\tilde{\mathbf{g}}_1 \circ t_1 \circ \varphi$ must be a strictly monotonic function on the interior of $I$, denoted $I \setminus \partial I$. Therefore we have for $\theta \in I \setminus \partial I$:

\begin{equation} \label{positive_Dg}
D (\tilde{\mathbf{g}}_1 \circ t_1 \circ \varphi) \quad = \quad D \tilde{\mathbf{g}}_1 \; D t_1 \; D \varphi \quad \neq \quad \mathbf{0}
\end{equation}

$\theta \in  \partial I$ corresponds to $\varphi(\theta) \in \mathbf{g}^{-1}_i(\text{min}_i)$ or  $\varphi(\theta) \in \mathbf{g}^{-1}_i(\text{max}_i)$.
For $\theta \in I \setminus \partial I$ it follows that $D \mathbf{g}_1, D \tilde{\mathbf{g}}_1, D t_1, D \varphi$ are each non-zero and $D t_1 \not\perp D \varphi$.
This readily proves theorem \ref{thm1} for $i=1$.
%We can readily finalize the proof for the case $i=1$:
%A path $\varphi_{\mathbf{a}, \mathbf{b}}$ in $\mathcal{V}$ that is never tangent to a fiber yields
%$D \varphi_{\mathbf{a}, \mathbf{b}} \not\perp D \mathbf{g}_1$ and since $D \mathbf{g}_1 \neq \mathbf{0}$, it follows that
%\begin{equation} \label{positive_Dg_phi}
%D (\mathbf{g}_1 \circ \varphi_{\mathbf{a}, \mathbf{b}}) \quad = \quad D \mathbf{g}_1 D \; \varphi_{\mathbf{a}, \mathbf{b}} \quad \neq \quad 0
%\end{equation}
%implying that $\mathbf{g}_1 \circ \varphi_{\mathbf{a}, \mathbf{b}}$ is strictly monotonic.

To extend the proof to $i > 1$ we need to recall how xSFA operates. After $\mathbf{g}_1$ is extracted, in an idealized xSFA the data is projected onto a space orthogonal to the space of continuous functions of $\mathbf{g}_1$. We can think of an infinite sequence of monomials of $\mathbf{g}_1$ as a basis of this space. Essentially the projection implies
\begin{equation}
\av{\mathbf{g}_1^k \; \mathbf{g}_2}_{\mathbf{s}} \quad = \quad 0 \qquad \forall \; k \in \mathbb{N}
\end{equation}
and consequently that
\begin{equation}
\av{\tilde{\mathbf{g}}^k_1(t_1) \; \tilde{\mathbf{g}}_2(t_2)}_{\mathbf{s}} \quad = \quad 0 \qquad \forall \; k \in \mathbb{N}
\end{equation}
Since $\tilde{\mathbf{g}}_1$ is monotonic we can build a Taylor expansion of the identity function from $\tilde{\mathbf{g}}_1^k(t_1)$, yielding that already the coordinate $t_1$ is orthogonal to $\mathbf{g}_2$:
\begin{equation}
\av{t_1 \; \tilde{\mathbf{g}}_2(t_2)}_{\mathbf{s}} \quad = \quad 0 \qquad \forall \; k \in \mathbb{N}
\end{equation}
Therefore we can express this constraint on coordinate level and restrict $t_2$ to fulfill
\begin{equation}
\av{t_1 \; t_2}_{\mathbf{s}} \quad = \quad 0 \qquad \forall \; k \in \mathbb{N}
\end{equation}
which imposes no further constraint on $\tilde{\mathbf{g}}_2$. That means, the decorrelation constraint for $i > j$ in xSFA sense only affects $t_i$, rather than $\mathbf{g}_i$, and is encoded in $\mathbf{K}_i(\varphi)$.
So we can apply \eqref{SFA_phi} with $i > 1$ and equation \eqref{positive_Dg}, %\eqref{positive_Dg_phi} for $i > 1
follows.
\end{proof}

Our next theorem characterizes the ideal spatial location of $\mathbf{g}^{-1}_i(\text{min}_i)$ and $\mathbf{g}^{-1}_i(\text{max}_i)$ in $\mathcal{V}$. Before we state it, we need to elaborate a bit on notation.

In the above proof of theorem \ref{thm1} it was mentioned that the matrix $\mathbf{K}(\mathbf{s})$ is related to a metric tensor on $\mathcal{V}$. Since it acts on the gradient of $\mathbf{g}_i$ it must correspond to the dual metric tensor on $\mathcal{V}$. That means, when we measure arc length in $\mathcal{V}$ we must do this w.r.t.\ a metric tensor $\mathbf{G}(\mathbf{s}) = \mathbf{K}(\mathbf{s})^{-1}$.
With $\nabla^{\perp}$ we denote the Jacobi matrix projected onto the space orthogonal to the gradient. This is also called the \textsl{skew gradient}. We can write
\begin{equation}
\det (D \varphi) = \sqrt{\det ( (D \varphi)^T (D \varphi))} =
\sqrt{\det ( (\nabla \varphi)^T (\nabla \varphi)) \det ( (\nabla^{\perp} \varphi)^T (\nabla^{\perp} \varphi))} = \abs{\nabla \varphi} g(\nabla^{\perp} \varphi)
\end{equation}
For a custom metric tensor we have the volume element $\rho(\varphi) = \abs{\nabla \varphi}_\mathbf{G} g(\nabla^{\perp} \varphi)$ with $g(\mathbf{V}) = \sqrt{\abs{\det(\mathbf{V}^T \mathbf{G}(\mathbf{s}) \mathbf{V})}}$ respectively. We are now ready to state theorem \ref{thm2}:

\begin{thm} \label{thm2}
	Let $\mathbf{g}_i$ be the solution components of xSFA.
	Let $\tilde{\mathbf{g}}_i$, $t_i$ be like defined in \eqref{tigi}, i.e.\ $\mathbf{g}_i = \tilde{\mathbf{g}}_i\!\!~\circ~\!\!t_i$.
	Let $\varphi \in R_i$ be the integral curves of $\nabla t_i$ running from
	$\mathbf{g}^{-1}_i(\text{min}_i)$ to $\mathbf{g}^{-1}_i(\text{max}_i)$.
	If $R_i$ is globally parameterizable by some appropriate parameter space $\mathbf{I}_{R_i}$, i.e.\ for $\mathbf{r} \in \mathbf{I}_{R_i}$ let $\varphi_\mathbf{r}$ denote a parametrization of $R_i$.
	If $t_i$ can be chosen such that every $\varphi \in R_i$ is parametrized by arc length, i.e.\
	\begin{equation} \label{nabla_ti_param_arglen}
	\forall \quad \mathbf{r} \in \mathbf{I}_{R_i} \colon \qquad \abs{(\nabla t_i)(\varphi_\mathbf{r}(\theta))}_\mathbf{G} \; = \; C_\mathbf{r}
	\end{equation}
	where $C_\mathbf{r}$ denotes a constant for a given $\mathbf{r}$, and such that $\abs{\nabla \varphi_{\mathbf{r}}(\theta)}_\mathbf{G}$ and $g(\nabla^{\perp} \varphi_{\mathbf{r}})$ induce independent coordinates on $R_i$ in the sense that for a fixed $\theta$
	\begin{equation} \label{thm2_independence}
	%\forall \quad \mathbf{r} \in \mathbf{I}_{R_i} \colon \qquad \abs{(\nabla t_i)(\varphi_\mathbf{r}(\theta))}_\mathbf{G} \; = \; C_\mathbf{r}
	\int_{\mathbf{I}_{R_i}} p_{\mathbf{s}}(\varphi_{\mathbf{r}}(\theta)) \abs{\nabla \varphi_{\mathbf{r}}(\theta)}_\mathbf{G} g(\nabla^{\perp} \varphi_{\mathbf{r}}) \; d \mathbf{r} \quad = \quad
	\int_{\mathbf{I}_{R_i}} \abs{\nabla \varphi_{\mathbf{r}}(\theta)}_\mathbf{G} \; d \mathbf{r}
	\int_{\mathbf{I}_{R_i}} p_{\mathbf{s}}(\varphi_{\mathbf{r}}(\theta)) g(\nabla^{\perp} \varphi_{\mathbf{r}}(\theta)) \; d \mathbf{r}
	\end{equation}
	%If $t_i$ can be chosen as a harmonic function, the following holds.
	%If $t_i$ can be chosen such that %each integral curve of $\nabla t_i$ has uniform velocity
	%Let $\varphi_{\mathbf{a}, \mathbf{b}} \colon \; I \longrightarrow \mathcal{V}$ be an integral curve of $\nabla t_i$
	Then, for $i=1$ the sets $\mathbf{g}^{-1}_i(\text{min}_i)$ and $\mathbf{g}^{-1}_i(\text{max}_i)$ are located in $\mathcal{V}$ such that all $\varphi \in R_i$
	\begin{itemize}
		\item have lengths as equal as possible
		\item are in average as long as possible
		\item are as close as possible to geodesics
		\item cover a volume of $\mathcal{V}$ as large as possible
	\end{itemize}
	%If the geometry of $\mathcal{V}$ does not permit to fulfill these goals simultaneously, for an average length $L$ and a covered volume $C$, the term $LC$ is maximized.

	For $i>1$ the named goals apply subject to the coordinate $t_i$ being uncorrelated to coordinates $t_j$ with $j<i$.
\end{thm}

It should be possible to translate the requirement of a global parametrization of $R_i$ to the setting of $R_i$ spanning multiple coordinate charts. This is a primarily technical challenge and is subject of future work.

We can say a bit about the geometrical implications of \eqref{nabla_ti_param_arglen} and \eqref{thm2_independence}. If $\mathbf{g}_i$ has connected level sets, due to continuity it should be possible to scale every streamline of $t_i$ intrinsically such that it fulfills \eqref{nabla_ti_param_arglen}. If the same can be applied to the level sets such that each level set is parameterized independently from $\mathbf{r}$, then $g(\nabla^{\perp} \varphi_{\mathbf{r}}(\theta))$ depends only on $\theta$ and \eqref{thm2_independence} is readily fulfilled. Further note that under some regularity assumptions on $\mathcal{V}$, especially on $\partial \mathcal{V}$, $t_i$ can be chosen as a harmonic function. Here the terminus should not be confused with harmonics in other sections of this work, referring to harmonic oscillations. Instead it refers to harmonic functions from potential theory, i.e.\ with $\Delta t_i = 0$. Then $\varphi$ would be a harmonic mapping, yielding geodesic images, implying constant gradient length. Based on this, the requirements \eqref{nabla_ti_param_arglen} and \eqref{thm2_independence} should be straight forward.

To see why $t_i$ can usually be considered a harmonic function, we think of $\mathbf{g}_i$ as a Dirichlet problem under Von Neumann boundary conditions. If $\mathbf{g}^{-1}_i(\text{min}_i)$, $\mathbf{g}^{-1}_i(\text{max}_i)$ reside on the boundary -- which is mostly the case -- $t_i$ is subject to a Dirichlet boundary condition in these areas. Otherwise, if one or several of these extrema are in the interior of $\mathcal{V}$, remove a surrounding $\epsilon$ ball to form a boundary, yielding the potential as Dirichlet boundary conditions. $t_i$ must still fulfill Von Neumann boundary conditions where level sets hit the boundary orthogonally. Thus, the slowest $t_i$ can be obtained from solving a Dirichlet problem under mixed boundary conditions, also known as \textsl{Zaremba's problem}. In \cite{Brown94themixed} solutions of such problems are studied on Lipschitz domains. They find that if a solution exists, it is a harmonic function. However, it is stated that a solution only exists if the different types of boundary conditions are separated by non-smooth points on the boundary such that they meet at an angel strictly smaller than $\pi$. This matches our experimental observations where $\mathbf{g}^{-1}_i(\text{min}_i)$, $\mathbf{g}^{-1}_i(\text{max}_i)$ always fill entire edges of the boundary in such a way. Alternatively we observe single-point extrema in some corners. Note that this observation may be biased as in an empirical exploration phase, the point in a corner has probability zero. Under analytic view, in such a case the area of Dirichlet boundary condition is a null set and thus not in scope of the setting in \cite{Brown94themixed}.

From theorem \ref{thm2} we conjecture that some points $\mathbf{a} \in \mathbf{g}^{-1}_1(\text{min}_1)$, $\mathbf{b} \in \mathbf{g}^{-1}_1(\text{max}_1)$ realize the intrinsic geodesic diameter of $\mathcal{V}$. One might expect that in domains with boundary, this maximal intrinsic diameter is always realized by points on the boundary. There actually exist counter examples for multiply connected domains. In \cite{Bae2013} such examples are given for polygonal domains in the plane. It is however stated that for simply connected polygons the diameter is always realized by some vertex points. This should be true for general simply connected domains with boundary, but a reference for this is hard to find. At least for convex domains it is somewhat obvious, because the diameter is always a straight line and there exists a cutting plane on which the problem of finding the diameter reduces to the two dimensional case. Since a smooth domain in the plane can be approximated by a polygon, the result for polygons transfers to this case.

Before we can start with the proof of theorem \ref{thm2}, we need the following lemma:

\begin{lem} \label{lem2}
For $k > 0$, on an $n$-dimensional surface $\mathcal{S}$ the variational optimization problem
\begin{align}
\optmin{f \colon \; \mathcal{S} \rightarrow \mathbb{R}}  & \int_{\mathcal{S}} \frac{1}{f^k(\mathbf{x})} \; d V \\ %\mathbf{x} \\
\subjectto	& \int_{\mathcal{S}} f(\mathbf{x}) \; d V  \quad = \quad C \label{lem2_const} \\
& f(\mathbf{x}) \quad > \quad 0 \qquad \forall \; \mathbf{x} \in \mathcal{S} \label{lem2_const2}
\end{align}
is solved by
\begin{equation}
f(x) \quad \equiv \quad \frac{C}{\vol(\mathcal{S})}
\end{equation}
Consequently we have $\int_{\mathcal{S}} f^{-k}(\mathbf{x}) \; d V \; = \; \tfrac{\vol(\mathcal{S})^{(k+1)}}{C^k}$
\end{lem}

\begin{proof}[Proof of Lemma \ref{lem2}]
We apply Euler-Lagrange equations with a Lagrange multiplier for the constraint \eqref{lem2_const}:
\begin{align}
\mathcal{L}(\mathbf{x}, \lambda, f, Df) \quad &= \quad f^{-k}(x) + \lambda f(x) \label{ansatz_lem2}\\
\frac{\partial \mathcal{L}}{\partial f} \quad &= \quad \sum_{j=1}^n \frac{\partial}{\partial \mathbf{x}_j} \frac{\partial \mathcal{L}}{\partial \left(\frac{\partial f}{\partial \mathbf{x}_j}\right)} \quad = \quad 0 \label{eu_la_mult}
\end{align}
Since $\mathcal{L}$ does not depend on derivatives of $f$, \eqref{eu_la_mult} is equal to zero.
Inserting \eqref{ansatz_lem2} into \eqref{eu_la_mult} yields
\begin{align}
- \frac{k}{f^{k+1}} (\mathbf{x}) + \lambda \quad &= \quad 0 \label{lem2_prefinal} \\
f (\mathbf{x}) \quad &= \quad \left(\frac{k}{\lambda}\right)^{\frac{1}{k+1}} \label{lem2_final}
\end{align}
This readily shows that $f$ must be constant. \eqref{lem2_const2} implies $\lambda > 0$ in \eqref{lem2_prefinal} and thus asserts that \eqref{lem2_final} yields a real number.
We do not need to calculate $\lambda$ explicitly, since the constant value of $f$ is directly given by \eqref{lem2_const}. Knowing $f$ is constant, we have
\begin{equation}
\eqref{lem2_const} \quad = \quad \vol(\mathcal{S}) f(\mathbf{x}) \quad = \quad C\\
\end{equation}
\end{proof}

\begin{proof}[Proof of Theorem \ref{thm2}]
We begin with the ansatz from \eqref{SFA_critereon_manifold}:
\begin{equation} \label{ansatz_thm2}
\av{(\nabla \mathbf{g}_i(\mathbf{s}) \; \dot{\mathbf{s}})^2}_{\mathbf{s}, \dot{\mathbf{s}}} \quad = \quad
\int_{\mathcal{V}} p_\mathbf{s}(\mathbf{s}) \; (D \mathbf{g}_i)(\mathbf{s})
\underbrace{\av{\dot{\mathbf{s}} \dot{\mathbf{s}}^T}_{\dot{\mathbf{s}} | \mathbf{s}}}_{\qquad\eqcolon \; \mathbf{K}(\mathbf{s})}
(D \mathbf{g}_i)^T(\mathbf{s}) \; dV
%& = \; \int_{\mathcal{V}} p_\mathbf{s}(\mathbf{s}) \; (D \tilde{\mathbf{g}}_i)^2 \; (D t_i) \av{\dot{\mathbf{s}} \dot{\mathbf{s}}^T}_{\dot{\mathbf{s}} | \mathbf{s}} (D t_i)^T \; dV \\
%& = \; \int_{R_i} p_\varphi(\varphi) \int_I \underbrace{p_{\theta | \varphi}(\theta | \varphi) \; g (D \varphi)}_{\qquad \eqcolon \; \tilde{p}(\theta)}\; (D \tilde{\mathbf{g}}_i)^2 \; \underbrace{(D t_i)(\varphi(\theta)) \av{\dot{\mathbf{s}} \dot{\mathbf{s}}^T}_{\dot{\mathbf{s}} | \varphi(\theta)} (D t_i)^T(\varphi(\theta))}_{\qquad \eqcolon \; \mathbf{K}_i(\theta)} \; d \theta dS
\end{equation}
We consider the integral curves $\varphi \in R_i$ of $t_i$, parameterized by arc length with the unit interval $I$ as parameter space. Thus, the velocity $\abs{\nabla \varphi}_\mathbf{G}$ does not depend on $\theta$. By \eqref{nabla_ti_param_arglen} we can assume that the gradients of $t_i$ have the same property and considering the parameter spaces and images of $\varphi$ and $t_i$ we have
\begin{equation}
t(\mathbf{s}) = \varphi^{-1}(\mathbf{s}) \qquad \forall \; \mathbf{s} \in \varphi(I)
\end{equation}
%This is possible because $t_i$ being a harmonic function implies $\Delta t_i = 0$.
%Choosing a parametrization over the unit interval for every $\varphi \in R_i$ we get \linebreak
This further yields
$\abs{(\nabla t_i)(\varphi(\theta))}_{\mathbf{G}^{-1}} = \frac{1}{\abs{\nabla \varphi(\theta)}_\mathbf{G}}$
%Because $t_i$ is harmonic we can assume that $g(\nabla^{\perp} \varphi)(\theta)$ does not depend on $\varphi$, but only on $\theta$.
%For some appropriate parameter space $\mathbf{I}_{R_i}$ and $\mathbf{r} \in \mathbf{I}_{R_i}$ we let $\varphi_\mathbf{r}$ denote a parametrization of $R_i$.
and the calculus \eqref{ansatz_thm2} continues as follows:
\begin{align}
\eqref{ansatz_thm2} \quad
&\eq_{\hphantom{\text{Hölder}}} \quad
\int_{\mathcal{V}} p_\mathbf{s}(\mathbf{s}) \; (D \mathbf{g}_i)(\mathbf{s})
\mathbf{G}^{-1}(\mathbf{s})
(D \mathbf{g}_i)^T(\mathbf{s}) \; dV \\
&\eq_{\hphantom{\text{Hölder}}} \quad
\int_{\mathbf{I}_{R_i}} \int_I p_{\mathbf{s}}(\varphi_{\mathbf{r}}(\theta)) \frac{(D \tilde{\mathbf{g}}_i)^2}{\abs{\nabla \varphi_{\mathbf{r}}(\theta)}_\mathbf{G}^2} \; \abs{\nabla \varphi_{\mathbf{r}}(\theta)}_\mathbf{G} g(\nabla^{\perp} \varphi_{\mathbf{r}}) \; d \theta d \mathbf{r} \\
&\eq_{\hphantom{\text{Hölder}}} \quad
\int_{\mathbf{I}_{R_i}} \int_I \frac{1}{\abs{\nabla \varphi_{\mathbf{r}}(\theta)}_\mathbf{G}} \; (D \tilde{\mathbf{g}}_i)^2 p_{\mathbf{s}}(\varphi_{\mathbf{r}}(\theta)) g(\nabla^{\perp} \varphi_{\mathbf{r}}) \; d \theta d \mathbf{r} \\
%&\eq_{\hphantom{\text{Hölder}}} \quad
&\eq_{\substack{\text{\eqref{nabla_ti_param_arglen}} \\ \text{\eqref{thm2_independence}} \\ \hphantom{\text{Hölder}}}} \quad
%&\dleq_{\text{Hölder}} \quad
%\left( \max_{\subalign{\theta &\in I \\ \mathbf{r} &\in \mathbf{I}_{R_i}}}
\int_{\mathbf{I}_{R_i}}
\frac{1}{\abs{\nabla \varphi_{\mathbf{r}}(\theta)}_\mathbf{G}} \; d \mathbf{r} %\right)
\quad
%\max_{\mathbf{r} \in \mathbf{I}_{R_i}}
\int_I %\underbrace{
(D \tilde{\mathbf{g}}_i)^2
\int_{\mathbf{I}_{R_i}}
p_{\mathbf{s}}(\varphi_{\mathbf{r}}(\theta)) g(\nabla^{\perp} \varphi_{\mathbf{r}})
%}_{\qquad = p_{\theta}(\theta) \frac{d S}{d \mathbf{r}}}
 \; d \mathbf{r} d \theta
\\
&\eq_{\hphantom{\text{Hölder}}} \quad
%\left( \max_{\subalign{\theta &\in I \\ \mathbf{r} &\in \mathbf{I}_{R_i}}}
\int_{\mathbf{I}_{R_i}}
\frac{1}{
L_\mathbf{G}(\varphi_{\mathbf{r}})
%\abs{\nabla \varphi_{\mathbf{r}}(\theta)}_\mathbf{G}
}  \; d \mathbf{r} %\right)
\quad
%\max_{\mathbf{r} \in \mathbf{I}_{R_i}}
\int_I p_{\theta}(\theta)
(D \tilde{\mathbf{g}}_i)^2 \; d \theta \label{thm2_ansatz_final}
\end{align}
We can transform the constraint in a similar manner, yielding
\begin{align}
\av{\mathbf{g}^2_i(\mathbf{s})}_{\mathbf{s}} \quad & = \quad
\int_{\mathbf{I}_{R_i}}
\abs{\nabla \varphi_{\mathbf{r}}(\theta)}_\mathbf{G} \; d \mathbf{r}
\quad
\int_I
p_{\theta}(\theta)
(\tilde{\mathbf{g}}_i)^2 \; d \theta %d \mathbf{r}
\label{thm2_const_final}
\end{align}
SFA requires the term \eqref{thm2_const_final} to be constantly one. Both integrals yield a certain constant value for the ideal solution $\mathbf{g}_i$. Let $C^*_{ir} = \int_{\mathbf{I}_{R_i}}
\abs{\nabla \varphi_{\mathbf{r}}(\theta)}_\mathbf{G} \; d \mathbf{r}$ and $C^*_{i \theta} = \int_I p_{\theta}(\theta) (\tilde{\mathbf{g}}_i)^2 \; d \theta$ denote these ideal constants. Surely, the ideal solution must satisfy $C^*_{ir} C^*_{i \theta} = 1$. We can now consider \eqref{thm2_ansatz_final}, \eqref{thm2_const_final} to yield two independently solvable optimization problems:

\begin{align}
\begin{split} \label{SFA_C_r}
\optmin{\varphi} & \int_{\mathbf{I}_{R_i}}
\frac{1}{
%\abs{\nabla \varphi_{\mathbf{r}}(\theta)}_\mathbf{G}
L_\mathbf{G}(\varphi_{\mathbf{r}})
}  \; d \mathbf{r} \\
\subjectto	& \int_{\mathbf{I}_{R_i}}
%\abs{\nabla \varphi_{\mathbf{r}}(\theta)}_\mathbf{G}
L_\mathbf{G}(\varphi_{\mathbf{r}})
\; d \mathbf{r}
\quad = \quad C^*_{ir}
\end{split}
\end{align}

\begin{align}
\begin{split} \label{SFA_C_theta}
\optmin{\tilde{\mathbf{g}}_i \in \mathcal{F}} & \int_I p_{\theta}(\theta)
(D \tilde{\mathbf{g}}_i)^2 \; d \theta \\
\subjectto	& \int_I p_{\theta}(\theta) (\tilde{\mathbf{g}}_i)^2 \; d \theta
\quad = \quad C^*_{i \theta}
\end{split}
\end{align}

In \eqref{SFA_C_theta}, $p_{\theta}(\theta)$ is the volume of the level set, weighted by $p_{\mathbf{s}}(\varphi(\theta))$:
\begin{equation}
p_{\theta}(\theta) \quad = \quad p_{\mathbf{s}}(\varphi(\theta)) \vol( \; \{ \; \mathbf{s} \in \mathcal{V} \colon \quad \exists \; \varphi \in R_i \colon \quad \varphi(\theta)\;  = \; \mathbf{s} \;\} \;)
\end{equation}

Problem \eqref{SFA_C_theta} is an ordinary SFA problem for a given $t_i$, while problem \eqref{SFA_C_r} minimizes the dominant cost factor w.r.t.\ the choice of $t_i$.
Since $\tilde{\mathbf{g}}_i$ is an arbitrary differentiable function, we can expect that it can be chosen to compensate the distribution of $p_{\mathbf{s}}(\varphi(\theta))$ in \eqref{SFA_C_theta}.
Equivalently to the technique in the proof of theorem \ref{thm1}, the decorrelation constraint on $\tilde{\mathbf{g}}_i$ is already resolved on coordinate level, i.e.\ is only a constraint on $t_i$. The unit variance constraint is also resolved because we formulated this for the ideal partition of $C^*_{ir} C^*_{i \theta} = 1$ assumed to be known.
Indeed we observe in experiments that larger level sets yield a flatter $\mathbf{g}_i$.
For problem \eqref{SFA_C_r} we apply lemma \ref{lem2} with $k=1$, $f(\mathbf{r}) = \frac{1}{\abs{\nabla \varphi_{\mathbf{r}}}_\mathbf{G}}$, $\mathcal{S} = \mathbf{I}_{R_i}$.
This shows that the ideal curves $\varphi$ are such that $\abs{\nabla \varphi_{\mathbf{r}}}_\mathbf{G}$ is constant and as large as possible. This implies that the ideal $t_i$ is the one that yields the longest integral curves w.r.t.\ $\mathbf{G}$.
A well-known fact from differential geometry is that the integral curves of a gradient field with constant gradient length are geodesics. Consequently, a constant $\abs{\nabla \varphi_{\mathbf{r}}}_\mathbf{G}$ would yield geodesic curves $\varphi$. Note that this case cannot be fulfilled for every $\mathcal{V}$. It only characterizes an attractor for the best solution w.r.t.\ what $\mathcal{V}$ permits.

If we assume a normalization $\vol(\mathbf{I}_{R_i}) = 1$ and
%assess the outcome of lemma \ref{lem2} for
further assume that the integral curves cover a fixed volume $C$, i.e.\ $\int_{\mathbf{I}_{R_i}} L_\mathbf{G}(\varphi_{\mathbf{r}}) \; d \mathbf{r} \; = \; \int_{\mathbf{I}_{R_i}} \abs{\nabla \varphi_{\mathbf{r}}}_\mathbf{G} \; d \mathbf{r} \; = \; C$, we obtain the minimal value of 
\eqref{SFA_C_r} as $\int_{\mathbf{I}_{R_i}} \frac{1}{\abs{\nabla \varphi_{\mathbf{r}}}_\mathbf{G}} \; d \mathbf{r} \; = \; \tfrac{1}{C}$. This shows that the covered volume $C$ should be as large as possible in order to minimize \eqref{SFA_C_r}.
\end{proof}

Intuitively, theorem \ref{thm2} states that $\mathbf{g}^{-1}_i(\text{min}_i)$ and $\mathbf{g}^{-1}_i(\text{max}_i)$ must be as distant as possible within $\mathcal{V}$, w.r.t.\ intrinsic distance. In the fashion of xSFA, this yields a hierarchical covering of $\mathcal{V}$ by the components $\mathbf{g}_i$. The fact that $\mathbf{g}_i$ are orthogonal in the sense of decorrelation suggests a connection to principal curves \cite{doi:10.1080/01621459.1989.10478797} and manifold learning. More specifically the relation of $\mathbf{g}_i$'s streamlines to geodesics suggests a connection to principal geodesic analysis (PGA), \cite{1318725}. The central difference to PGA is however, that PGA finds geodesics emanating from a central mean location in $\mathcal{V}$. Depending on the application, this can be a limitation if $\mathcal{V}$ consists of multiple branches. There are, however, more recent approaches in manifold learning to overcome the limitation of a central mean, e.g.\ \cite{7312494}.
A more systematical comparison to these approaches would certainly be an interesting future study. 

\subsection{Global navigation algorithm} \label{sec:Global-navigation-algorithm}

Our main idea is to decompose the sensor signal into monotonic flows using the first harmonics obtained by xSFA. We can then navigate along each component subsequently into a global optimum. Because of monotonicity this can be achieved by local optimization provided by the PFAx algorithm.

In consistence with notation from xSFA we refer to the components as \textsl{sources} in this section.
Note that our procedure will not necessarily encounter physical sources. With \textsl{sources} we rather refer to whatever xSFA discovers. E.g.\ consider vision input, where the sensor is composed of visual features emitted from opposed walls, yet visible in a single field of view at the same time. Moving closer to one wall will increase vision of that wall's features and decrease vision of the opposing wall's features. In other words, the walls as sources of visual features are not statistically independent. The geometrical analysis in section \ref{sec:xSFA_on_manifolds} suggests that SFA will in such a case identify the agent's position along a coordinate axis between the walls as a virtual source. We focus on such virtual sources, because they are well suited for navigation, even though they might not correspond to actual physical features or entities. The slowest of such sources usually corresponds to the longest geodesic path that can be fitted into the environment, e.g.\ connecting the most distant pair of rooms. This is sometimes called the intrinsic geodesic diameter of the environment.
The underlying heuristic of this principle is that a sensory perception of a consistent environment can always be decomposed into components that behave like monotonic flows. %As long as sufficient monotonicity is obtained, the presented principle can be used to achieve global optimization within the explored state space.
%In section \ref{sec:controlPFA} we accounted for (linear) SFA on top of PFAx. This does not cover navigation with xSFA on top of PFAx as the arising optimization problem would be harder to solve. We postpone this issue for the moment and assume that PFAx-based navigation under xSFA was feasible. Algorithms \ref{alg:pre-final} and \ref{alg:final} are actually tractable with the method from section \ref{sec:controlPFA}.

It turns out that connectedness of the level sets of the extracted xSFA components is a crucial property. In mathematical topology, functions with this property are called \textsl{monotone}. Such topologically monotone functions are particularly valuable for the presented approach as they connect any pair of points in $\mathcal{V}$ by a strictly monotonic path w.r.t. the component's value, as far as the points reside on distinct level sets. Theorem \ref{thm1} asserts a weaker form of this property for xSFA components and indeed we mostly discover components that are topologically monotone. So far, we observed only two causes for eventually disconnected level sets:
\begin{itemize}
	\item approximately equally slow sources are mixed due to superposition principle
	\item a multiply connected domain can yield one connection component of a level set per connectivity path
\end{itemize}
A special subroutine will be required to handle the case of disconnected level sets. Another issue can be caused by very flat regions in some xSFA components as no local gradient significantly points into a direction. While -- in theory -- still monotonic yet rather flat, in practice we can encounter regions with a representation close to constant or even with a moderate noise. This is caused by numerical approximation of the analytic solution. We will also present a subroutine to deal with this effect. The following algorithm is the basic approach that works well for topologically monotone xSFA components. Consider step $6$ as a slot where we can plug in the mentioned subroutines.

\begin{alg} \label{alg:navxSFA}
Task: Navigate the agent into a goal state $\mathbf{x}^*$.
\begin{enumerate} %[wide=0.5em, leftmargin =*, nosep, before = \leavevmode\vspace{-\baselineskip}]
	\item Apply PFAx to extract $r$ manipulatable features (pre feature space).
	\item Apply xSFA to decompose these features into sources $\mathbf{s}_{\alpha i}$ (feature space).
	\item Use the obtained extraction rules on $\mathbf{x}^*$ to compute the equivalent goal $\mathbf{s}^*$ in feature space.
	\item For $\alpha = 1, \ldots, S$: \\
	While $\sum_{\beta = 1}^{\alpha} (\mathbf{s}_{\beta 1} - \mathbf{s}_{\beta 1}^{*})^2 > \theta$: Choose $\mathbf{u}(t)$ to minimize $\sum_{\beta = 1}^{\alpha} (\mathbf{s}_{\beta 1} - \mathbf{s}_{\beta 1}^{*})^2$. \\
	If we cannot reduce $\sum_{\beta = 1}^{\alpha} (\mathbf{s}_{\beta 1} - \mathbf{s}_{\beta 1}^{*})^2$ by more than ${\tilde{\theta}}$: Goto $6$. % with $j = 2$.
\item End.
\item Report failure.
\end{enumerate}
\end{alg}
\absatz

Note that for statistically independent sources it would be sufficient to directly minimize the cost function $\sum_{\beta = 1}^S (\mathbf{s}_{\beta 1} - \mathbf{s}_{\beta 1}^{*})^2$ of the final iteration. Extending the sum gradually during navigation is -- however -- more stable if the sources were not accurately separated by xSFA.

It can happen that PFAx has to minimize a source for which the current state resides in an almost flat area, c.f.\ section \ref{sec:bottleneck}. Then it might not be possible to find a proper direction for local optimization. We add a routine to deal with flat areas by using higher harmonics:

\begin{subr} \label{subr:flat_xSFA}
	Task: Deal with flat areas.
	%Inherit steps $1$-$5$ from algorithm \ref{alg:navxSFA}.
	\begin{enumerate} %[wide=0.5em, leftmargin =*, nosep, before = \leavevmode\vspace{-\baselineskip}]
		\setcounter{enumi}{5}
		%\item Apply PFAx to extract $r$ manipulatable features (pre feature space).
		%\item Apply xSFA to decompose these features into sources $\mathbf{s}_{\alpha i}$ (feature space).
		%\item Use the obtained extraction rules on $\mathbf{x}^*$ to compute the equivalent goal $\mathbf{s}^*$ in feature space.
		%\item For $\alpha = 1, \ldots, S$: \\
		%While $\sum_{\beta = 1}^{\alpha} (\mathbf{s}_{\beta 1} - \mathbf{s}_{\beta 1}^{*})^2 > \theta$: Choose $\mathbf{u}(t)$ to minimize $\sum_{\beta = 1}^{\alpha} (\mathbf{s}_{\beta 1} - \mathbf{s}_{\beta 1}^{*})^2$. \\
		%If we cannot reduce $\sum_{\beta = 1}^{\alpha} (\mathbf{s}_{\beta 1} - \mathbf{s}_{\beta 1}^{*})^2$ by more than ${\tilde{\theta}}$: Goto $6$ with $j = 2$.
		%\item End.
		\item If coming from $4$: Set $j = 2$. \\
		While $\sum_{\beta = 1}^{\alpha} (\mathbf{s}_{\beta 1} - \mathbf{s}_{\beta 1}^{*})^2 + \sum_{i = 1}^{j} (\mathbf{s}_{\alpha i} - \mathbf{s}_{\alpha i}^{*})^2 > \theta$: \\
		Choose $\mathbf{u}(t)$ to minimize $\sum_{\beta = 1}^{\alpha} (\mathbf{s}_{\beta 1} - \mathbf{s}_{\beta 1}^{*})^2 + \sum_{i = 1}^{j} (\mathbf{s}_{\alpha i} - \mathbf{s}_{\alpha i}^{*})^2$ \\
		If we cannot reduce $\sum_{\beta = 1}^{\alpha} (\mathbf{s}_{\beta 1} - \mathbf{s}_{\beta 1}^{*})^2 + \sum_{i = 1}^{j} (\mathbf{s}_{\alpha i} - \mathbf{s}_{\alpha i}^{*})^2$ by more than ${\tilde{\theta}}$: \\
		Repeat $6$ with $j$ increased by $1$.
		\item Goto $4$, i.e.\ perform another sweep.
	\end{enumerate}
\end{subr}
\absatz

We outlined the problem with disconnected level sets above. It can happen that the navigation reaches the right level set at the wrong connectivity component. This case can be detected by looking at other components.
%There will be components where the distinct connectivity components of level set are spanned by distinct level sets, running more or less orthogonally to the levels of the former component.
Then the heuristics is that, as xSFA fits as many orthogonal components into an environment as possible, every possible pair of points will be covered by some component such that the level sets run somewhat orthogonally to a path connecting the points.

Consider a multiply connected domain. We can split it into subdomains such that each subdomain is simply connected. Within a single subdomain, the level sets of an xSFA component are connected. If the issue was caused by superposition principle mixing two components, assume the domain was split such that each section only contains connected level sets. We have no problem if the navigation task only concerns locations within one of these sections. However, note that each component would yield another kind of split. Thus, for tasks involving more than one section, we just have to find the right component, i.e.\ a component where our navigation task resides within the same section.

We formulate a subroutine based on this idea. First we detect that we are stuck by finding that the navigation is locally optimal, while the distance to the goal measured in feature space is still significantly high. We conclude that all components considered so far are not well suited for the current task in terms of level set connectivity. Starting at the current component, we search for a single component that yields significant improvement. After fully exploiting that component we must restart the algorithm at the first component, because this procedure has likely brought us to another connectivity component of the level sets considered so far. That means, all earlier components are now relevant again. We suppose that this procedure automatically avoids navigation cycles, because components are ordered by slowness.
\begin{subr} \label{subr:disconnected_xSFA}
	Task: Deal with disconnected level sets.
	\begin{enumerate}
		\setcounter{enumi}{5}
		\item
		For $\beta = \alpha, \ldots, S$: \\
		%If coming from $4$: Set $j = 2$. \\
		While $(\mathbf{s}_{\beta 1} - \mathbf{s}_{\beta 1}^{*})^2 > \theta$: Choose $\mathbf{u}(t)$ to minimize $(\mathbf{s}_{\beta 1} - \mathbf{s}_{\beta 1}^{*})^2$. \\
		If this reduced $(\mathbf{s}_{\beta 1} - \mathbf{s}_{\beta 1}^{*})^2$ by more than ${\tilde{\theta}}$: Goto $4$, i.e.\ perform another sweep.
		\item Report failure.
	\end{enumerate}
\end{subr}
\absatz

To combine routines \ref{subr:flat_xSFA} and \ref{subr:disconnected_xSFA} we suggest to handle flat areas first. A new parameter $j_{\text{max}}$ is required as a termination condition of the first routine. This results in the combined routine
\begin{subr} \label{subr:flat_disconnected_xSFA}
	Task: Combined routine to deal with flat areas and disconnected level sets.
	\begin{enumerate}
		\setcounter{enumi}{5}
		\item If coming from $4$: Set $j = 2$. \\
		While $\sum_{\beta = 1}^{\alpha} (\mathbf{s}_{\beta 1} - \mathbf{s}_{\beta 1}^{*})^2 + \sum_{i = 1}^{j} (\mathbf{s}_{\alpha i} - \mathbf{s}_{\alpha i}^{*})^2 > \theta$: \\
		Choose $\mathbf{u}(t)$ to minimize $\sum_{\beta = 1}^{\alpha} (\mathbf{s}_{\beta 1} - \mathbf{s}_{\beta 1}^{*})^2 + \sum_{i = 1}^{j} (\mathbf{s}_{\alpha i} - \mathbf{s}_{\alpha i}^{*})^2$ \\
		If we cannot reduce $\sum_{\beta = 1}^{\alpha} (\mathbf{s}_{\beta 1} - \mathbf{s}_{\beta 1}^{*})^2 + \sum_{i = 1}^{j} (\mathbf{s}_{\alpha i} - \mathbf{s}_{\alpha i}^{*})^2$ by more than ${\tilde{\theta}}$: \\
		If $j < j_{\text{max}}$: Repeat $6$ with $j$ increased by $1$, else goto $8$.
		\item Goto $4$, i.e.\ perform another sweep.
		\item
		For $\beta = \alpha, \ldots, S$: \\
		%If coming from $4$: Set $j = 2$. \\
		While $(\mathbf{s}_{\beta 1} - \mathbf{s}_{\beta 1}^{*})^2 > \theta$: Choose $\mathbf{u}(t)$ to minimize $(\mathbf{s}_{\beta 1} - \mathbf{s}_{\beta 1}^{*})^2$. \\
		If this reduced $(\mathbf{s}_{\beta 1} - \mathbf{s}_{\beta 1}^{*})^2$ by more than ${\tilde{\theta}}$: Goto $4$, i.e.\ perform another sweep.
		\item Report failure.
	\end{enumerate}
\end{subr}
\absatz

%To achieve tractable navigation with the optimization described in section \ref{sec:controlPFA} we can modify the algorithms above to operate with ordinary SFA instead of xSFA.
%This has further efficiency implications:
Finding $S$ sources using xSFA involves $S$ runs of ordinary SFA plus a significant amount of computation to generate and filter nonlinearities of already obtained sources. %The modified approach will require only a single ordinary SFA run.
We propose a modified approach that requires only one single ordinary SFA run.

The central observation to achieve this is that once $(\mathbf{s}_{\alpha 1} - \mathbf{s}_{\alpha 1}^{*})^2$ is globally minimal, also $(\mathbf{s}_{\alpha i} - \mathbf{s}_{\alpha i}^{*})^2 \; \forall \; i > 1$ are globally minimal, at least if $\mathbf{x}^*$ is a position that actually exists in the environment and not an artificial goal. Remember that each output component $\mathbf{g}_{\mathbf{i}}$ of plain SFA is composed of harmonics $\mathbf{s}_{\alpha i}$ of earlier obtained sources. The first component found by SFA is however not a mixture, but the first harmonic of the slowest source, i.e.\ $\mathbf{g}_{1} = \mathbf{s}_{1 1}$, assuming that $\alpha = 1$ indicates the slowest source. Once we have globally minimized $(\mathbf{g}_{1} - \mathbf{s}_{1 1}^{*})^2 = (\mathbf{s}_{1 1} - \mathbf{s}_{1 1}^{*})^2$, any further improvement potential in $\mathbf{g}_{j}$ for $j > 1$ must stem from a source with $\alpha > 1$. That means, we can alternatively optimize along $\mathbf{g}_{i}$ subsequently instead of $\mathbf{s}_{\alpha 1}$:

\begin{alg} \label{alg:navSFA}
	Task: Navigate the agent into a goal state $\mathbf{x}^*$.
	\begin{enumerate} %[wide=0.5em, leftmargin =*, nosep, before = \leavevmode\vspace{-\baselineskip}]
		\item Apply PFAx to extract $r$ manipulatable features (pre feature space).
		\item Apply SFA to decompose these features into $R$ mixtures $\mathbf{g}_{i}$ of sources (feature space).
		\item Use the obtained extraction rules on $\mathbf{x}^*$ to compute the equivalent goal $\mathbf{g}^*$ in feature space.
		\item For $j = 1, \ldots, R$: \\
		While $\sum_{i = 1}^{j} (\mathbf{g}_{i} - \mathbf{g}_{i}^{*})^2 > \theta$: Choose $\mathbf{u}(t)$ to minimize $\sum_{i = 1}^{j} (\mathbf{g}_{i} - \mathbf{g}_{i}^{*})^2$. \\
		If PFAx cannot reduce $\sum_{i = 1}^{j} (\mathbf{g}_{i} - \mathbf{g}_{i}^{*})^2$ by more than ${\tilde{\theta}}$: Goto $6$. % with $j$ increased by $1$.
		\item End.
		\item Report failure.
	\end{enumerate}
\end{alg}
\absatz

We translate the routine for flat areas to this notion:

\begin{subr} \label{subr:flat_SFA} %\label{alg:final}
	Task: Deal with flat areas.
	\begin{enumerate} %[wide=0.5em, leftmargin =*, nosep, before = \leavevmode\vspace{-\baselineskip}]
		\setcounter{enumi}{5}
		\item  If coming from $4$: Increase $j$ by $1$. \\
		While $\sum_{i = 1}^{j} (\mathbf{g}_{i} - \mathbf{g}_{i}^{*})^2 > \theta$: Choose $\mathbf{u}(t)$ to minimize $\sum_{i = 1}^{j} (\mathbf{g}_{i} - \mathbf{g}_{i}^{*})^2$. \\
		If PFAx cannot reduce $\sum_{i = 1}^{j} (\mathbf{g}_{i} - \mathbf{g}_{i}^{*})^2$ by more than ${\tilde{\theta}}$: Repeat $6$ with $j$ increased by $1$.
		\item Goto $4$, i.e.\ perform another sweep.
	\end{enumerate}
\end{subr}
\absatz

The routine to handle disconnected level sets translates as follows:

\begin{subr} \label{subr:disconnected_SFA}
	Task: Deal with disconnected level sets.
	\begin{enumerate}
		\setcounter{enumi}{5}
		\item For $i = j, \ldots, R$: \\
		While $(\mathbf{g}_{i} - \mathbf{g}_{i}^{*})^2 > \theta$: Choose $\mathbf{u}(t)$ to minimize $(\mathbf{g}_{i} - \mathbf{g}_{i}^{*})^2$. \\
		If this reduced $(\mathbf{g}_{i} - \mathbf{g}_{i}^{*})^2$ by more than ${\tilde{\theta}}$:  Goto $4$, i.e.\ perform another sweep.
		\item Report failure.
	\end{enumerate}
\end{subr}
\absatz

We conclude this section by providing the combined routine for algorithm \ref{alg:navSFA}. Here we do not need the parameter $j_{\text{max}}$, because we can use $R$ instead.

\begin{subr} \label{subr:flat_disconnected_SFA}
	Task: Combined routine to deal with flat areas and disconnected level sets.
	\begin{enumerate}
		\setcounter{enumi}{5}
		\item  If coming from $4$: Set $k = j + 1$. \\
		While $\sum_{i = 1}^{k} (\mathbf{g}_{i} - \mathbf{g}_{i}^{*})^2 > \theta$: Choose $\mathbf{u}(t)$ to minimize $\sum_{i = 1}^{k} (\mathbf{g}_{i} - \mathbf{g}_{i}^{*})^2$. \\
		If PFAx cannot reduce $\sum_{i = 1}^{k} (\mathbf{g}_{i} - \mathbf{g}_{i}^{*})^2$ by more than ${\tilde{\theta}}$: \\
		If $k < R$: Repeat $6$ with $k$ increased by $1$, else goto $8$.
		\item Goto $4$, i.e.\ perform another sweep.
		\item For $i = j, \ldots, R$: \\
		While $(\mathbf{g}_{i} - \mathbf{g}_{i}^{*})^2 > \theta$: Choose $\mathbf{u}(t)$ to minimize $(\mathbf{g}_{i} - \mathbf{g}_{i}^{*})^2$. \\
		If this reduced $(\mathbf{g}_{i} - \mathbf{g}_{i}^{*})^2$ by more than ${\tilde{\theta}}$:  Goto $4$, i.e.\ perform another sweep.
		\item Report failure.
	\end{enumerate}
\end{subr}
\absatz

\section{Experiments and Applications} \label{sec:experiments}

This section continues in a sense the experiments from \cite{2017arXiv171200634R}, but using the global navigation technique developed in this paper. We solve the problematic obstacle scenario from there and also tackle even more complex multiroom scenarios. We start with a comprehension of the general navigation setting.

Inspired by reinforcement learning (RL) we have an agent in an environment -- e.g.\ think of a virtual rat on a table. During a training phase it can explore the environment in order to solve navigation tasks.
In contrast to RL we do not consider an arbitrary reward signal for now, but focus on navigating the agent into a desired goal state. In terms of RL this can be seen as using a distant function as reward signal, measuring the distance between the agent's current state and the goal state. As distance measure we use least squares distance in feature space like it is denoted in the algorithms throughout section \ref{sec:Global-navigation-algorithm}.

For exploration we assume a random walk with a fixed step size, choosing a new direction by a uniform random distribution after each step. A more sophisticated exploration routine could be applied in future work, e.g.\ curiosity-driven exploration by aiming for the largest change in the so far discovered slow feature space.
%During training phase, PFAx can observe a sensor signal and the random control commands driving the agent. The sensor is used as main input, i.e.\ to compose features from, while the control signal is provided as supplementary information. To constrain the experiment's complexity we let the agent walk with constant speed during the training phase and when solving the navigation task, i.e.\ the only control information is direction.

\begin{figure}[ht]
	\centering
	\captionsetup{width=.60\linewidth}
	\includegraphics[width=0.6\hsize]{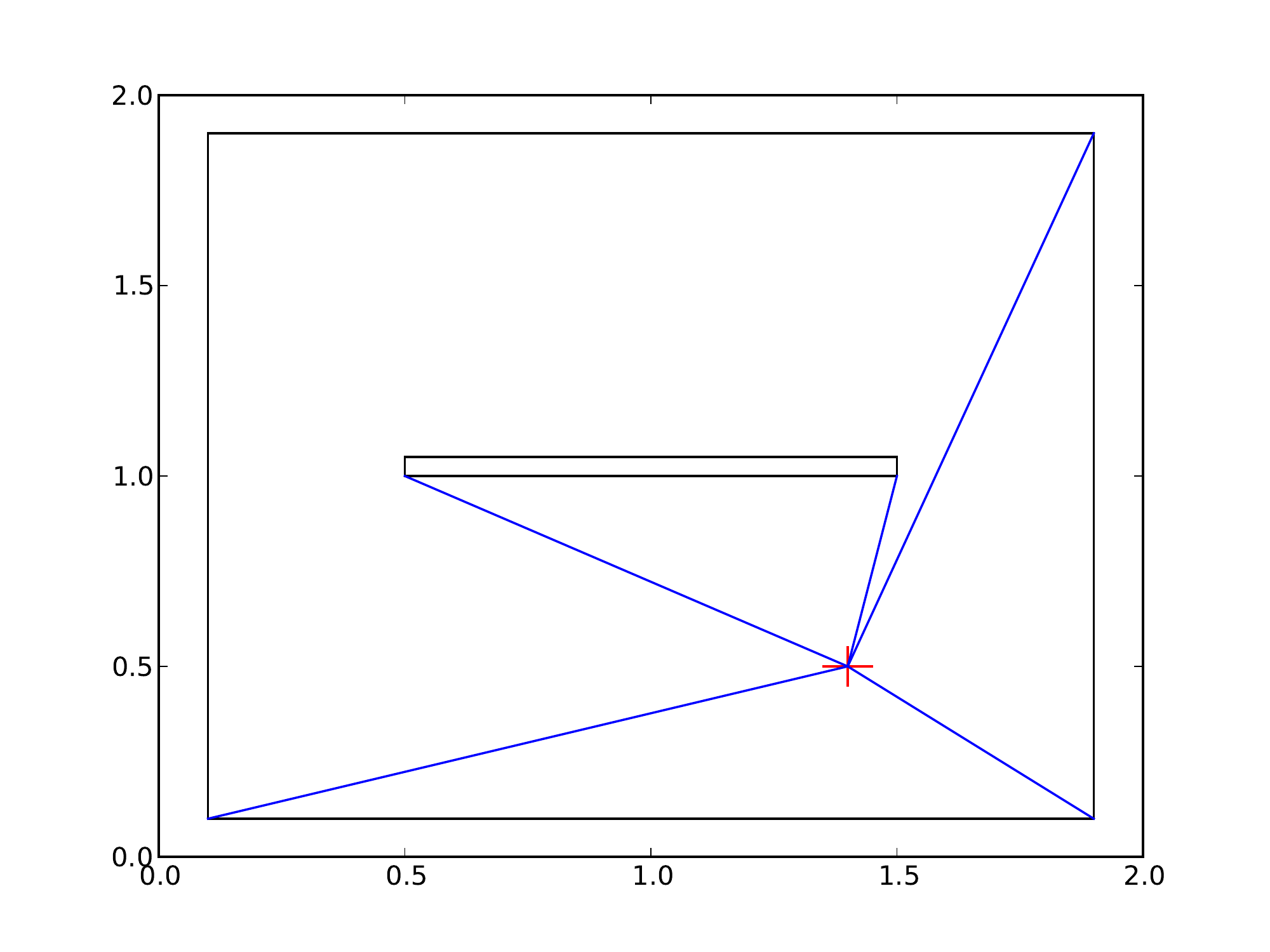}
	\caption{The full field of view is split up into sections occupied by each wall segment in a $360$° field of view. The proposed wall sensor measures the fraction of each wall segment visible from the agent's current location.}
	\label{fig:wall-sensor}
\end{figure}

To model the sensory input signal we mostly focus on the wall sensor introduced in \cite{2017arXiv171200634R} -- a virtual sensor that measures for a current location the visible fraction of each wall segment in an overall $360$° field of view, see figure \ref{fig:wall-sensor}.
E.g.\ for a plain square single room environment the sensor would emit four components, one for each wall.
Note that because of the $360$° field of view, the sensor is by construction invariant regarding head direction. This is a simplification, allowing us to focus on the navigation task itself. In \cite{FranziusSprekelerEtAl-2007e} it was shown that SFA is capable to find head direction invariant features, so this simplification is not a general restriction.

%After the training phase we apply algorithm \ref{alg:final}. Depending on the environment we use an appropriate expansion degree and apply -- unless noted otherwise -- PFAx with the configuration described in \cite{2017arXiv171200634R}. Most notably we use $p = 1$, because it is exhaustively sufficient for making predictions based on movement deltas, as far as no higher moments like acceleration are incorporated into the simulation.

We apply a simplified version of PFAx in these experiments. If the relation between control signal and sensor signal is sufficiently simple, SFA can be used as a proxy for the extraction of predictable features. We applied \eqref{pfa-fittingBUExplicit} and \eqref{pfa-fittingU} to obtain prediction matrices from an SFA extraction matrix. This is mainly done for technical simplicity as this simplification turned out to be sufficient for the experiments presented here. It was studied in \cite{8353107} that SFA often is a good proxy for extracting predictable features. Experiments concerning a complexer control relationship that requires an actual dimensionality reduction in terms of PFAx may be part of future work.

For each environment studied in this section, we illustrate the features found by SFA, provide interpretations and discuss their suitability for navigation tasks.
%Then we illustrate exemplary navigation runs.
Then we illustrate how algorithm \ref{alg:navSFA} would solve one or two exemplary navigation tasks by showing the paths that would arise during its first iterations.
In these plots, the yellow crosshair always indicates the goal and the colormap displays Euclidean/least squares distance of each point to the goal in feature space. The navigation path is rendered in white on top of the colormap. Beyond that, we illustrate a vector field-like navigation flow for the whole environment in cyan. This is computed by performing a few navigation steps for each starting point on a uniform lattice across the whole environment. The environment's bounding box is normalized to $[0, 1]^2 \subset \mathbb{R}^2$ with preserved aspect ratio. Throughout this section we use an overly exhaustive training phase of $200000$ steps. This shows that the algorithm cannot be over-trained and it yields very clean and interpretable results. We can sometimes even hypothesize that we gained results visually close to the unknown ideal SFA-solution for the respective environment.
We start by examining the trivial case of a plain square single room.

\subsection{SFA in two dimensions %(single room)
	} \label{sec:SFA-2D}

For a better understanding of multi room experiments we first illustrate the features found by SFA in case of a single room environment. These are closely related to the harmonics shown in section \ref{sec:harmonics-1D}. Especially the first two components in figure \ref{fig:SFA-2D-d4-8} are easily recognizable as the first harmonic along two distinct axes.

\begin{figure}[ht]
	\centering
	\captionsetup{width=.90\linewidth}
	\includegraphics[width=1.0\hsize]{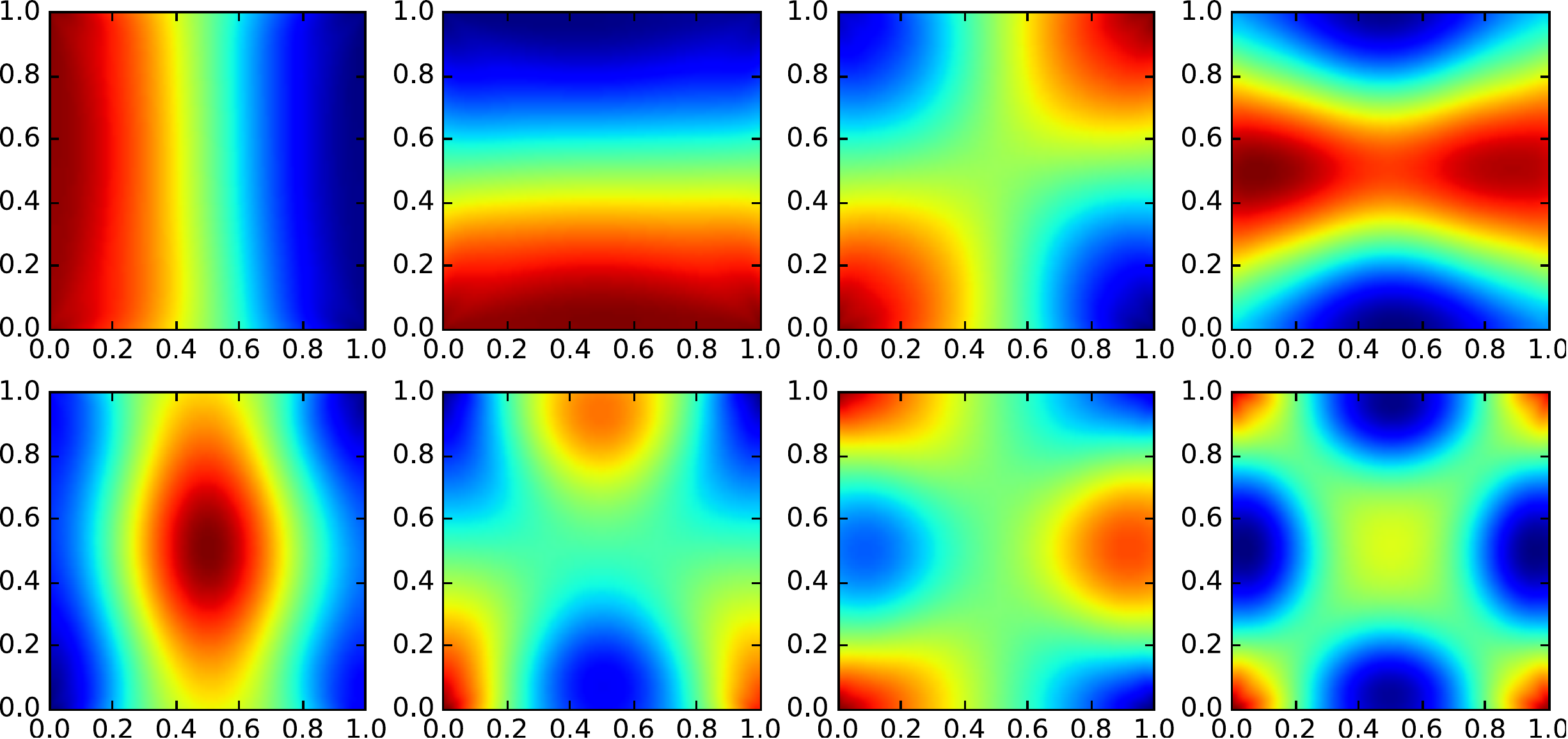}
	\caption{Illustration of SFA in two dimensions. The walk consists of $200$k steps with random direction and fixed step size of $0.02$. Nonlinear expansion was performed using monomials up to the fourth degree. The eight slowest components are shown.}
	\label{fig:SFA-2D-d4-8}
\end{figure}

We conclude that the algorithm recognizes the agent's $x$-coordinate as the primary source because the slowest feature displays the first 1D harmonic laid out along the $x$-axis of the environment. 
The second feature displays the same harmonic along the $y$-axis, which is therefore the second source. Note that these features are almost equally slow and their order of appearance is arbitrary for a square environment. In a rectangular but non-square environment, the longest edge would yield the slowest source. Also their sign, i.e\ their direction of descent may be flipped.
Subsequent components are mixtures of higher harmonics of these two sources and no third source can be discovered in this setting.
We do not provide a navigation run for this setting, because it was already solved in \cite{2017arXiv171200634R}.

In the following experiments this pattern will show up frequently in subregions, especially for each of the several rooms the complexer environments are composed of. The pattern will however show some perturbation close to doors and other bottlenecks and will be augmented by more global features spanning several rooms.

\subsection{Two rooms} \label{sec:twoRooms}

We investigate the simplest case with multiple rooms. Figure \ref{fig:SFA-two-rooms} displays the eight slowest components for a symmetrical environment that is split into two rooms connected by a central pathway. 

\begin{figure}[ht]
	\centering
	\captionsetup{width=.90\linewidth}
	\includegraphics[width=1.0\hsize]{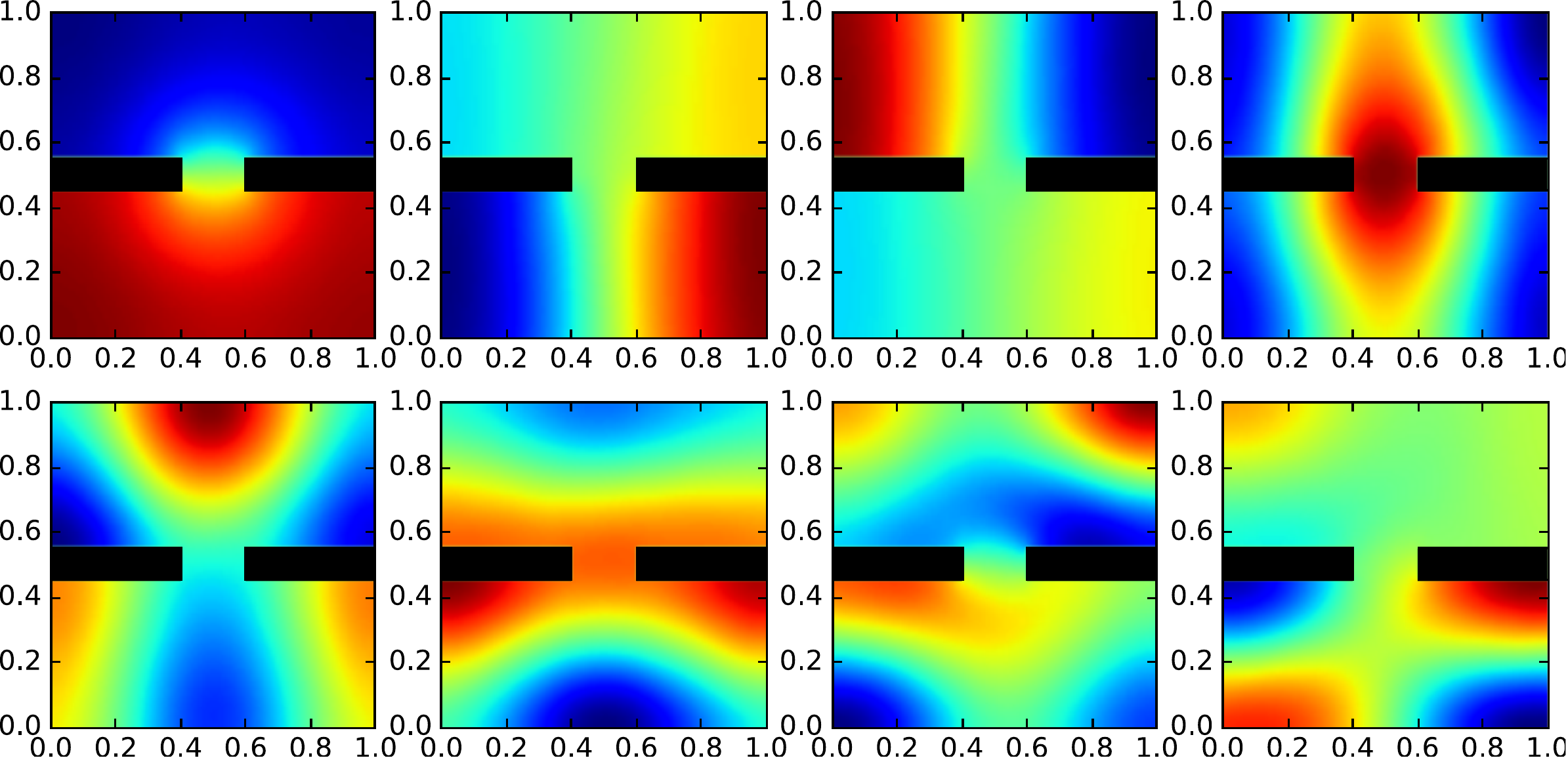}
	\caption{Illustration of SFA in two rooms. The walk consists of $200$k steps with random direction and fixed step size of $0.02$. Nonlinear expansion was performed using monomials up to the fifth degree. The eight slowest components are shown.}
	\label{fig:SFA-two-rooms}
\end{figure}

The first component is the slowest source and is the only source that spans both rooms. It is the crucial feature to navigate into the correct room and can even be seen as an indicator for room identity, linking this scenario to SFA based classification \cite{Escalante-B.Wiskott-2013b}. For our purpose, this component serves to guide the agent into the right room. Beneath its indicator characteristic the signal is still continuous and monotonically increasing/decreasing towards the pathway. Note the circular equipotential levels which serve to guide the agent to the door in case a room change is necessary.

The second and third components are equivalents of the first source we found in the previous section \ref{sec:SFA-2D}, but scoped on one room each. Indeed we would (more or less) find the whole decomposition from section \ref{sec:SFA-2D} for each room over time. E.g.\ the eighth component corresponds to the third component from figure \ref{fig:SFA-2D-d4-8}, scoped on the bottom room. The second component is an overlap of second harmonics of the sources.
No further sources are discovered. All we find in subsequent components are mixtures of higher harmonics of the first three sources. The vertical component for each room is already provided by the first component as a side effect of its room-crossing nature. Compared to the pure vertical component in figure \ref{fig:SFA-2D-d4-8} it shows expectable perturbation at the pathway.
Interestingly all three sources that exist in this environment are discovered almost cleanly unmixed, even though plain SFA was used and not xSFA. While we frequently observe rather unmixed initial occurrences of the first harmonics of the sources, this is not guaranteed. However, it is quite helpful for interpretation of the results.

The pathway itself is an attractor for steepness and we observe a similar perturbation of the harmonics as studied in section \ref{sec:bottleneck}. As concluded in that section this is not a big issue, but might require some special care if the signal should become too flat for proper navigation outside of the bottleneck. On the other hand there are use-cases for a detector of bottleneck states (\cite{DBLP:conf/icml/McGovernB01, 10.1007/3-540-45622-8_16}). %Automatic Discovery of Subgoals in Reinforcement Learning, Learning Options in Reinforcement Learning
A sudden increase of steepness in the first harmonic of each source can serve to detect such a bottleneck state, e.g.\ by applying a threshold on it's squared derivative. This principle can further be seen as a model for surprise, which is e.g.\ a central notion in \cite{Schapiro2013}.
%Neural representations of events arise from temporal community structure

The repeated occurrence of source characteristics and higher harmonics for each room is expectable as the rooms divide the environment into regions of low sensor correspondence. This observation supports the notion of a hierarchical decomposition of the environment into easier subtasks corresponding to independent sources.

\begin{figure}[ht]
	\centering
	\captionsetup{width=.90\linewidth}
	\includegraphics[width=1.0\hsize]{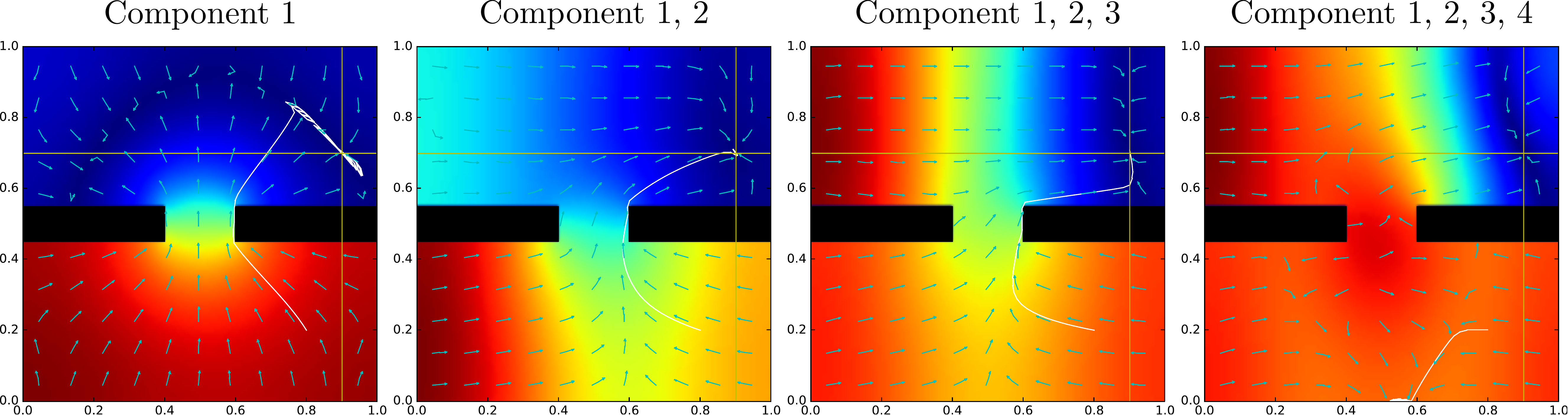}
	\caption{Illustration of SFA-based navigation using approximately independent sources. Combining the first three, i.e.\ the actual sources, yields successful navigation results. Adding a higher component right from the start would break navigation (right).}
	\label{fig:SFA-two-rooms_nav}
\end{figure}

Figure \ref{fig:SFA-two-rooms_nav} illustrates the combinations of sources the algorithm would use for navigation. Here we decompose the algorithm's iterations and navigate using only the sources from one iteration alone, thus illustrating its workability independently from $\theta$. Depending on $\theta$ the navigation pathways would be merged. Note that when the algorithm would start to consider the fourth component, it would be already fairly close to the goal. The navigation failure of the fourth component in figure \ref{fig:SFA-two-rooms_nav} only demonstrates that the component cannot be used right from the beginning.

%\begin{figure}[ht]
%	\centering
%	\captionsetup{width=.90\linewidth}
%	\includegraphics[width=1.0\hsize]{two_room_navigation2.pdf}
%	\caption{Illustration of SFA-based navigation using other combinations of the sources than occurring in our algorithm. None of these combinations yields successful navigation. Combining source 1 and 3 (right) works for our example, but would fail e.g.\ for starting points in the bottom left corner.}
%	\label{fig:SFA-two-rooms_nav2}
%\end{figure}
%
%Figure \ref{fig:SFA-two-rooms_nav2} shows that out of order combinations of the sources or navigating along single sources not ordered by slowness does not yield reliable navigation.

\subsubsection{Using karthesian coordinates}

As an alternative to the wall sensor we take a short look at learning the two-room-scenario based on plain karthesian $(x, y)$-coordinates as a sensor, e.g.\ like a GPS signal would provide. One might intuitively expect that navigation based on such coordinates is trivial. While this is the case for a single room, plain coordinates are particularly ill-suited for multiple rooms. Imagine two nearby spots, separated by a wall. Karthesian coordinates would poorly represent the fact that in terms of navigation these spots might be actually fairly distant.

An even more significant issue with this setting becomes clear once we remember that SFA obtained three sources describing this scenario. That means, with karthesian coordinates, our sensor would be lower dimensional than the number of sources forming our model, i.e.\ $\dim (\mathbf{x}) < S$. This is a rather strange relation: Instead of a manifold being embedded into a higher dimensional sensor space, it is now encoded into a lower dimensional space.

It turns out that SFA can still retrieve exactly the same sources that we previously found based on the wall sensor, but this requires an extremely high degree of nonlinear expansion. We start seeing the correct structures at expansion degrees between $30$ and $40$. Using monomials, only expansion degrees $<20$ are numerically feasible with 64 bit floating point arithmetics. That means, this sensor representation requires a numerically more stable way of nonlinear expansion. Like in section \ref{sec:bottleneck}, we can leverage Legendre polynomials for this purpose.

\begin{figure}[ht]
	\centering
	\captionsetup{width=.90\linewidth}
	\includegraphics[width=1.0\hsize]{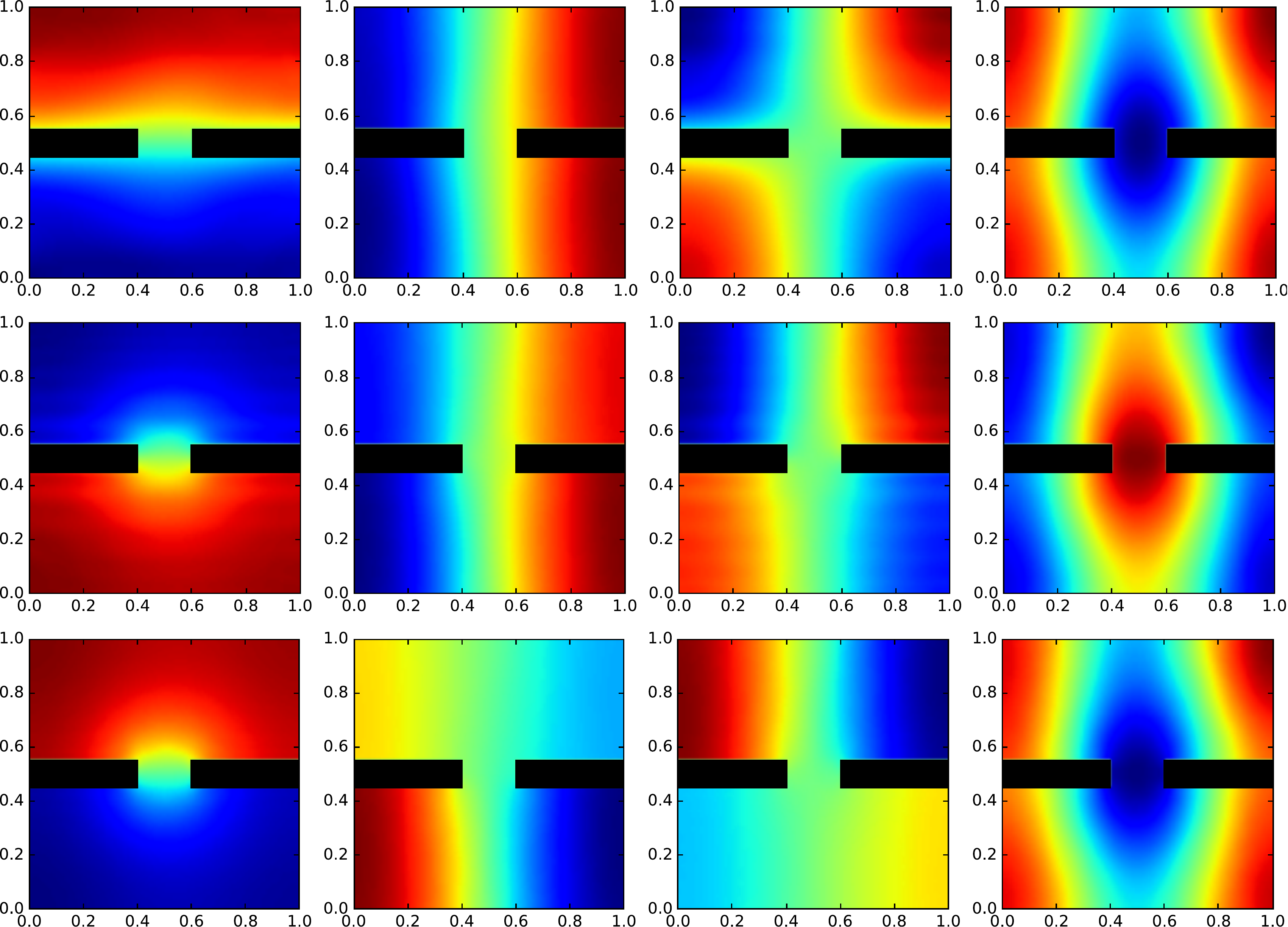}
	\caption{Illustration of SFA in two rooms based on karthesian coordinates as sensory input. The walk consists of $200$k steps with random direction and fixed step size of $0.02$. \textsl{Top}: Using monomials up to the 40th degree. \textsl{Center}: Using Legendre polynomials up to the 40th degree. \textsl{Bottom}: Using Legendre polynomials up to the 80th degree.}
	\label{fig:SFA-two-rooms-xy}
\end{figure}

Figure \ref{fig:SFA-two-rooms-xy} presents our results. Note that the top line, based on monomials up to the 40th degree, is visually approximately identical to monomials up to the 20th degree or Legendre polynomials up to the 20th degree. This illustrates the fact that monomials do not actually add new data representation from a certain degree onwards. Due to limited floating point precision they effectively compress all information to zero above that degree. As soon as we switch to Legendre polynomials, while keeping the same degree, we can clearly see how the features (center row) become closer to those based on the wall sensor in figure \ref{fig:SFA-two-rooms}. Finally, using Legendre polynomials up to degree 80 or higher, we get approximately the same features as previously in figure \ref{fig:SFA-two-rooms}. From this equivalence we conclude that the source-yielding components are visually close to the unknown ideal SFA-solutions of the two-room scenario. This hypothesis is supported by the high degree of nonlinear expansion that was applied.

Besides demonstrating the advantage of Legendre polynomials for nonlinear expansion, figure \ref{fig:SFA-two-rooms-xy} illustrates the transition from the single room case (Figure \ref{fig:SFA-2D-d4-8}) to the two-room-case in figure \ref{fig:SFA-two-rooms}. The lower the degree, the less perception for the wall is represented in the features.

\subsection{Three rooms} \label{sec:threeRooms}

In this section we study a more complex example. Our environment here consists of three rooms forking from a central corridor.

\begin{figure}[ht]
	\centering
	\captionsetup{width=.90\linewidth}
	\includegraphics[width=1.0\hsize]{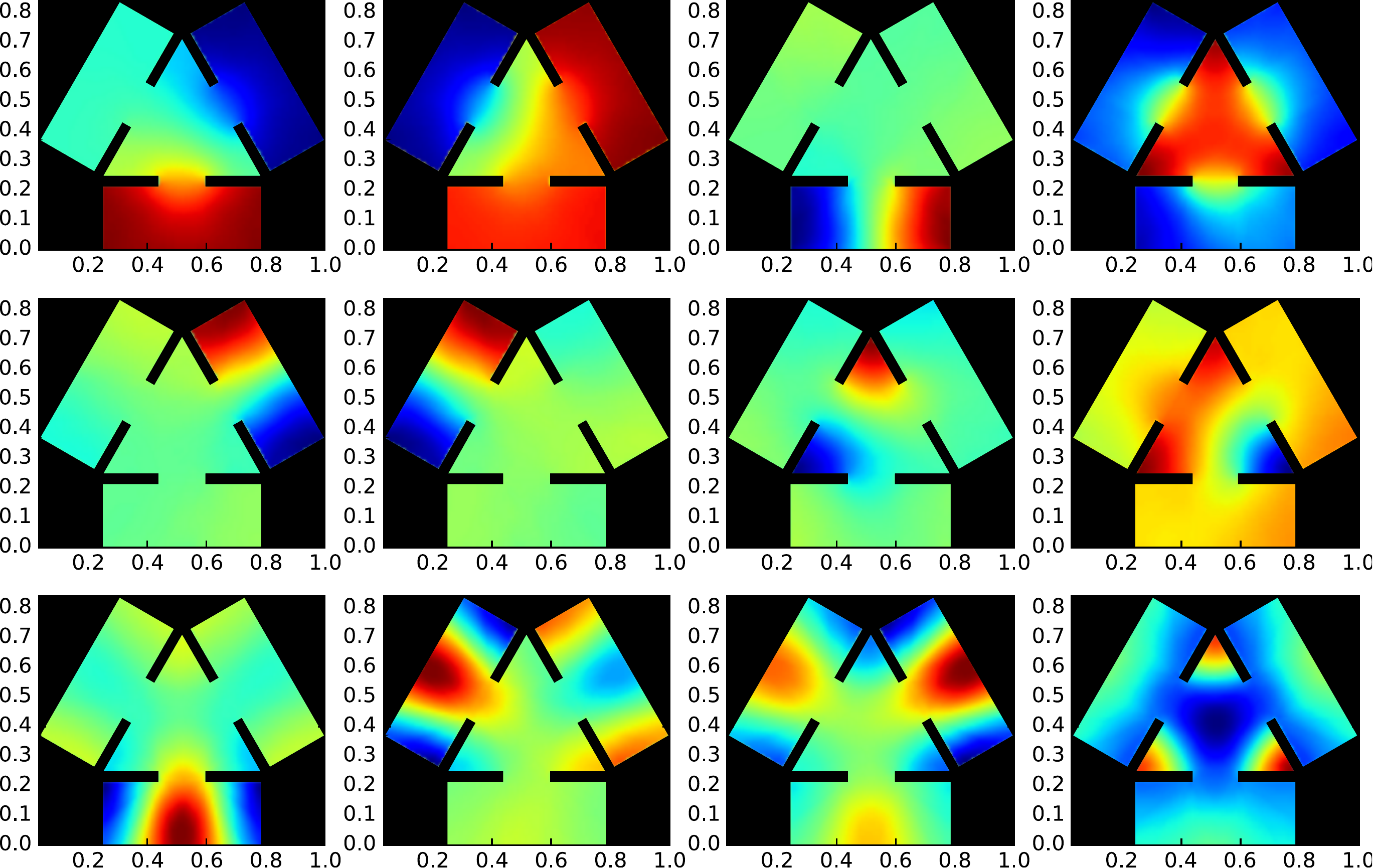}
	\caption{Illustration of SFA in three rooms with central corridor. The walk consists of $200$k steps with random direction and fixed step size of $0.02$. Nonlinear expansion was performed using monomials up to the second degree. The twelve slowest components are shown.}
	\label{fig:SFA-three-rooms}
\end{figure}

Figure \ref{fig:SFA-three-rooms} presents the features found by SFA. Like in the two-room example the first component is the slowest source and spans multiple rooms. It corresponds to the longest non-cyclic path that can be fitted into the environment. With non-cyclic we denote that the path must connect two points in environment space without detour.
Note that the first component leaves one of the rooms plain. This is because each single source is a one-dimensional feature, embedded into a higher dimensional space. Thus it cannot span all three rooms in the given layout.

So, in contrast to the two room setting, we find the second component to span multiple rooms as well. It orthogonally connects the room that was previously plain with the path indicated by the first component. However, the second source is not represented purely, but is intermixed with the second harmonic of the first source. Figure \ref{fig:SFA-three-rooms_nav} demonstrates the logic behind algorithm \ref{alg:navSFA} with subroutines,
i.e.\ that navigation still succeeds in this case. The figure further asserts that the first two components are in combination sufficient to navigate into the right room from any starting point.
The fourth component is an intermix of the second harmonics of the first two sources.

Components 3, 5 and 6 introduce the room-scoped sources we already found in earlier examples. We find the familiar pattern for each of the rooms, components 9, 10, 11 corresponding to the second harmonic of the room-internal sources.

\begin{figure}[ht]
	\centering
	\captionsetup{width=.90\linewidth}
	\includegraphics[width=1.0\hsize]{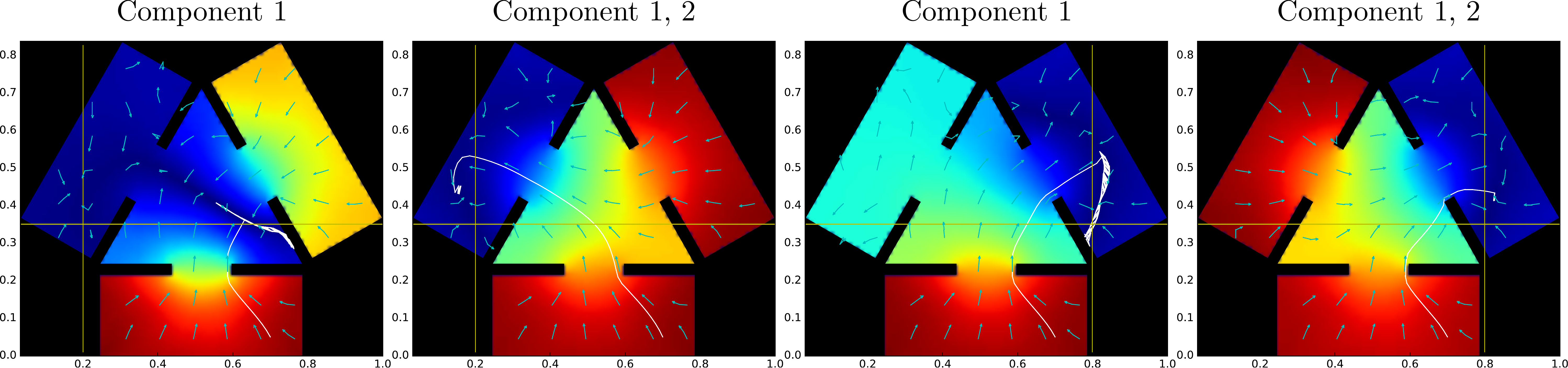}
	\caption{Illustration of SFA-based navigation using approximately independent sources. For two different tasks the navigation is displayed using only the first or the first two components.}
	\label{fig:SFA-three-rooms_nav}
\end{figure}

To study the suitability of the obtained components for navigation, we present two tasks in figure \ref{fig:SFA-three-rooms_nav}: From bottom room to the left and from bottom room to the right.
Given that the first source spans the bottom room and the right room, while leaving the left room flat, it can directly guide the agent from bottom to the right. To reach the left room, the second component is required, but the first one is still a useful prerequisite: It serves to guide the agent out of the starting room through the pathway into the corridor, settling it in front of the correct door. The second component is suitable to pull it into the room. Both navigation tasks would reach the goal position precisely if higher components were taken into account. In this demonstration we stopped after the second component, having the agent in the correct destination room, fairly close to the goal position.

\subsection{Four rooms} \label{sec:fourRooms}

Extending the scenario by another room we can assert that the procedure scales well. The results here are friendlier for interpretation, because the number of rooms is a multiple of two. Since each source is one dimensional it can connect two rooms, allowing for an even split of the overall structure into sources. Figure \ref{fig:SFA-four-rooms} displays this effect in the sense that each source appears in a pair consisting of a horizontal and vertical counterpart. Note that SFA retrieves each source purely in this example. Due to the very clean and interpretable result, we hypothesize that the components shown in figure \ref{fig:SFA-four-rooms} are visually close to the unknown ideal SFA-solutions of this environment.
%The components 1, 2, 5, 6, 7, 8 yield the pure sources that describe this environment.

\begin{figure}[ht]
	\centering
	\captionsetup{width=.90\linewidth}
	\includegraphics[width=1.0\hsize]{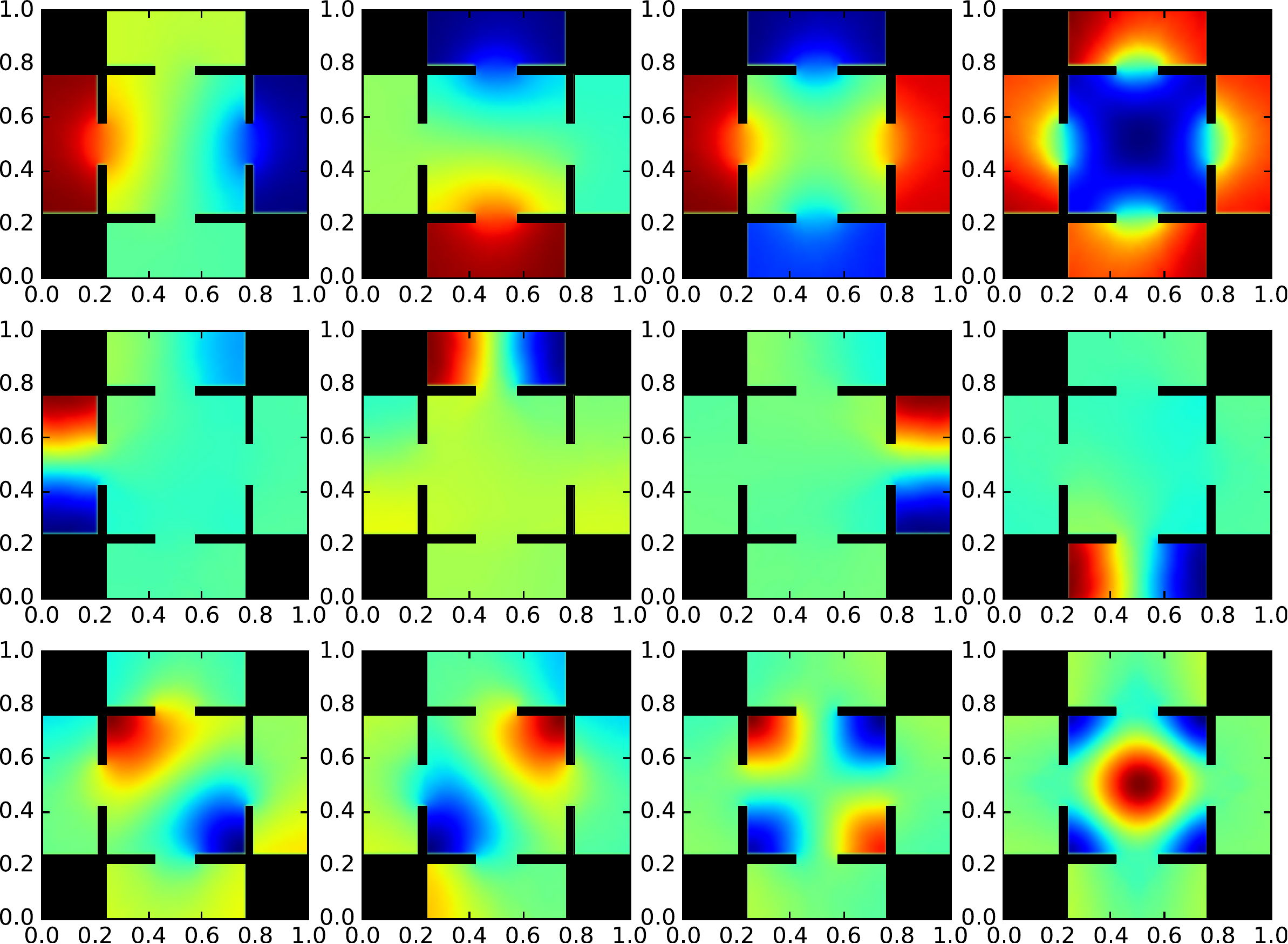}
	\caption{Illustration of SFA in four rooms with central corridor. The walk consists of $200$k steps with random direction and fixed step size of $0.02$. Nonlinear expansion was performed using monomials up to the second degree. The twelve slowest components are shown.}
	\label{fig:SFA-four-rooms}
\end{figure}

Like in the previous experiments we get some components that span multiple rooms (1, 2), later followed by room-internal sources (5, 6, 7, 8). Components 2 and 3 are mixtures of the second harmonics of the first two sources. The first two components correspond again to the longest non-cyclic paths that can be fitted into the environment.

\begin{figure}[ht]
	\centering
	\captionsetup{width=.90\linewidth}
	\includegraphics[width=1.0\hsize]{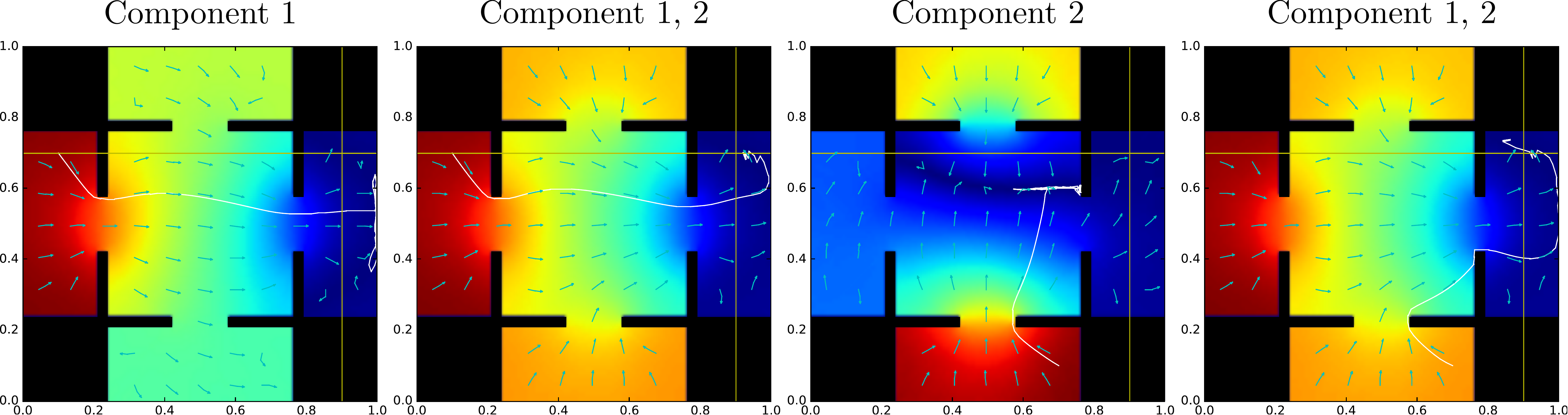}
	\caption{Illustration of SFA-based navigation using approximately independent sources. For two different tasks the navigation is displayed using only one or two components.}
	\label{fig:SFA-four-rooms_nav}
\end{figure}

We illustrate two navigation tasks in figure \ref{fig:SFA-four-rooms_nav}, roughly equivalent to those from the three-room example. Navigation along one of the first sources guides the agent directly into the destination room, like the navigation from left to right shows. Navigating around the corner, e.g.\ from bottom to right is only feasible using two components.

\subsection{Three rooms, asymmetric}

In this example we investigate a less symmetric arrangement of three rooms with a large rectangular corridor. Figure \ref{fig:SFA-three-rooms-corridor} illustrates clearly how the slowest source corresponds to the longest path that directly connects two points in the environment. It spans two rooms and the corridor, demonstrating that the rooms are not chosen arbitrarily but that it spans specifically the most distant rooms while leaving the central room plain.

The second source takes the formerly plain central room into account, yielding a mixture with the second harmonic of the first source across the rest of the environment. Yet again, the first two components are feasible to direct the agent into the destination room, which is illustrated by two navigation tasks in figure \ref{fig:SFA-three-rooms-corridor_nav}: Left to right and left to center.
Like in the symmetric three-room example, the second component is required for entering the central room, which is flatly represented by the first component. Without the second component the agent is at least guided to the entrance of the destination room.

Due to the large corridor we must consider more components than in earlier examples, before we find the typical room-internal sources in components 12, 14 and 15 displayed in figure \ref{fig:SFA-three-rooms-corridor}.

\begin{figure}[ht]
	\centering
	\captionsetup{width=.90\linewidth}
	\includegraphics[width=1.0\hsize]{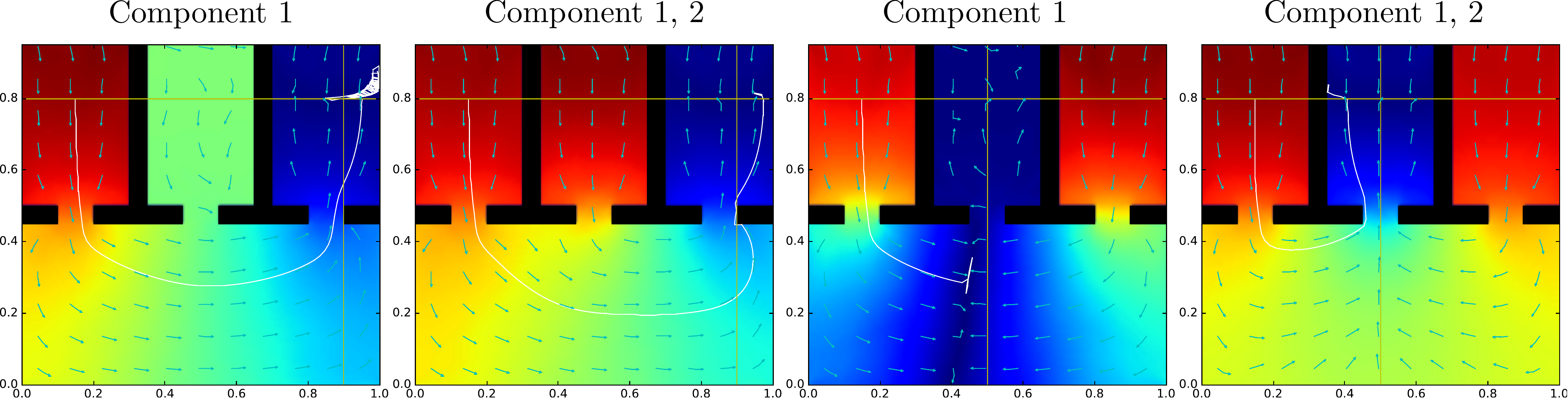}
	\caption{Illustration of SFA-based navigation using approximately independent sources. For two different tasks the navigation is displayed using only the first or the first two components.}
	\label{fig:SFA-three-rooms-corridor_nav}
\end{figure}

\begin{figure}[ht]
	\centering
	\captionsetup{width=.90\linewidth}
	\includegraphics[width=1.0\hsize]{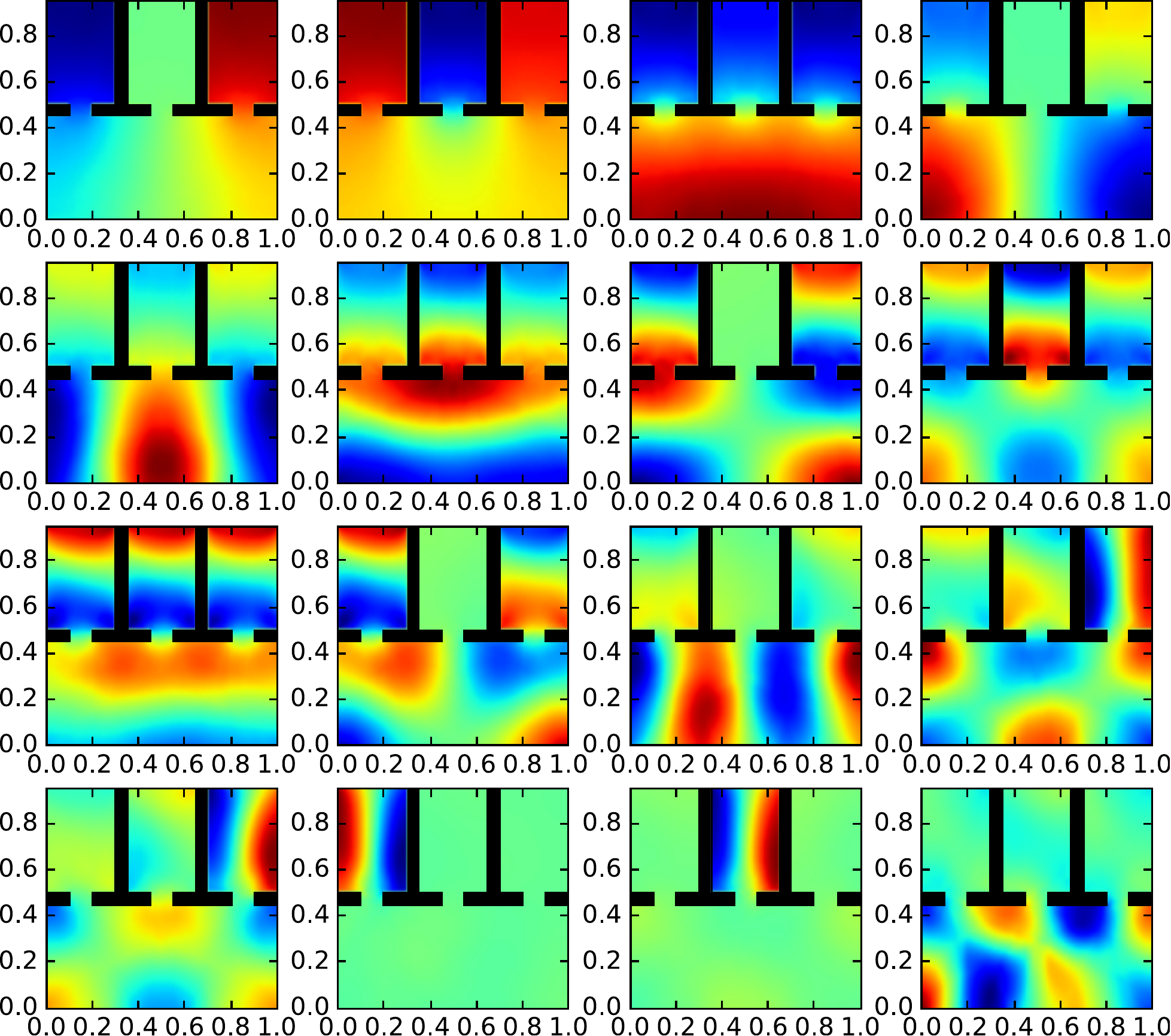}
	\caption{Illustration of SFA in three rooms with large corridor at the bottom. The walk consists of $200$k steps with random direction and fixed step size of $0.02$. Nonlinear expansion was performed using monomials up to the second degree. The sixteen slowest components are shown.}
	\label{fig:SFA-three-rooms-corridor}
\end{figure}

\subsection{Navigation with obstacle} \label{sec:obstacle}

This example deals with the obstacle environment familiar from \cite{2017arXiv171200634R}. Figure \ref{fig:SFA-obstacle} displays the first eight components. Indeed, the first component is suitable to guide an agent vertically around the obstacle if the navigation task requires it. In figure \ref{fig:SFA-obstacle_nav}/left the navigation strives the obstacle because the obtained features were not sufficiently predictable. Note that the actual gradient would have avoided the obstacle. Predictability can be improved by a using a higher expansion.

Surrounding the obstacle horizontally does not work based on the first component alone.
A fundamental difference to the previous examples is that this environment yields no simply connected space. This issue manifests in the fact that the first source -- while being monotonic -- might still yield local optima for navigation. These can be circumvented by using higher components, but might in general require an additional routine for component selection.

\begin{figure}[ht]
	\centering
	\captionsetup{width=.90\linewidth}
	\includegraphics[width=1.0\hsize]{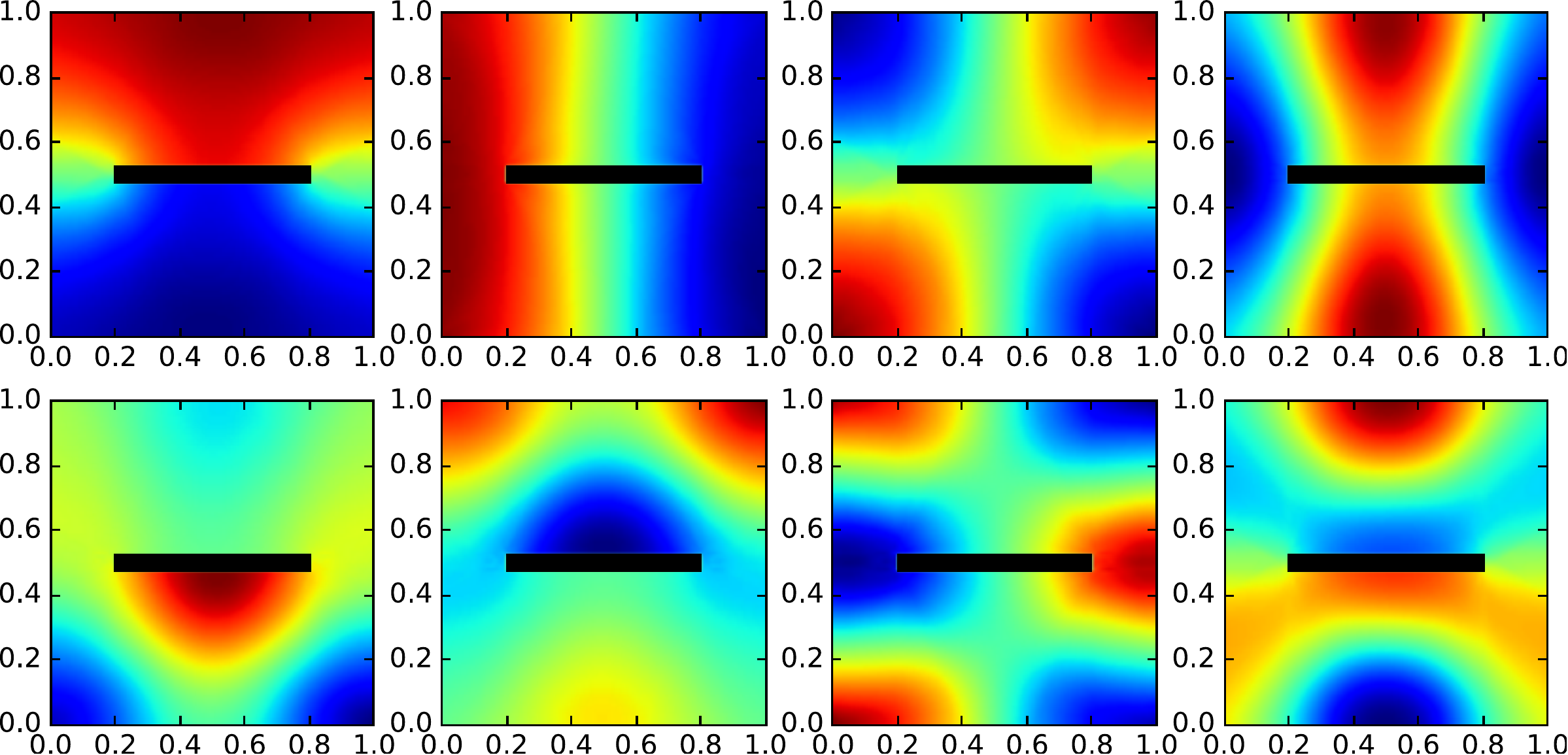}
	\caption{Illustration of SFA in an environment with obstacle. The walk consists of $200$k steps with random direction and fixed step size of $0.02$. Nonlinear expansion was performed using monomials up to the fourth degree. The eight slowest components are shown.}
	\label{fig:SFA-obstacle}
\end{figure}

\begin{figure}[ht]
	\centering
	\captionsetup{width=.90\linewidth}
	\includegraphics[width=1.0\hsize]{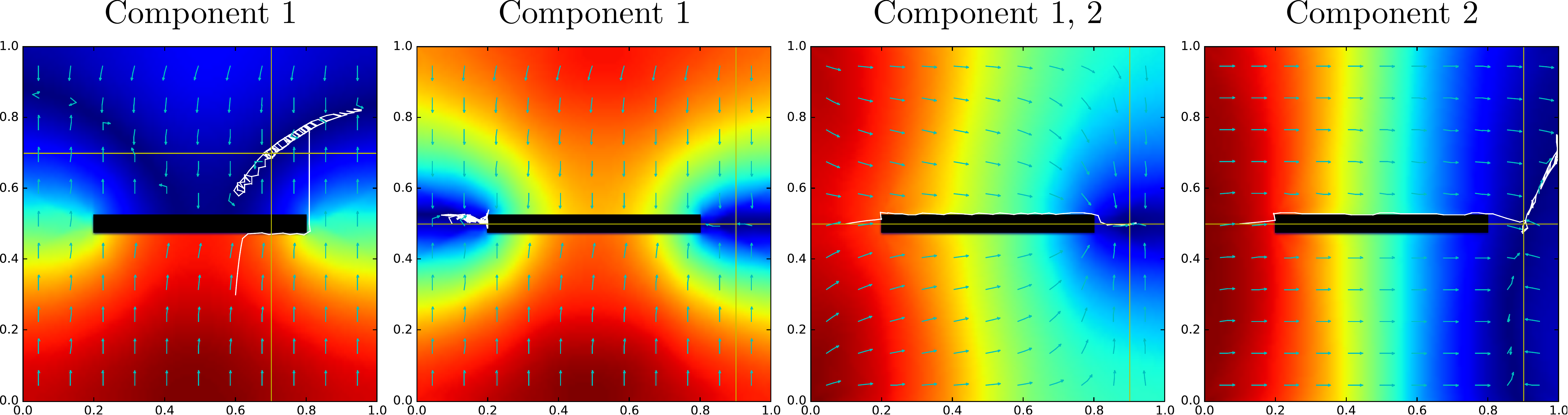}
	\caption{Illustration of SFA-based navigation around an obstacle for two navigation tasks. \textsl{Left}: Due to poor prediction the agent strives the obstacle, but can still reach the goal. \textsl{Second}: The first component is not feasible to guide around the obstacle horizontally. \textsl{Third}: The second component dominates sufficiently to guide the agent around. \textsl{Right}: The second component alone would be sufficient in this case.}
	\label{fig:SFA-obstacle_nav}
\end{figure}

The effect is illustrated in figure \ref{fig:SFA-obstacle_nav}/second, where we have an example where navigation along the first component ends up in the wrong state. Adding the second component solves the issue in this case (Figure \ref{fig:SFA-obstacle_nav}/third), because it dominates and can guide the agent to the right destination.
%A special routine could detect this case and skip the first component temporarily. Figure \ref{fig:SFA-obstacle_nav}/fourth shows that the second component alone is feasible to solve this task. A proper extension of this navigation principle to environments that are not simply connected is subject to ongoing research.

\subsection{Pendulum swing-up with limited torque}

We finally apply SFA to a dynamical system manipulation task. Pendulum swing-up with limited torque is a classical RL problem where the agent controls a harmonic pendulum by directly applying a torque to its fixpoint. The goal is to bring the pendulum into vertical stand-up position. The torque is limited such that the agent cannot simply turn up the pendulum, but has to swing it. With the goal being a specific destination state the reward function can be translated to our setting. We perceive the control task as a navigation task in the phase space of the pendulum. In this section we plot velocity on the $x$ axis and angular amplitude on the $y$ axis. The goal state is therefore at the center of the top boundary with maximum amplitude (pendulum pointing up) and zero velocity.

\begin{figure}[ht]
	\centering
	\captionsetup{width=.60\linewidth}
	\includegraphics[width=0.5\hsize]{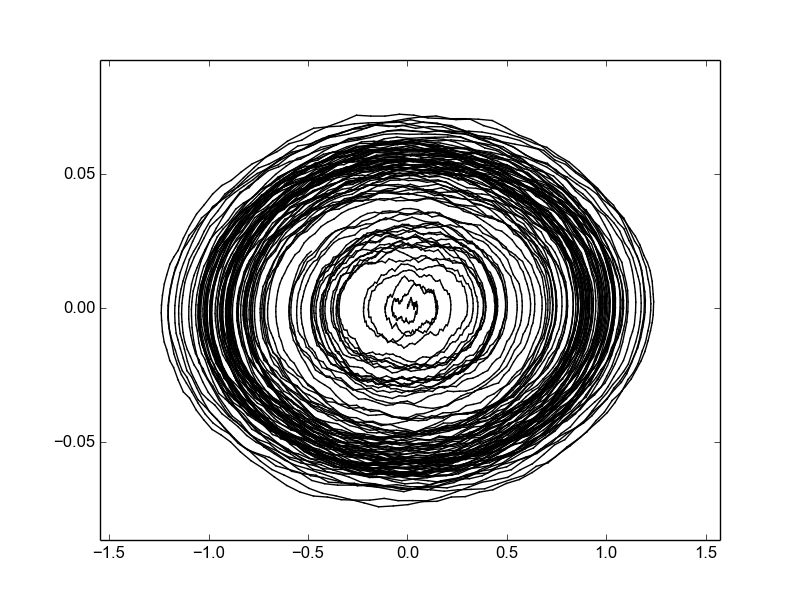}
	\caption{Phase space of a harmonic pendulum during training. The pendulum is controlled randomly over $10000$ steps.}
	\label{fig:SFA-pendulum-training}
\end{figure}

Data for the training phase is generated by controlling the pendulum randomly (Figure \ref{fig:SFA-pendulum-training}).

\begin{figure}[ht]
	\centering
	\captionsetup{width=.90\linewidth}
	\includegraphics[width=1.0\hsize]{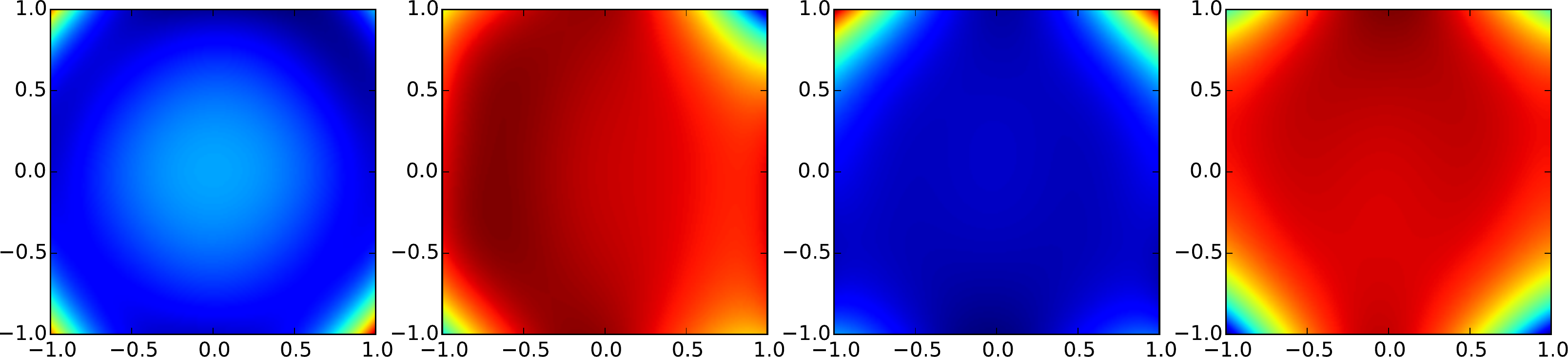}
	\caption{Illustration of SFA on the phase space of a harmonic pendulum. The four slowest components are shown.}
	\label{fig:SFA-pendulum}
\end{figure}

The slowest component in figure \ref{fig:SFA-pendulum} shows a clear representation of the circular nature of the training phase. Its gradient captures the rule that for reaching any position at the boundary it is crucial to walk away from the center into any direction. Navigating along the first component tells the agent to swing the pendulum, regardless of the precise goal location. Considering the higher components when a sufficient velocity is acquired, will settle in the specific goal state.

%For technical reasons, we cannot provide a navigation example for this setting as of this writing. It may be added in a later revision of this paper.

\section{Discussion}

\subsection{Summary}

We combined the unsupervised PFAx and SFA algorithms to efficiently model a previously unknown environment in a way that is suitable for globally solving navigation and control tasks.
%The approach hardly requires assumptions about the sensor an control signals besides general sanity properties like ergodicity and not being chaotic.
PFAx is used to utilize the command signal that controls the agent and to model a low dimensional space of well manipulatable features. SFA is applied on top of that to decompose the global structure of the environment into monotonic features that allow PFAx-driven local navigation to find a global optimum.

To explain and support this decomposition, the monotonicity and geometrical properties of the obtained features, we extended mathematical theory of SFA and xSFA to manifolds. Former applications of SFA mainly exploited its ability to find invariances. %Monotonicity of the sources was e.g.\ used to perform inversion.
In this work we explicitly utilize SFA-induced monotonicity as the link between local and global optimization.

In experiments of stepwise increased complexity we demonstrated how this principle scales to multiple rooms and leads to a hierarchical decomposition of the environment into components that yield globally solvable tasks. As soon as such a decomposition is achieved, it can be used to efficiently solve a whole range of tasks. The corresponding goal states do not need to be known during the training phase as the decomposition -- if complete and sufficiently accurate -- is suitable to solve for any possible goal state in the environment.
Finally we demonstrated applicability to phase spaces of dynamical systems.

The whole procedure is engineered to fully consist of computationally efficient building blocks, most notably of eigenvalue decompositions.
A remaining goal is to assemble a better scalable implementation, e.g.\ by applying PFAx hierarchically or using incremental implementations of SFA and PFAx. With this in line, the principle can be used on higher dimensional sensors and thus in more realistic and complex settings.

\subsection{Conclusion and perspective} \label{sec:conclusion}

We conclude that SFA and PFAx can augment each other to form a framework for globally solving tasks that involve navigation and control. Especially SFA typically yields results that are well interpretable and can provide new insights into the structure of such a task. These insights are valuable beyond the computation of a solution. PFAx on the other hand provides insight into the relation between sensor and control signal.

Both building blocks -- SFA and PFAx -- are dimensionality reduction algorithms suitable to operate on approximately continuous signals. As such they form a promising approach to deal with the curse of dimensionality that traditionally affects the RL setting.
Given that the described procedure is well scalable, e.g.\ by using a hierarchical or incremental setup, this framework can form the basis of an efficient and flexible engine for continuous reinforcement learning. Most notably, the method only needs to compute a single model -- without having to know the actual goal state -- and is finally capable to achieve any goal state based on this single model. This characteristic is especially valuable for use cases where the goal can suddenly change over time.

To cover the original RL setting, a better exploration method and incorporation of an arbitrary reward signal are still gaps to be closed.
\begin{itemize}
	 \item To improve exploration efficiency, the algorithm could operate in an online mode during exploration and choose exploration steps by curiosity, i.e.\ by aiming for the largest possible change of the agent's current state in feature space.
	 \item An arbitrary reward signal or its accumulation could be tracked as an additional sensor component during training phase. Later, the navigation goal could be formulated in terms of this sensor component.
\end{itemize}

In section \ref{sec:introduction} we mentioned some simplifying assumptions:
\begin{itemize}
	\item The environment is fully observable, i.e.\ every position yields a unique representation in sensor space
	\item The environment is stationary, i.e.\ constant over time, contains no blinking lights, no flickering colors or moving objects
\end{itemize}

%Note that unique sensor representations in a stationary environment imply that the environment is fully observable as the agent can conclude the state of the whole environment from its own current state.
Fully observable does \textsl{not} mean that every pair of locations must be visible from each other. This was demonstrated in various scenarios involving multiple rooms, where the goal position is not visible from the starting position, see section \ref{sec:experiments}.

The named limitations are not inherent and we can readily suggest extensions of the algorithm to overcome each of them:

\begin{itemize}
%	\item A not simply connected environment can be handled by an additional routine for weighting the components for optimization. In section \ref{sec:obstacle} we provide an example of this type that is even readily solved by the algorithm in its current form.
	\item An environment with ambiguous sensor representations can be handled by time embedding or incorporating an episodic memory module into the sensor signal. Steepness in SFA components can be used -- in sense of a model for surprise -- to trigger memory write access.
	\item A time-dynamic environment can be handled by time embedding or by setting the time-order parameter of PFAx ($p$) sufficiently high. To be feasible this would likely require a hierarchical PFAx implementation (c.f.\ hierarchical SFA, \cite{FranziusSprekelerEtAl-2007e, Schoenfeld2015}).
\end{itemize}

Each of the proposals in this section would yield a significant extension to the algorithm and could fill a future publication on its own.

\section*{Acknowledgments}
This work is funded by a grant from the German Research Foundation (Deutsche Forschungsgemeinschaft, DFG) to L. Wiskott (SFB 874, TP B3) and
supported by the German Federal Ministry of
Education and Research within the National Network Computational
Neuroscience - Bernstein Fokus: “Learning behavioral models: From human
experiment to technical assistance”, grant FKZ 01GQ0951.

\pagebreak
\bibliography{MachineLearning,Miscellaneous,Predictability,ReinforcementLearning,SFA,Extern} %,WiskottGroup}

\begin{thebibliography}{}

\bibitem[Bae et~al., 2013]{Bae2013}
Bae, S.~W., Korman, M., and Okamoto, Y. (2013).
\newblock The geodesic diameter of polygonal domains.
\newblock {\em Discrete {\&} Computational Geometry}, 50(2):306--329.

\bibitem[Berkes and Wiskott, 2002]{10.1007/3-540-46084-5_14}
Berkes, P. and Wiskott, L. (2002).
\newblock Applying slow feature analysis to image sequences yields a rich
  repertoire of complex cell properties.
\newblock In Dorronsoro, J.~R., editor, {\em Artificial Neural Networks ---
  ICANN 2002}, pages 81--86, Berlin, Heidelberg. Springer Berlin Heidelberg.

\bibitem[B{\"o}hmer et~al., 2013]{JMLR:v14:boehmer13a}
B{\"o}hmer, W., Gr{\"u}new{\"a}lder, S., Shen, Y., Musial, M., and Obermayer,
  K. (2013).
\newblock Construction of approximation spaces for reinforcement learning.
\newblock {\em Journal of Machine Learning Research}, 14:2067--2118.

\bibitem[B{\"{o}}hmer et~al., 2015]{DBLP:journals/ki/BohmerSBRO15}
B{\"{o}}hmer, W., Springenberg, J.~T., Boedecker, J., Riedmiller, M.~A., and
  Obermayer, K. (2015).
\newblock Autonomous learning of state representations for control: An emerging
  field aims to autonomously learn state representations for reinforcement
  learning agents from their real-world sensor observations.
\newblock {\em {KI}}, 29(4):353--362.

\bibitem[Botvinick et~al., 2009]{BOTVINICK2009262}
Botvinick, M.~M., Niv, Y., and Barto, A.~C. (2009).
\newblock Hierarchically organized behavior and its neural foundations: A
  reinforcement learning perspective.
\newblock {\em Cognition}, 113(3):262 -- 280.
\newblock Reinforcement learning and higher cognition.

\bibitem[Box and Tiao, 1977]{boxTiao1977}
Box, G. E.~P. and Tiao, G.~C. (1977).
\newblock A canonical analysis of multiple time series.
\newblock {\em Biometrika}, 64(2):pp. 355--365.

\bibitem[Brown, 1994]{Brown94themixed}
Brown, R. (1994).
\newblock The mixed problem for laplace's equation in a class of lipschitz
  domains.
\newblock {\em Comm. Partial Diff. Eqns}, 19.

\bibitem[Engedy and Horvath, 2009]{5286557}
Engedy, I. and Horvath, G. (2009).
\newblock Artificial neural network based mobile robot navigation.
\newblock In {\em 2009 IEEE International Symposium on Intelligent Signal
  Processing}, pages 241--246.

\bibitem[Escalante-B. and Wiskott, 2013]{Escalante-B.Wiskott-2013b}
Escalante-B., A.~N. and Wiskott, L. (2013).
\newblock How to solve classification and regression problems on
  high-dimensional data with a supervised extension of {S}low {F}eature
  {A}nalysis.
\newblock {\em Journal of Machine Learning Research}, 14:3683--3719.

\bibitem[Fletcher et~al., 2004]{1318725}
Fletcher, P.~T., Lu, C., Pizer, S.~M., and Joshi, S. (2004).
\newblock Principal geodesic analysis for the study of nonlinear statistics of
  shape.
\newblock {\em IEEE Transactions on Medical Imaging}, 23(8):995--1005.

\bibitem[Franzius et~al., 2007]{FranziusSprekelerEtAl-2007e}
Franzius, M., Sprekeler, H., and Wiskott, L. (2007).
\newblock Slowness and sparseness lead to place-, head direction-, and
  spatial-view cells.
\newblock In {\em Proc. 3rd Annual Computational Cognitive Neuroscience
  Conference, Nov. 1--2, San Diego, USA}, pages III--8.

\bibitem[Franzius et~al., 2008]{10.1007/978-3-540-87536-9_98}
Franzius, M., Wilbert, N., and Wiskott, L. (2008).
\newblock Invariant object recognition with slow feature analysis.
\newblock In K{\r{u}}rkov{\'a}, V., Neruda, R., and Koutn{\'i}k, J., editors,
  {\em Artificial Neural Networks - ICANN 2008}, pages 961--970, Berlin,
  Heidelberg. Springer Berlin Heidelberg.

\bibitem[Franzius et~al., 2011]{FranziusWilbertEtAl-2011}
Franzius, M., Wilbert, N., and Wiskott, L. (2011).
\newblock Invariant object recognition and pose estimation with slow feature
  analysis.
\newblock {\em Neural Computation}, 23(9):2289--2323.

\bibitem[Garrido et~al., 2006]{4058742}
Garrido, S., Moreno, L., Abderrahim, M., and Martin, F. (2006).
\newblock Path planning for mobile robot navigation using voronoi diagram and
  fast marching.
\newblock In {\em 2006 IEEE/RSJ International Conference on Intelligent Robots
  and Systems}, pages 2376--2381.

\bibitem[Goerg, 2013]{goerg13}
Goerg, G. (2013).
\newblock Forecastable component analysis.
\newblock In Dasgupta, S. and Mcallester, D., editors, {\em Proceedings of the
  30th International Conference on Machine Learning (ICML-13)}, volume~28,
  pages 64--72. JMLR Workshop and Conference Proceedings.

\bibitem[Goldberger et~al., 2004]{DBLP:conf/nips/GoldbergerRHS04}
Goldberger, J., Roweis, S.~T., Hinton, G.~E., and Salakhutdinov, R. (2004).
\newblock Neighbourhood components analysis.
\newblock In {\em Advances in Neural Information Processing Systems 17 [Neural
  Information Processing Systems, {NIPS} 2004, December 13-18, 2004, Vancouver,
  British Columbia, Canada]}, pages 513--520.

\bibitem[Hastie and Stuetzle, 1989]{doi:10.1080/01621459.1989.10478797}
Hastie, T. and Stuetzle, W. (1989).
\newblock Principal curves.
\newblock {\em Journal of the American Statistical Association},
  84(406):502--516.

\bibitem[Hauberg, 2016]{7312494}
Hauberg, S. (2016).
\newblock Principal curves on riemannian manifolds.
\newblock {\em IEEE Transactions on Pattern Analysis and Machine Intelligence},
  38(9):1915--1921.

\bibitem[Igarashi, 2002]{Igarashi2002}
Igarashi, H. (2002).
\newblock Path planning of a mobile robot by optimization and reinforcement
  learning.
\newblock {\em Artificial Life and Robotics}, 6(1):59--65.

\bibitem[Jonschkowski and Brock, 2013]{Jonschkowski-13-ERLARS}
Jonschkowski, R. and Brock, O. (2013).
\newblock Learning task-specific state representations by maximizing slowness
  and predictability.
\newblock In {\em Proceedings of the 6th International Workshop on Evolutionary
  and Reinforcement Learning for Autonomous Robot Systems (ERLARS)}.

\bibitem[Kollar and Roy, 2008]{doi:10.1177/0278364907087426}
Kollar, T. and Roy, N. (2008).
\newblock Trajectory optimization using reinforcement learning for map
  exploration.
\newblock {\em The International Journal of Robotics Research}, 27(2):175--196.

\bibitem[Kun~Su and Hu, 2015]{Su_robotpath}
Kun~Su, Y.~W. and Hu, X. (2015).
\newblock Robot path planning based on random coding particle swarm
  optimization.
\newblock {\em International Journal of Advanced Computer Science and
  Applications(IJACSA)}, 6(4).

\bibitem[Lagoudakis et~al., 2002]{Lagoudakis:2002:LMR:645861.670291}
Lagoudakis, M.~G., Parr, R., and Littman, M.~L. (2002).
\newblock Least-squares methods in reinforcement learning for control.
\newblock In {\em Proceedings of the Second Hellenic Conference on AI: Methods
  and Applications of Artificial Intelligence}, SETN '02, pages 249--260,
  London, UK, UK. Springer-Verlag.

\bibitem[Luciw and Schmidhuber, 2012]{DBLP:conf/icann/LuciwS12}
Luciw, M.~D. and Schmidhuber, J. (2012).
\newblock Low complexity proto-value function learning from sensory
  observations with incremental slow feature analysis.
\newblock In {\em Artificial Neural Networks and Machine Learning - {ICANN}
  2012 - 22nd International Conference on Artificial Neural Networks, Lausanne,
  Switzerland, September 11-14, 2012, Proceedings, Part {II}}, pages 279--287.

\bibitem[Mahadevan and Maggioni, 2007]{MahadevanMaggioni-2007}
Mahadevan, S. and Maggioni, M. (2007).
\newblock Proto-value functions: A laplacian framework for learning
  representation and control in markov decision processes.
\newblock {\em Journal of Machine Learning Research}, 8(2169-2231):16.

\bibitem[Mattheij and Söderlind, 1987]{MATTHEIJ1987507}
Mattheij, R. and Söderlind, G. (1987).
\newblock On inhomogeneous eigenvalue problems. i.
\newblock {\em Linear Algebra and its Applications}, 88-89(Supplement C):507 --
  531.

\bibitem[McGovern and Barto, 2001]{DBLP:conf/icml/McGovernB01}
McGovern, A. and Barto, A.~G. (2001).
\newblock Automatic discovery of subgoals in reinforcement learning using
  diverse density.
\newblock In {\em Proceedings of the Eighteenth International Conference on
  Machine Learning {(ICML} 2001), Williams College, Williamstown, MA, USA, June
  28 - July 1, 2001}, pages 361--368.

\bibitem[Metka et~al., 2017]{8202307}
Metka, B., Franzius, M., and Bauer-Wersing, U. (2017).
\newblock Efficient navigation using slow feature gradients.
\newblock In {\em 2017 IEEE/RSJ International Conference on Intelligent Robots
  and Systems (IROS)}, pages 1311--1316.

\bibitem[Richthofer and Wiskott, 2015]{DBLP:conf/icmla/RichthoferW15}
Richthofer, S. and Wiskott, L. (2015).
\newblock Predictable feature analysis.
\newblock In {\em 14th {IEEE} International Conference on Machine Learning and
  Applications, {ICMLA} 2015, Miami, FL, USA, December 9-11, 2015}, pages
  190--196.

\bibitem[{Richthofer} and {Wiskott}, 2017]{2017arXiv171200634R}
{Richthofer}, S. and {Wiskott}, L. (2017).
\newblock {PFAx: Predictable Feature Analysis to Perform Control}.
\newblock {\em ArXiv e-prints}.

\bibitem[Romero-Martí et~al., 2016]{7955160}
Romero-Martí, D.~P., Núñez-Varela, J.~I., Soubervielle-Montalvo, C., and
  de-la Paz, A.~O. (2016).
\newblock Navigation and path planning using reinforcement learning for a
  roomba robot.
\newblock In {\em 2016 XVIII Congreso Mexicano de Robotica}, pages 1--5.

\bibitem[Schapiro et~al., 2013]{Schapiro2013}
Schapiro, A.~C., Rogers, T.~T., Cordova, N.~I., Turk-Browne, N.~B., and
  Botvinick, M.~M. (2013).
\newblock Neural representations of events arise from temporal community
  structure.
\newblock 16:486 EP --.
\newblock Article.

\bibitem[Schönfeld and Wiskott, 2015]{Schoenfeld2015}
Schönfeld, F. and Wiskott, L. (2015).
\newblock Modeling place field activity with hierarchical slow feature
  analysis.
\newblock {\em Front Comput Neurosci}, 9:51.
\newblock 26052279[pmid].

\bibitem[Singh et~al., 1994]{NIPS1993_843}
Singh, S.~P., Barto, A.~G., Grupen, R., and Connolly, C. (1994).
\newblock Robust reinforcement learning in motion planning.
\newblock In Cowan, J.~D., Tesauro, G., and Alspector, J., editors, {\em
  Advances in Neural Information Processing Systems 6}, pages 655--662.
  Morgan-Kaufmann.

\bibitem[Sprague, 2009]{DBLP:conf/ijcai/Sprague09}
Sprague, N. (2009).
\newblock Predictive projections.
\newblock In {\em {IJCAI} 2009, Proceedings of the 21st International Joint
  Conference on Artificial Intelligence, Pasadena, California, USA, July 11-17,
  2009}, pages 1223--1229.

\bibitem[Sprague, 2014]{DBLP:conf/icann/Sprague14}
Sprague, N. (2014).
\newblock Contingent features for reinforcement learning.
\newblock In {\em Artificial Neural Networks and Machine Learning - {ICANN}
  2014 - 24th International Conference on Artificial Neural Networks, Hamburg,
  Germany, September 15-19, 2014. Proceedings}, pages 347--354.

\bibitem[Sprekeler, 2011]{doi:10.1162/NECO_a_00214}
Sprekeler, H. (2011).
\newblock On the relation of slow feature analysis and laplacian eigenmaps.
\newblock {\em Neural Computation}, 23(12):3287--3302.
\newblock PMID: 21105830.

\bibitem[Sprekeler and Wiskott, 2008]{SprekelerWiskott-2008}
Sprekeler, H. and Wiskott, L. (2008).
\newblock Understanding {S}low {F}eature {A}nalysis: a mathematical framework.
\newblock Cognitive Sciences EPrint Archive (CogPrints).

\bibitem[Sprekeler et~al., 2014]{SprekelerZitoEtAl-2014}
Sprekeler, H., Zito, T., and Wiskott, L. (2014).
\newblock An extension of {S}low {F}eature {A}nalysis for nonlinear blind
  source separation.
\newblock {\em Journal of Machine Learning Research}, 15:921--947.

\bibitem[Stachenfeld et~al., 2014]{NIPS2014_5340}
Stachenfeld, K.~L., Botvinick, M., and Gershman, S.~J. (2014).
\newblock Design principles of the hippocampal cognitive map.
\newblock In Ghahramani, Z., Welling, M., Cortes, C., Lawrence, N.~D., and
  Weinberger, K.~Q., editors, {\em Advances in Neural Information Processing
  Systems 27}, pages 2528--2536. Curran Associates, Inc.

\bibitem[Stolle and Precup, 2002]{10.1007/3-540-45622-8_16}
Stolle, M. and Precup, D. (2002).
\newblock Learning options in reinforcement learning.
\newblock In Koenig, S. and Holte, R.~C., editors, {\em Abstraction,
  Reformulation, and Approximation}, pages 212--223, Berlin, Heidelberg.
  Springer Berlin Heidelberg.

\bibitem[TAN et~al., 2007]{TAN2007279}
TAN, G.-Z., HE, H., and SLOMAN, A. (2007).
\newblock Ant colony system algorithm for real-time globally optimal path
  planning of mobile robots.
\newblock {\em Acta Automatica Sinica}, 33(3):279 -- 285.

\bibitem[Vadakkepat et~al., 2000]{870304}
Vadakkepat, P., Tan, K.~C., and Ming-Liang, W. (2000).
\newblock Evolutionary artificial potential fields and their application in
  real time robot path planning.
\newblock In {\em Proceedings of the 2000 Congress on Evolutionary Computation.
  CEC00 (Cat. No.00TH8512)}, volume~1, pages 256--263 vol.1.

\bibitem[Warren, 1989]{100007}
Warren, C.~W. (1989).
\newblock Global path planning using artificial potential fields.
\newblock In {\em Proceedings, 1989 International Conference on Robotics and
  Automation}, pages 316--321 vol.1.

\bibitem[Weghenkel et~al., 2017]{Weghenkel:2017:GPF:3140707.3140724}
Weghenkel, B., Fischer, A., and Wiskott, L. (2017).
\newblock Graph-based predictable feature analysis.
\newblock {\em Mach. Learn.}, 106(9-10):1359--1380.

\bibitem[Weghenkel and Wiskott, 2018]{8353107}
Weghenkel, B. and Wiskott, L. (2018).
\newblock Slowness as a proxy for temporal predictability: An empirical
  comparison.
\newblock {\em Neural Computation}, 30(5):1151--1179.

\bibitem[Wiskott and Sejnowski, 2002]{WiskottSejnowski-2002}
Wiskott, L. and Sejnowski, T. (2002).
\newblock Slow {F}eature {A}nalysis: unsupervised learning of invariances.
\newblock {\em Neural Computation}, 14(4):715--770.

\bibitem[Xu et~al., 2017]{8242802}
Xu, Q.~L., Cai, M.~M., and Zhao, L.~H. (2017).
\newblock The robot path planning based on ant colony and particle swarm fusion
  algorithm.
\newblock In {\em 2017 Chinese Automation Congress (CAC)}, pages 411--415.

\bibitem[Zuo et~al., 2014]{6974463}
Zuo, B., Chen, J., Wang, L., and Wang, Y. (2014).
\newblock A reinforcement learning based robotic navigation system.
\newblock In {\em 2014 IEEE International Conference on Systems, Man, and
  Cybernetics (SMC)}, pages 3452--3457.

\end{thebibliography}

\pagebreak
\appendix
\section{Appendix}
\subsection{Notation overview} \label{sec:notation}
This section gives an overview of the notation used in this paper.

\begin{tabular}{p{1.1cm}p{9.7cm}ll}
$\mathbf{x}(t)$ & denotes the raw input signal.\\

$\mathbf{u}(t)$ & denotes the external information signal.\\

$\trph$ & $\coloneq~\{t_0, \ldots, t_k\}$ denotes a discrete time sequence (considered as equidistant with step size normalized to $1$). We usually
refer to $\trph$ as the \textit{training phase}.\\

$\av{\mathbf{s}(t)}_{t \in S}$ & $\coloneq~\frac{1}{\abs{S}} \sum_{t\in S} \mathbf{s}(t)$ denotes the average of some signal $\mathbf{s}$ over a finite set $S$.
For $S~=~\trph$ we just write $\av{\mathbf{s}(t)}_t$ or even $\av{\mathbf{s}}$, if it is obvious, what unbound variable is targeted.\\

%$\av{\mathbf{s}(t)}$ & $\coloneq~\frac{1}{\abs{\trph}} \sum_{t\in \trph} \mathbf{s}(t)$ denotes the time average of a signal $\mathbf{s}$ over a contextual training phase $\trph$.\\
$\mathbf{h}(\mathbf{x})$ & denotes the expansion function and usually consists of a set of monomials of low degree.\\

$\mathbf{z}(t)$ & denotes $\mathbf{h}(\mathbf{x}(t))$ after sphering.\\

%$d_s$ & denotes the number of output-components (output-dimension) of any signal $\mathbf{s}$.\\

$\mathbf{m}(t)$ & denotes the optimized output signal ($\mathbf{m}$ for \textit{model}).\\

$n$ & $\coloneq~dim(\mathbf{h}(\mathbf{x}))$
denotes the number of components to be analyzed (after expansion).\\

$n_{\mathbf{u}}$ & $\coloneq~dim(\mathbf{u})$
denotes the number of components in $\mathbf{u}$\\

%$r$ & $\coloneq~d_m$ denotes the number of extracted components (“features”).\\
$r$ & denotes the number of extracted components (“features”).\\

$\mathbf{A}, \mathbf{a}$ & denotes the matrix (or vector if $r=1$) holding the linear composition of the output-signal.
We set $\mathbf{m}(t)~=~\mathbf{A}^T\mathbf{z}(t)$.\\% ($\mathbf{m}(t)~=~\mathbf{a}^T\mathbf{z}(t)$).\\

$\mathbf{a}_i$ & denotes the $i$'th column of $\mathbf{A}$, so we can write $m_i(t)~=~\mathbf{a}_i^T \mathbf{z}(t)$.\\

$\orth(n)$ & $\subset \mathbb{R}^{n \times n}$ denotes the orthogonal group of dimension $n$, i.e.\ $\forall \; \mathbf{A} \in \orth(n) \colon \; \mathbf{A}\mathbf{A}^T = \mathbf{A}^T\mathbf{A} = \mathbf{I}$\\

$p$ & denotes the number of recent signal-values involved in the prediction. We also call it the \textit{prediction-order}.\\

$\mathbf{I}_{s, r}$ & denotes the $s\times r$ identity matrix ($s$ counting rows, $r$ counting columns).
For $s = r$ this is a usual square identity, while in the non-square case it consists
of a square identity block in the top or left area, filled up with zeros to fit the given shape.\\

$\mathbf{I}_{r}$ & $\coloneq~\mathbf{I}_{n, r}$\\

%$\mathbf{m}_{|_r}(t)$ & $\coloneq~\mathbf{I}_r^T \mathbf{m}(t)$\\

$\mathbf{A}_r$ & $\coloneq~\mathbf{A}\mathbf{I}_r$

%$\mathbf{B}_i, \mathbf{b}$ & denotes matrices (or scalars $b_i$ -- packed into vector $\mathbf{b}$) holding the auto-regressive prediction weights of the output-signal. \\%Thus we have $\mathbf{m}(t)~=~\mathbf{A}\mathbf{z}(t)$ ($\mathbf{m}(t)~=~\mathbf{a}\mathbf{z}(t)$).\\

% $\mathbf{Q}_i, \mathbf{q}_i$ & denotes matrices (or vectors) holding prediction weights, which predict the output-signal based on the input-signal. \\%Thus we have $\mathbf{m}(t)~=~\mathbf{A}\mathbf{z}(t)$ ($\mathbf{m}(t)~=~\mathbf{a}\mathbf{z}(t)$).\\
% 
% $\mathbf{b}_{ij}, \mathbf{q}_{ij}$ & denotes the $j$'th column of $\mathbf{B}^T_i$ or $\mathbf{Q}^T_i$ (in analogy to $\mathbf{A}, \mathbf{a}_j$).
\end{tabular}
\renewcommand{\arraystretch}{1}

Further more we sometimes use the Kronecker product $\otimes$ and the $\mvec$-operator defined as follows:

For matrices $\mathbf{A} \in \mathbb{R}^{m \times n}$ and $\mathbf{B} \in \mathbb{R}^{k \times l}$ and with $a_{ij}$ denoting the entries, $\mathbf{a}_i$ the columns of $\mathbf{A}$:

\begin{equation}
 \mathbf{A} \otimes \mathbf{B} \quad \coloneq \quad \left( \begin{matrix} a_{11}\mathbf{B}& \cdots & a_{1n}\mathbf{B} \\ \vdots & \ddots & \vdots \\ a_{m1} \mathbf{B} & \cdots & a_{mn}\mathbf{B} \end{matrix} \right) \; \in \; \mathbb{R}^{mk \times nl}
\end{equation}

\begin{equation}
 %\mvec(\mathbf{A}) \quad \coloneq \quad \left( \begin{matrix} a_{11}\\ \vdots \\ a_{1n}\\ \vdots \\ a_{m1} \\ \vdots \\ a_{mn} \end{matrix} \right)
\mvec(\mathbf{A}) \quad \coloneq \quad \left( \begin{matrix} \mathbf{a}_1\\ \vdots \\ \mathbf{a}_n \end{matrix} \right) \; \in \; \mathbb{R}^{mn}
\end{equation}

Additionally, we sometimes make use of the following shortcut:
\begin{equation} \label{multiA}
 %\mvec(\mathbf{A}) \quad \coloneq \quad \left( \begin{matrix} a_{11}\\ \vdots \\ a_{1n}\\ \vdots \\ a_{m1} \\ \vdots \\ a_{mn} \end{matrix} \right)
\underline{\mathbf{A}} \quad \coloneq \quad \mathbf{I}_{p, p} \otimes \mathbf{A} \quad = \qquad \underbrace{\!\!\!\!\!\!\left( \begin{matrix} \mathbf{A}&  & \mathbf{0} \\  & \ddots &  \\ \mathbf{0} &  & \mathbf{A} \end{matrix} \right)\!\!\!\!\!\!}_{\text{$p$ times $\mathbf{A}$}} 
\end{equation}

\subsection{SFA harmonics concerning a bottleneck} \label{sec:bottleneck}
\providecommand{\gmn}[0]{\mu_{ab}} % Gaussian mean

When dealing with multiple rooms, we encounter the transition between rooms as bottlenecks in the environment, i.e.\ areas of low probability during a random exploration of the environment. In this section we study the results of SFA for the case that a source is not uniformly distributed, but concerns a bottleneck in the probability distribution. Using a normally-distributed repeller $\eta(x)$, we model a bottleneck situation on a one dimensional interval $[a, b]$, i.e.\ concerning only a single source $\mathbf{s}_{\alpha} = x$, $n = 1$. Throughout this section we use the shortcut $\gmn \coloneq \frac{a+b}{2}$. We define

\begin{align}
\eta_{\tau, \sigma}(x) \quad &\coloneq \quad \sign(x-\gmn) \frac{\tau}{\sqrt{2 \pi \sigma^2}} e^{- \frac{(x-\gmn)^2}{2 \sigma^2}} \label{repeller} \\
p_{\dot{x} | x}(\dot{x} | x) \quad &\coloneq \quad \frac{1}{2 \delta} \; \mathbf{1}_{[-\delta + \eta_{\tau, \sigma}(x), \; \delta+\eta_{\tau, \sigma}(x)]} (\dot{x}) \label{diff-p-repeller}
\end{align}
with $\delta$ denoting the strength of the random walk and $\tau$ denoting the strength of the repeller. Empirically we find that if $\delta$ is sufficiently large, $p(x)$ can be modeled as an accordingly scaled difference between a uniform and a normal distribution:
\begin{equation}
p_x(x) \quad \approx \quad \mathbf{1}_{[a, b]}(x) \; \left( \frac{1 + \tilde{\tau}(\erf_{\gmn, \tilde{\sigma}}(b)-\erf_{\gmn, \tilde{\sigma}}(a))}{b-a} - \frac{\tilde{\tau}}{\sqrt{2 \pi \tilde{\sigma}^2}} e^{- \frac{(x-\gmn)^2}{2 \tilde{\sigma}^2}}\right)
\end{equation}
with $\erf_{\gmn, \sigma} = \int \frac{1}{\sqrt{2 \pi \sigma^2}} \exp(- \frac{(x-\gmn)^2}{2 \sigma^2})$ denoting the cumulative density function of a normal distribution. Of course, $\tilde{\tau}$ and $\tilde{\sigma}$ depend on $\tau$ and $\sigma$, but formulating the specific relationship would require a deeper study of the heterogeneous random walk induced by \eqref{diff-p-repeller}. Further, we find that $\av{\dot{x}^2}_{\dot{x} | x} = \frac{\eta_{\tau, \sigma}(x)^3}{3 \delta} + \delta \eta_{\tau, \sigma}(x)$ is not constant, which results in a difficult differential equation for the analytic SFA solutions. These aspects are out of scope of this paper, so we focus on empirical results.

To simulate an unrestricted function space we use Legendre polynomials rather than monomials for nonlinear expansion. Monomials lead to invalid solutions or failures already at expansion degrees around $15$ due to their poor numerical properties concerning floating point arithmetics (dirty zero effect). Using Legendre polynomials we are able to retrieve valid solutions for degrees $>100$. Depending on the data we were able to apply expansions up to degree $140$.

\begin{figure}[ht]
	\centering
	\captionsetup{width=.90\linewidth}
	\includegraphics[width=1.0\hsize]{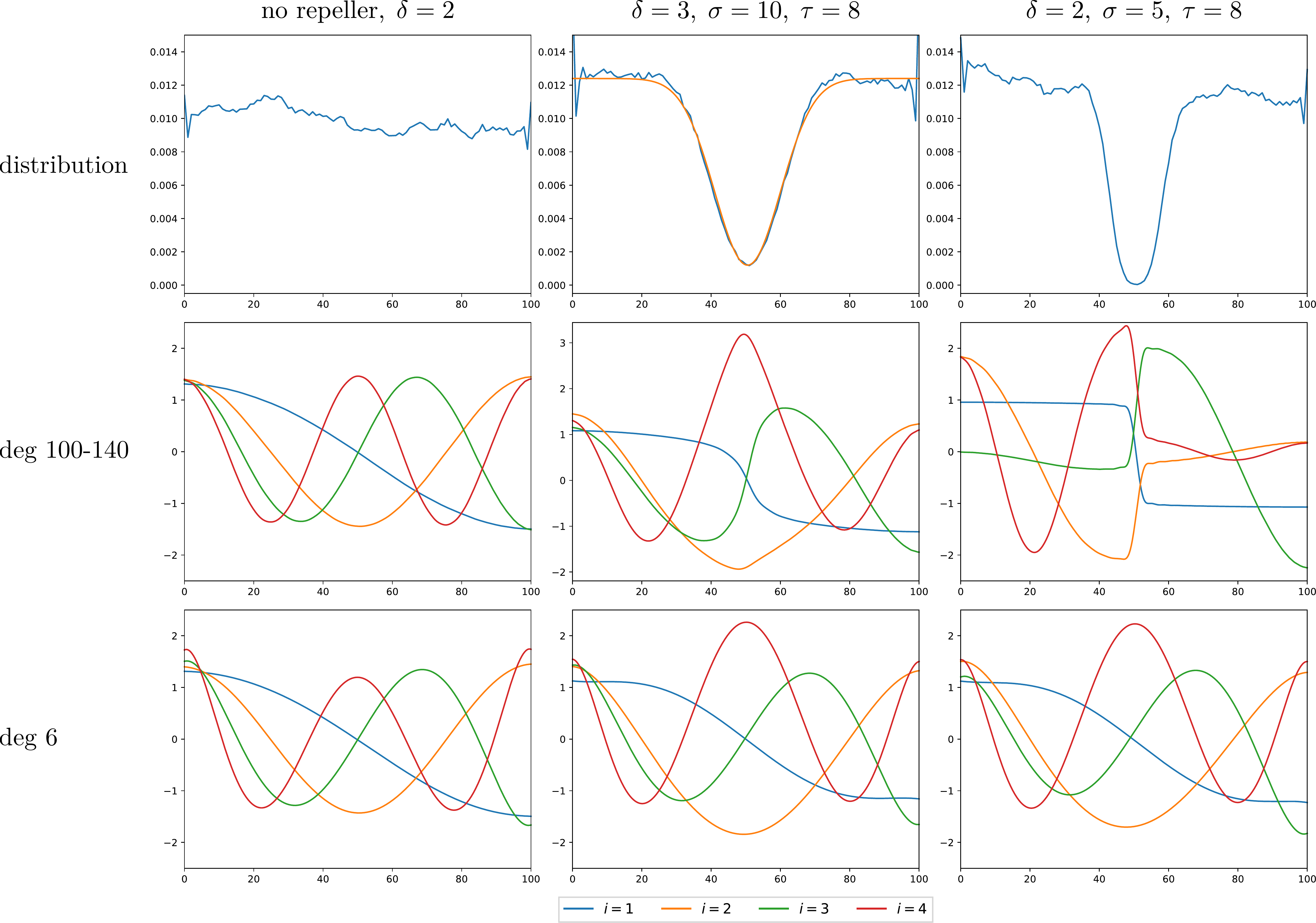}
	\caption{Illustration of SFA concerning a bottleneck. Each walk consists of $10^6$ steps. For nonlinear expansion, Legendre polynomials up to the denoted degree were used: $100$ in the center, $140$ left and right. The second line displays approximately ideal solutions with a high nonlinear expansion simulating an unrestricted function space. The third line shows how solutions are smoothened due to a more restricted function space, i.e.\ Legendre polynomials up to the sixth degree.}
	\label{fig:bottleneck}
\end{figure}

From the results in figure \ref{fig:bottleneck} we conclude that a bottleneck in the probability distribution acts like an attractor for steepness on the SFA solutions. The original harmonics are perturbed by a sudden concentration of steepness at the bottleneck. This is exactly expected behavior, because the bottleneck is a low-weight region and steepness is a high-cost factor for SFA. So the algorithm uses the bottleneck to store as much of the overall cost as possible. As a consequence, regions outside the bottleneck become flatter. Especially the first harmonic can gain -- while still being monotonic -- very flat regions, even close to constant. Thus, the first harmonic turns into an approximate indicator function distinguishing the two regions separated by the bottleneck.
%This is an interesting reference to work that uses SFA for classification, e.g.\ \cite{Escalante-B.Wiskott-2013b}.
Also note that the sudden concentration of steepness can be used to detect the bottleneck state, which is a relevant notion in \cite{DBLP:conf/icml/McGovernB01, 10.1007/3-540-45622-8_16}. It can also model communities and surprise, which are relevant e.g.\ in \cite{Schapiro2013}.
%Automatic Discovery of Subgoals in Reinforcement Learning, Learning Options in Reinforcement Learning
%Neural representations of events arise from temporal community structure

For our navigation approach, flat regions can yield issues and we will propose a method to handle them. Note that -- as illustrated in figure \ref{fig:bottleneck}/bottom row -- a more restricted function space has a smoothening effect on the solutions. This already compensates issues with flat regions to some extend.

\end{document}